\documentclass[11pt]{article}

\usepackage{amsmath}
\usepackage{amsfonts}
\usepackage{dsfont}
\usepackage[english]{babel}
\usepackage{fullpage}
\usepackage{subcaption}
\usepackage[authoryear,round,longnamesfirst]{natbib}
\usepackage{mathtools}
\usepackage{amsthm}
\usepackage{hyperref}

\theoremstyle{definition}

\newtheorem{theorem}{Theorem}
\newtheorem{lemma}{Lemma}
\newtheorem{definition}{Definition}

\makeatletter
\renewcommand\paragraph{\@startsection{paragraph}{4}{\z@}%
            {-2.5ex\@plus -1ex \@minus -.25ex}%
            {1.25ex \@plus .25ex}%
            {\normalfont\normalsize}}
\makeatother
\DeclarePairedDelimiter\abs{\lvert}{\rvert}%
\DeclarePairedDelimiter\norm{\lVert}{\rVert}%
\makeatletter
\let\oldabs\abs
\def\abs{\@ifstar{\oldabs}{\oldabs*}}
\let\oldnorm\norm
\def\norm{\@ifstar{\oldnorm}{\oldnorm*}}
\makeatother

\begin{document}

\title{Understand Functionality and Dimensionality of Vector Embeddings: the Distributional Hypothesis, the Pairwise Inner Product Loss and Its Bias-Variance Trade-off}

\author{ Zi Yin \\ \texttt{ s0960974@gmail.com} \\
        Department of Electrical Engineering\\
       Stanford University\\
       Stanford, CA 94305, USA
}

\maketitle

\begin{abstract}
Vector embedding is a foundational building block of many deep learning models, especially in natural language processing. In this paper, we present a theoretical framework for understanding the effect of dimensionality on vector embeddings. We observe that the distributional hypothesis, a governing principle of statistical semantics, requires a natural unitary-invariance for vector embeddings. Motivated by the unitary-invariance observation, we propose the Pairwise Inner Product (PIP) loss, a unitary-invariant metric on the similarity between two embeddings. We demonstrate that the PIP loss captures the difference in functionality between embeddings, and that the PIP loss is tightly connect with two basic properties of vector embeddings, namely similarity and compositionality. By formulating the embedding training process as matrix factorization with noise, we reveal a fundamental bias-variance trade-off between the signal spectrum and noise power in the dimensionality selection process. This bias-variance trade-off sheds light on many empirical observations which have not been thoroughly explained, for example the existence of an optimal dimensionality. Moreover, we discover two new results about vector embeddings, namely their robustness against over-parametrization and their forward stability. The bias-variance trade-off of the PIP loss explicitly answers the fundamental open problem of dimensionality selection for vector embeddings.
\end{abstract}
\section{Introduction}
Vector space models (VSMs) are extremely useful and versatile tools that serve as keys to many fundamental problems in numerous research areas \citep{turney2010frequency}. To name a few, VSMs are widely applied in information retrieval \citep{salton1971smart,salton1988term,sparck1972statistical}, recommendation systems \citep{breese1998empirical,yin2017deepprobe}, natural language processing \citep{NIPS2013_5021,levy2014neural,church1990word}, image processing \citep{frome2013devise} and even coding theory \citep{nachmani2017deep}. Interesting and useful objects in their own rights, VSMs also serve as the building blocks for other powerful tools. VSMs map discrete tokens into vectors of Euclidean spaces, making arithmetic operations possible. For example, addition of words is not a well-defined operation, but it immediately becomes a meaningful operation if we use their vector representations. Some equally important operations made possible by VSMs are token transformations (through matrix-vector product and non-linearity), token-token interactions (through vector inner product and cosine-similarity), \textit{etc}. These operations can be directly used to compare tokens, for example, obtain the semantic and syntactic relations between words \citep{NIPS2013_5021}, similarity between documents \citep{sparck1972statistical} and comparing vocabulary for different languages \citep{mikolov2013exploiting}. Furthermore, numerous important applications are built on top of the VSMs. Some prominent examples are recurrent neural networks, long short-term memory (LSTM) networks \citep{hochreiter1997long} that are used for language modeling \citep{bengio2003neural}, machine translation \citep{sutskever2014sequence,bahdanau2014neural}, text summarization \citep{nallapati2016abstractive} and image caption generation \citep{xu2015show,vinyals2015show}. Other important applications include named entity recognition \citep{lample2016neural}, sentiment analysis \citep{socher2013recursive}, generative language models \citep{arora2015rand} and so on.

Despite the wide range of applicability and effectiveness of vector embedding, its underlying theory has not been fully uncovered. Recent papers \citep{arora2015rand,gittens2017skip} discussed some generative explanations of embeddings, which are novel and ingenious. However, their hypotheses do not solve fundamental open problems like dimensionality selection. It has been observed that vector embedding tends to have a ``sweet dimensionality'', going above or below which leads to sub-optimality in various tasks. \citet{bradford2008empirical} studied empirically the effect of dimensionality on VSMs in natural language processing, and concluded that taking an embedding of dimensionality 300-500 is desirable. Similarly \citet{landauer2006latent} recommended embeddings with dimensionalities being a few hundred. However, besides empirical studies, embedding dimensionality selection has not been discussed from a theoretical point of view. Dimensionality, a critical hyper-parameter, has great impact on the performance of the embedding models in a few aspects. First, dimensionality has a direct impact on the quality of embeddings. When the dimensionality is too small, the model is not expressive enough to capture all the token relations. When the dimensionality is too large, embeddings will suffer from over-fitting. Second, dimensionality is directly connected to the number of parameters in the model, affecting the training and inference time. Specifically, the number of parameters depends linearly on the dimensionality, for both the embeddings themselves and the models that use the embeddings as building blocks (for example, recurrent neural networks built upon word embeddings). As a result, large dimensionality increases model complexity, slows down training speed and adds inference latency, all of which are practical constraints of model deployment and applicability \citep{wu2016google}.

In this paper, we introduce a mathematical framework for understanding vector embedding's dimensionality. Our theory answers many open questions, in particular:
\begin{enumerate}
\item What is an appropriate metric for comparing two sets of vector embeddings?
\item  How to efficiently select optimal dimensionality for vector embeddings?
\item What is the effect of over-parametrization on vector embeddings?
\item How stable are the embedding algorithms with noise and/or hyper-parameter mis-specification?
\end{enumerate}

The organization of the paper is as follows. In Section \ref{sec:factorization}, we introduce the background knowledge of vector embeddings and matrix factorization. In Section \ref{sec:UIP}, we discuss the functionality of vector embeddings, the distributional hypothesis, and how they lead to the unitary-invariance of vector embeddings. We propose a novel objective, the Pairwise Inner Product (PIP) loss in Section \ref{sec:PIP}. We show that the PIP loss is closely related to the functionality differences between the embeddings, and a small PIP loss means the two embeddings are essentially similar for all practical purposes. In Section \ref{sec:perturbation}, we develop matrix perturbation tools that quantify the PIP loss between the trained embedding and the oracle embedding, where we reveal a natural bias-variance trade-off arising from dimensionality selection. Moreover, we show that the bias-variance trade-off leads to the existence of a ``sweet dimensionality''. With this theory, we provide three new discoveries about vector embeddings in Section \ref{sec:discoveries}, namely vector embeddings' robustness to over-parametrization, their numerical forward stability, and an explicit dimensionality selection procedure based on minimizing the PIP loss. 

\subsection{Definitions and Notations}
Throughout the paper, we assume a vocabulary of $n$ tokens, $\mathcal V=\{1,2,\cdots, n\}$. Each token $i$ is mapped to an vector $v_i$, its vector embedding. By stacking all the vectors, we get the \textbf{embedding matrix} $E\in\mathbb{R}^{n\times d}$, where the $i$-th row, $E_{i,\cdot}=v_i$ is the embedding for token $i$. The embedding matrix $E$ is trained through an \textbf{embedding procedure}, an algorithm that produces the embedding matrix $E$ given a corpus as the input. The embedding procedures have a few hyper-parameters as knobs on which one can tune, with dimensionality being one of them. Some popular embedding procedures include the Latent Semantic Analysis (LSA) \citep{deerwester1990indexing}, GloVe \citep{pennington2014glove} and skip-gram Word2Vec \citep{mikolov2013efficient}. The following concepts will be central to our discussions.

\begin{itemize}
\item Embedding Procedure. Let the embedding procedure be $f(\cdot)$, which is an algorithm taking a corpus to an embedding matrix. The input to $f$ can be the corpus itself, or quantities derived from a corpus (\textit{e.g.} co-occurrence matrix). While this definition is generic enough to encompass all possible embedding procedures, we specifically focus on the subset of procedures involving matrix factorization. Every embedding procedure of this category can be denoted as $f_{\alpha,k}(M)$, where the input is a matrix $M$ derived from the corpus like the co-occurrence matrices. The procedure has two hyper-parameters, the exponent $\alpha$ and the dimensionality $k$. The mathematical meanings of the hyper-parameters will be discussed in detail in later sections.

\item Oracle Embedding. Let the oracle embedding be $E=f_{\alpha,d}(M)$, where $M$ is the unobserved, ``clean'' signal matrix and $d$ is the rank of $M$. For example, if the signal matrix $M$ is the co-occurrence matrix, the oracle $E$ will be the output of $f$ on $M$, where $M$ does not contain any noise, and $d$ will be the ground-truth rank of $M$. In reality, this requires us to have access to an infinitely-long corpus and noiseless estimation procedures to get the exact signal matrix $M$, which is not feasible. The oracle embedding is the best one can hope for with a procedure $f$, and it is desirable to get as close to it as possible during training with the actual noisy data and possibly estimation/numerical errors.

\item Trained Embedding. Let the trained embedding be $\hat E=f_{\alpha,k}(\tilde M)$, where $\tilde M$ is the estimated, ``noisy'' signal matrix. As in the previous example, in order to practically train an embedding, the procedure $f$ will operate on $\tilde M$, the co-occurrence matrix estimated on a finite-length corpus. There are various sources of noise, including the finite length of the corpus and estimation errors. Apart from estimation errors, the procedure $f$ sometimes includes stochastic optimization in order to obtain the factorization of the signal matrix, which introduces computational noise.

\item Token Space Relational Property. A token space relational property $p_{token}(n_1,\cdots, n_l)$ takes a fixed number of tokens as arguments, and outputs a numerical score which measures the degree to which the property holds for the set of tokens. For example, $p_{token}(i,j)=\text{sim}(i,j)$ is the relational property reflecting the similarity between token $i$ and token $j$, measured as a numerical score between 0 and 1. In a natural language example where the tokens are words, $\text{sim}(\text{aircraft},\text{plane})$ will be large, and $\text{sim}(\text{lighthouse},\text{write})$ will be small. Other token space relational properties include compositionality, analogy, relatedness and more. An elaborated discussion will be in Section \ref{sec:UIP}.

\item Vector Space Relational Property. A vector space relational property $p_{vec}(v_1,\cdots, v_l)$ takes a fixed number of vectors as arguments, and outputs a numerical score which measures the degree to which the property holds for the vectors. For example, $p_{vec}(v_i,v_j)=\text{sim}(v_i,v_j)=\cos(v_i,v_j)$ is the vector space relational property where the similarity between vector $v_i$ and $v_j$ is measured using their cosine similarity. Vector space similarity $\text{sim}(v_i,v_j)$ is a counterpart of the token space similarity $\text{sim}(i,j)$. Likewise, there are counterparts for other token space relational properties, like vector compositionality and vector clustering.
\end{itemize}

As we will demonstrate shortly in Section \ref{sec:twospaces}, the relational properties are the nexus between the token space and vector space. As a prerequisite, we introduce in Section \ref{sec:factorization} the embedding procedures whose input $M$ is a matrix, and whose output is an embedding matrix from matrix factorization, either explicitly or implicitly. This framework is not too limiting. In fact, most embedding procedures are under this category, including Latent Semantic Analysis \citep{sparck1972statistical}, Word2Vec with skip-gram \citep{NIPS2013_5021} and GloVe \citep{pennington2014glove}.

\section{Vector Embeddings and Matrix Factorization}
\label{sec:factorization}
In this section, we discuss the relation between vector embedding and matrix factorization, which serves as a prerequisite for the rest of the paper. Embedding methods were first developed using matrix factorization, for example the LSA/LSI method in natural language processing and information retrieval \citep{turney2010frequency}. Recent advance in deep learning and neural networks motivated new embedding algorithms, where the objectives are to minimize co-occurrence based objective function using stochastic gradient methods. Of the algorithms, two most widely used ones are the skip-gram \citep{mikolov2013efficient} and GloVe \citep{pennington2014glove}. Although they were not introduced as matrix factorization methods, studies show that both are actually doing implicit matrix factorizations \citep{levy2014neural,levy2015improving}.

\subsection{Embeddings from Explicit Matrix Factorization}
Embeddings obtained by matrix factorization cover a wide range of practical machine learning and NLP algorithms, including the Latent Semantics Analysis/Indexing (LSA/LSI). We briefly introduce some of the popular methods of obtaining embeddings from explicit matrix factorization.
\begin{itemize}
\item Document embedding using Latent Semantics Indexing (LSI).  Since first introduced by \citet{deerwester1990indexing}, LSI has become a powerful tool for document analysis and information retrieval. A term-document co-occurrence matrix $M$ is computed from the data. $M$ can be the Term Frequency (TF) matrix, the raw token frequency counts for every document. TF matrix can be biased by the popular words, and different re-weightings schemes were proposed to address this issue, the most popular of which is the TF-IDF matrix \citep{salton1988term,sparck1972statistical}. The embedding of the documents
\[E=f_{\alpha, k}(M)=U_{1:k}D_{1:k,1:k}^\alpha\]
is computed using SVD, where $\alpha$ is conventionally chosen to be 0.5 or 1 \citep{landauer1998introduction}. Downstream tasks, like information retrieval, operates on the document embedding $E$.
\item Word embedding using Latent Semantics Analysis (LSA). A word-context co-occurrence matrix $M$, sometimes with re-weighting, is computed from the corpus. Embeddings are obtained by applying truncated SVD on $M$
\[E=f_{\alpha, k}(M)=U_{1:k}D_{1:k,1:k}^\alpha\]
Popular choices of the matrices $M$ are the Pointwise Mutual Information (PMI) matrix \citep{church1990word}, the positive PMI(PPMI) matrix \citep{niwa1994co} and the Shifted PPMI (SPPMI) matrix \citep{levy2014neural}. The exponent hyper-parameter $\alpha$ can be any number in $[0,1]$ \citep{landauer2006latent}, which is usually chosen as 0.5 for symmetry. However $\alpha=0$ or $1$ is not uncommon.
\end{itemize}

We notice that in a few papers \citep{caron2001experiments,bullinaria2012extracting,turney2012domain, levy2014neural}, it was noticed that there is a continuum of ways to obtain embeddings from matrix factorization by varying $\alpha$, the exponent parameter. In particular, suppose the embedding procedure factorized matrix $M$. Let $M=UDV^T$ be its singular value decomposition. A dimensionality $k$ embedding can be obtained by truncating the left singular matrix at dimensionality $k$, and multiply it with a power of the truncated diagonal matrix, i.e. $E=f_{\alpha, k}(M)=U_{1:k}D_{1:k,1:k}^\alpha$. In \citet{caron2001experiments,bullinaria2012extracting} the authors empirically checked that different $\alpha$ between 0 and 1 works for different language tasks. In \citet{levy2014neural} when the authors explained the connection between Word2Vec and matrix factorization, they set $\alpha=0.5$ to enforce symmetry. We assume $\alpha$ is given as part of the model specification, and do not discuss its effect in this paper.

It is generally believed that the role of matrix factorization in obtaining embeddings is for dimensionality reduction, sparsity reduction and
variance reduction \citep{rapp2003word,turney2010frequency}. While the first two are obvious as direct consequences of truncated SVD, the last role, namely variance reduction, is not mathematically clear, especially from the conventional, low-rank approximation point of view. In fact, applying the conventional low-rank approximation theory is a mis-conception; it cannot explain the observation that the models with reduced dimensionalities often perform better. As we will soon demonstrate, reducing dimensionality will help clear out noise, at the price of losing some signal. This trade-off between signal and noise can be captured by our theory, which explicitly leads to an approach for dimensionality selection.

\subsection{Embeddings from Implicit Matrix Factorization}
\label{sec:implicit_factorization}
In NLP, two most widely used embedding models are skip-gram Word2Vec \citep{NIPS2013_5021} and GloVe \citep{pennington2014glove}. Skip-gram Word2Vec maximizes the likelihood of co-occurrence of the center word and context words. The log likelihood is defined as
\[\sum_{i=0}^{n}\sum_{j=i-w,j\ne i}^{i+w} \log(\sigma(v_j^Tv_i)),\ \sigma(x)=\frac{e^x}{1+e^x}\]
This objective is optimized using stochastic gradient method. Skip-gram Word2Vec is sometimes enhanced with other techniques like negative sampling \citep{mikolov2013exploiting}, where every (word, context) pair is regarded as a positive training example. Along with every positive example, $l$ negative examples will be constructed by randomly sample ``fake'' (word, context) pairs. The objective function, with the presence of negative sampling, becomes maximizing the log likelihood for the positive samples, and at the same time minimizing the log likelihood for the negative samples. \citet{mikolov2013efficient} argued that negative sampling helps increase differentiability. 

In \citet{levy2014neural}, the authors showed that skip-gram Word2Vec's objective is an implicit factorization of the pointwise mutual information (PMI) matrix:
\[\text{PMI}_{ij}=\log\frac{p(v_i,v_j)}{p(v_i)p(v_j)}\]
For skip-gram Word2Vec with negative sampling, \citet{levy2014neural} used a novel Shifted PPMI matrix as surrogate:
\[\text{SPPMI}_{ij}=\left(\log\frac{p(v_i,v_j)}{p(v_i)p(v_j)}-\log(l)\right)_{+}\]

Regarding Glove, it was pointed out by the original authors \citep{pennington2014glove} and recently by \citet{levy2015improving} that GloVe's objective is to factorize the log-count matrix. The factorization is sometimes augmented with bias vectors and the log-count matrix is sometimes raised to an exponent $\gamma\in[0,1]$ \citep{pennington2014glove}. Due to the connections of the above two popular embeddings and matrix factorization, our theory applies to Word2Vec and GloVe, and other implicit models. This will later be verified experimentally.

\section{Functionality of Embeddings: the Unitary Invariant Properties}
\label{sec:UIP}
In order to decide whether the trained embeddings are close to optimal, a meaningful metric for comparing embeddings should first be developed. We address the question that has long been overlooked by previous research, that is, given two embedding matrices $E_1$ and $E_2$, what is a good distance metric, $\text{dist}(E_1, E_2)$? One might think this can be resolved by looking at $\|E_1-E_2\|$, where $\|\cdot\|$ is some suitable norm (like the matrix 2-norm or the Frobenius norm). We will now demonstrate that directly comparing $E_1$ and $E_2$ is insufficient, which misses a key observation on the invariance of functionality under unitary transformations. To provide an illustrative example, we trained two character embeddings by performing two runs of the same embedding procedure (Word2Vec) with identical hyper-parameters on the same dataset (Text8). In this case, it is natural to expect similar results between the runs. Figure \ref{fig:embeddings} is the visualization of the two embeddings using t-SNE \citep{maaten2008visualizing}. Surprisingly, at a first glance, they seem to be different, but a careful examination reveals that \ref{fig:embed2} can approximately be obtained by rotating \ref{fig:embed1}. This example reveals the shortcoming of direct embedding comparisons, namely their distance can be large while the embeddings are functionally close to each other. Indeed, the two trained embeddings are functionally close; all character-to-character relations are similar in Figure \ref{fig:embed1} and Figure \ref{fig:embed2}. At the same time, $\|E_1-E_2\|$ will be large despite their functional closeness. In other words, a distance metric serving as the gauge of functionality, rather than the absolute differences of vectors, is needed. 

\begin{figure}[htb]
\centering
\hspace*{\fill}%
\begin{subfigure}[b]{0.45\textwidth}
\includegraphics[width=\textwidth]{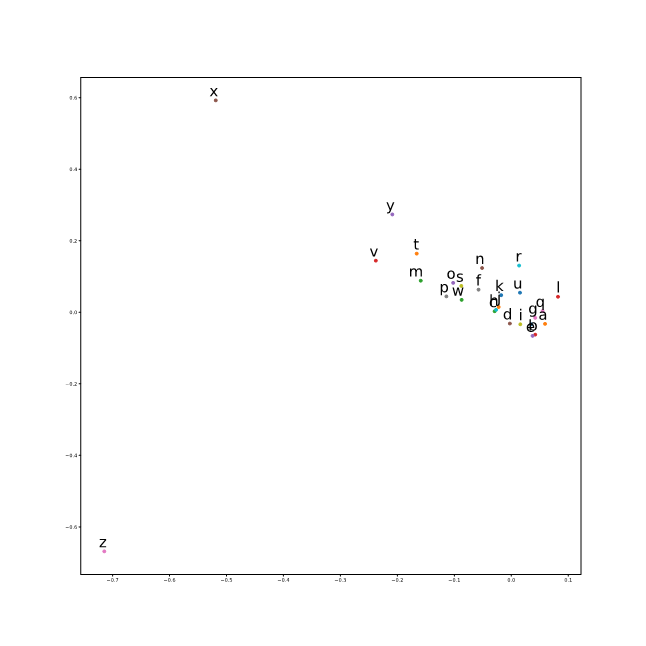}
            \caption[]%
            {{\small First Run}} 
            \label{fig:embed1}
        \end{subfigure}
        \hfill
        \begin{subfigure}[b]{0.45\textwidth}           \includegraphics[width=\textwidth]{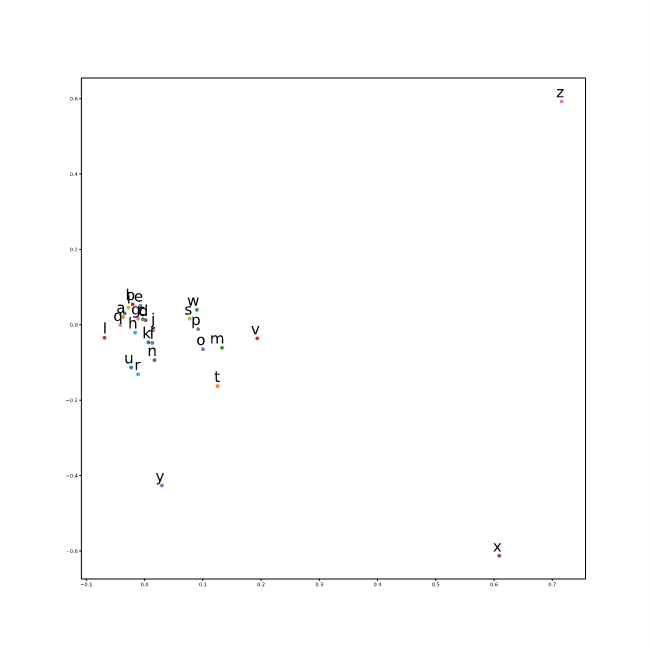}
            \caption[]%
            {{\small Second Run}}    
            \label{fig:embed2}
        \end{subfigure}
        \hspace*{\fill}%
\caption{Two Trained Embeddings: the Second is Approximately a Rotation of the First}
\label{fig:embeddings}
\end{figure}

To determine whether the trained embedding $\hat E$ is good or not, the distance metric for embeddings becomes crucial. Since the oracle embedding $E=f_{\alpha,d}(M)$ is the best possible one can get with the procedure $f$, the optimality of $\hat E=f_{\alpha,k}(\tilde M)$ is defined on how far it deviates from the oracle $E$. As we have already seen earlier, a direct comparison like $\|E-\hat E\|$ does not capture the essence. What we desire is that $\hat E$ be functionally close to $E$, in terms of their abilities to capture the relational properties of the original tokens. Specifically, the vector space relational properties should be close between $E$ and $\hat E$, and both should be close to their corresponding token space relational properties, in order for $\hat E$ to qualify as good embeddings.

To better see this point, we present a discussion of word embedding evaluations. Most evaluation tests focus on the functionality of embeddings, which means some intrinsic token space relational property is compared against their corresponding vector space relational property. The degree to which they agree determines the quality of the vector embeddings on that test. The evaluation procedure involves two steps. First, human judges will evaluate numerical scores of some token space relational property, which serves as the ground-truth scores. For example, when evaluating the intrinsic similarity of word pairs, judges will compare and rate the words using their language domain knowledge. It is likely that they will give a high score on the pair (potato, yam) but a low score on (drink, aircraft). The scores are collected and compiled into a testset, against which the embeddings are evaluated. In the example of word similarity test, for every pair of words in the testset, the cosine similarity of their vector embeddings are calculated and compared against the scores labeled by human judges. A high correlation between the two scores indicates the vectors capture these intrinsic similarity of the tokens. An embedding is good if it achieves high scores on a variety of intrinsic functionality tests.

There are many intrinsic functionality tests proposed in previous work. \citet{schnabel2015evaluation, baroni2014don} summarized existing intrinsic functionality tests, which include relatedness, analogy, synonyms, clustering, categorization, and selectional preferences. We notice that the above intrinsic functionality tests are variants of two basic relational properties.
\begin{itemize}
\item Similarity. Token space similarity $\text{sim}(i,j)$, as its name suggests, is a similarity score between two tokens $i$ and $j$. It reflects the similarity in meaning and usage of tokens and can be determined by human judges. $\text{sim}(i,j)$ is the human labeled score, normalized into range $[0,1]$, where a higher score indicates higher similarity. Its counterpart, vector space similarity, is the cosine similarity, $\cos(v_i,v_j)$.

\[\cos(v_i,v_j)=\frac{\langle v_i,v_j\rangle}{\|v_i\| \|v_j\|}\]
Relatedness, synonyms, clustering/categorization and selectional preferences are based on this property.
\item Compositionality. Token space compositionality is the property between two pairs of tokens, which measures whether the relations between the two pairs are similar. For example, (husband, wife) and (king, queen) have a high compositionality rating, as they both reflect the same relation between genders. However, (water, ice) and (Paris, France) have a low compositionality rating, as the first pair is about phase relation and the second relation is country and its capital. The counterpart in vector space is $\|(v_i-v_j)-(v_l-v_k)\|$, where $\|(v_i-v_j)-(v_l-v_k)\|\approx 0$ indicates good compositionality. Analogy tests are based on this property.
\end{itemize}
\subsection{The Tale of Two Spaces}
\label{sec:twospaces}
In this section, we will be frequently switching between two spaces: the space of tokens, and the space of vector embeddings. As discussed in Section \ref{sec:UIP}, the goal of training vector embeddings is to obtain a mathematical handle, so that the token space relational properties can be reflected by the corresponding vector space relational properties. It is generally believed that one of the governing rules in the token space is the \textit{distributional hypothesis} \citep{harris1954distributional}, which was later summarized and popularized by Firth: ``You shall know a word by the company it keeps'' \citep{firth1957synopsis}. The original distributional hypothesis applies to natural language where the tokens are words. It imposes a relativity on the token space, stating that the meaning of the tokens are determined \textit{relatively} by other tokens. Recent research \citep{lin2001dirt,turney2003measuring} proposed extensions of the distributional hypothesis to broader objects, such as patterns and latent relations. These generalizations apply to a broad class of tokens other than words, including patterns and documents. Some examples of the distributional hypothesis extensions are
\begin{itemize}
\item Meanings of words are determined by the contexts. Words appear in similar contexts have similar meanings.
\item Meanings of patterns are determined by the word pair they connect. Patterns that co-occur with similar word pairs have similar meanings.
\item Meanings of document are determined by the words they contain. Documents that co-occur with similar words have similar meanings.
\end{itemize}
Embedding procedures are guided by the distributional hypothesis and its extensions. Examples include word embedding (using word-context co-occurrence), document embedding (using document-word co-occurrence), and pattern embedding (using pattern-word pair co-occurrence). As a result of the training procedures, the relativity of token meanings, as required by the distributional hypothesis, will be inherited by the vector embeddings. A vector by itself does not reflect any properties of the token it represents; only when vectors are compared with each other can the meanings and relations be revealed. As a quick evidence, we note that all the previously mentioned intrinsic functionality tests for word embeddings, including relatedness, analogy, synonyms, clustering and categorization, are relative by nature. A natural consequence of this relativity is that the relevant vector space relational properties should not be affected by any vector space operations that preserve the \textit{relative geometry}, for example a rotation. In operator theory, this is called \textit{unitary-invariant}. A unitary operation on a vector corresponds to multiplying the vector by a unitary matrix, \textit{i.e.} $v'=vU$, where $U$ is an operator of the orthogonal group $\text O(d)$, or $U^TU=UU^T=Id$. As a side note, not only does the unitary-invariance apply to embeddings of the same set of tokens \citep{hamilton2016diachronic}; recent research in NLP and word embedding \citep{artetxe2016learning, smith2017offline} discovered that even embeddings of two languages can be mapped to each other through unitary transformations. It should be immediately noticed that a unitary transformation keeps the relative geometry of the vectors, and hence defines an \textit{equivalence class} of embeddings. In other words, one can define the relation $E_1\overset{U}{\sim} E_2$ as $E_2$ is equivalent to $E_1$ up to a unitary transformation. The well-definedness of this equivalence relation can be readily seen from the following properties:
\begin{itemize}
\item Reflexivity. $E\overset{U}{\sim} E$ through the identity transform $E=EI$.
\item Symmetry. If $E_1\overset{U}{\sim} E_2$, then $E_2=E_1U$ for some $U$. So $E_1=E_2U^T$ which means $E_2\overset{U}{\sim} E_1$.
\item Transitivity. If $E_1\overset{U}{\sim} E_2$ and $E_2\overset{U}{\sim} E_3$, then there exists two unitary matrices $U_1$ and $U_2$ such that $E_2=E_1U_1$ and $E_3=E_2U_2$. As a result, $E_3=E_1(U_1U_2)$ which means $E_3\overset{U}{\sim} E_1$.
\end{itemize}
We should keep in mind that the relativity and equivalence relation will manifest themselves multiple times in different aspects throughout our discussion. When we propose a loss metric in the Section \ref{sec:PIP}, we emphasize that the metric should be invariant under this equivalence relation. In other words, for a distance metric to be well-defined for our purposes, it is required that if $E_1\overset{U}{\sim} E_2$, then $\text{dist}(E_1,E_2)=0$. This relativity principle guides the development of the matrix perturbation analysis in Section \ref{sec:perturbation}, where a bias-variance trade-off is proved. The trade-off depends on the spectrum of the signal matrix and the standard deviation of the noise, but not on the signal direction. This is because signal directions are non-unique under the equivalence relation; a change of basis will alter the signal directions while the embedding matrix with the new basis is essentially the same as the old one. On the other hand, both the signal spectrum and noise standard deviation remain unchanged with respect to unitary transformations.

\subsection{Unitary-invariant Properties: a Class of Vector Space Relational Properties}
Consider an embedding matrix $E\in\mathbb{R}^{n\times k}$, or equivalently a set of vector embeddings $\{v_i\}_{i=1}^n$. We define the Unitary-Invariant Properties (UIPs) to be the vector space relational properties on $\{v_i\}_{i=1}^n$ that do not change under unitary operations, which include the original token space relational properties that the vectors are capturing. Specifically, for a relational property $p_{vec}(v_1,\cdots, v_l)$, if the following holds
\[\forall U\in O(d),~ p_{vec}(v_1,\cdots, v_l)=p_{vec}(v_1U,\cdots, v_lU),\]
then $p_{vec}$ is a UIP. In other words, UIPs are invariant under the equivalence relation $\overset{U}{\sim}$, in the sense that if $E_1\overset{U}{\sim}E_2$, then all their UIPs are identical. Not every vector space relational property is unitary-invariant. To illustrate this point, we look at the property $p(v_1,v_2)$, the magnitude of vector pair compositionality. If defined as $p(v_1,v_2)=\|v_1-v_2\|_\infty$, the size of the most prominent direction, it is \textit{not }unitary invariant. Take $v_1=e_1,~ v_2=e_2$, the first two canonical basis vectors, as an example. Define the unitary matrix
\[
U=
\left[
\begin{array}{c c}
\left[
\begin{array}{c c}
\sqrt{2}/2 & \sqrt{2}/2 \\
\sqrt{2}/2 & -\sqrt{2}/2
\end{array}
\right] & 0 \\
0 & I_{d-2,d-2}
\end{array}
\right]
\]
Then $v_1U=(\sqrt{2}/2,\sqrt{2}/2,0,\cdots, 0)$ and $v_2U=(\sqrt{2}/2,-\sqrt{2}/2,0,\cdots, 0)$. In this case
\[p(v_1,v_2)=1,~ p(v_1U,v_2U)=\sqrt{2}.\]
On the other hand, it is not difficult to show that $p(v_1,v_2)=\|v_1-v_2\|_2$ is unitary-invariant. In general, if a vector space relational property depends explicitly on the basis, it is not unitary-invariant.

In the discussions in Section \ref{sec:twospaces}, we touched briefly on the distributional hypothesis, and argued that it imposes certain relativity on the tokens and their corresponding vector representations. This ultimately leads to the invariance of the intrinsic functionality test results under unitary operations. A careful examination reveals that the intrinsic functionality tests described in previous literature, including relatedness, analogy, synonyms, clustering/categorization and selectional preferences, are all unitary invariant, which is by no means just a coincidence. The unitary-invariance of the aforementioned intrinsic tests boils down to the the unitary-invariance of the two fundamental relational properties of vector embeddings, namely \textit{similarity} and 
\textit{compositionality}. It is trivial to see that both are unitary-invariant, \textit{i.e.} substituting every $v_i$ by $v_iU$ yields same results. $\cos(v_i,v_j)=\cos(v_iU,v_jU)$ and $\|(v_i-v_j)-(v_l-v_k)\|=\|(v_iU-v_jU)-(v_lU-v_kU)\|$ for all $U\in O(d)$.

The above discussion leads to an intuitive answer to the previous question on distance metric for embeddings. Meanings and functionalities of token embeddings are relatively determined by each other as a consequence of the distributional hypothesis. As a result, the properties being reflected by vector embeddings should remain unchanged under any unitary operation, since unitary operations do not change the relative geometry of the vector. Because of this, the distance between the functionality of two embeddings are reflected on their UIP differences. Specifically, if $E_1$ and $E_2$ has similar UIPs, they will achieve similar scores in the intrinsic tests, hence they are functionally similar. As a result, to measure the closeness of $\hat E$ and $E$, we check the discrepancy between their UIPs, as UIPs serve as \textit{functionality} gauges. For example, a faithful $\hat E$ should maintain similarity and compositionality predicted by $E$, along with other UIPs. This leads to the Pairwise Inner Product (PIP) distance, as will be discussed in Section \ref{sec:PIP}.

\section{PIP Distance: a Metric on Embedding Functionalities}
\label{sec:PIP}
We propose the Pairwise Inner Product (PIP) loss as a distance metric for vector embedding functionalities, which is a natural consequence of using UIPs. At the same time, this formulation allows us to evaluate embedding training from a matrix perturbation theory perspective.

\begin{definition}[PIP matrix]
Given an embedding matrix $E\in\mathbb R^{n\times d}$, define its associated Pairwise Inner Product (PIP) matrix to be
\[\text{PIP}(E)=EE^T\]
\end{definition}
It can be easily identified that the $(i,j)$-th entry of the PIP matrix corresponds to the inner product between vector embeddings for token $i$ and $j$, \textit{i.e.} $\text{PIP}_{i,j}=\langle v_i,v_j\rangle$, and hence the name. To compare $E$ and $\hat E$, the oracle embedding and the trained embedding (possibly with different dimensionalities), we propose the \textbf{PIP loss} as a measure on their deviation:
\begin{definition}[PIP distance]
The Pairwise Inner Product (PIP) distance (or PIP loss, when used in optimization context) between two embeddings $E$ and $\hat E$ is defined as
\[\|\text{PIP}(E)-\text{PIP}(\hat E)\|_F=\sqrt{\sum_{i,j}(\langle v_i,v_j\rangle-\langle \hat v_i,\hat v_j\rangle)^2}\]
\end{definition}

Note that the $i$-th row of the PIP matrix, $v_iE=(\langle v_i, v_1\rangle,\cdots, \langle v_i, v_n\rangle)$, can be viewed as the relative position of $v_i$ anchored against all other vectors. With this intuition, we can see that in essence, the PIP loss measures the vectors' \textit{relative position shifts} between $E_1$ and $E_2$. As a result, we remove the dependency of embeddings on a specific coordinate system. The PIP loss is unitary-invariant. If $E_1\overset{U}{\sim}E_2$, then $E_2=E_1U$ for some unitary $U$. In this case, $\text{PIP}(E_2)=E_1UU^TE_1^T=\text{PIP}(E_1)$. Moreover, the PIP loss serves as a metric of {\it functionality} similarity. For example, compositionality and similarity \citep{schnabel2015evaluation, baroni2014don}, two most important UIPs of word embeddings, are derived from vector inner products. The PIP loss measures the difference in inner products, hence closely tracks these two relations. As we will show, a small PIP loss between $E_1$ and $E_2$ leads to a small difference in $E_1$ and $E_2$'s similarity and compositionality, and vice versa. To prove this claim, we need both sufficiency and necessity. Without loss of generality, the embeddings are normalized so that the the vectors in two embeddings have the same energy, \textit{i.e.} $\mathbb{E}[\|v\|^2]=\mathbb{E}[\|\hat v\|^2]$ (which is a common practice). To help understand the intuition, we show the gist of the proof, without diving into the long $(\epsilon, \delta)$ arguments.
\begin{itemize}
\item 
(Sufficiency) The sufficiency part requires that if two embeddings have similar UIPs, then their PIP distance is small. Assume the UIPs (which contains similarity and compositionality) are close for two embeddings $\{v_i\}_{i=1}^n$ and $\{\hat v_i\}_{i=1}^n$. For any two pairs of tokens $(i,j)$ and $(k,l)$ where compositionality hold, we have the following equations.
\begin{enumerate}
\item compositionality for the oracle embeddings: 
\[(\|v_i\|^2-\langle v_i,v_j\rangle)-(\langle v_i,v_l\rangle-\langle v_i, v_k\rangle)\approx 0,\]
\item compositionality for the trained embeddings: 
\[(\|\hat v_i\|^2-\langle \hat v_i,\hat v_j\rangle)-(\langle \hat v_i,\hat v_l\rangle-\langle \hat v_i, \hat v_k\rangle)\approx 0,\]
\item similarity between both embeddings and original tokens:
\[\cos(v_i,v_j) \approx \cos(\hat v_i,\hat v_j)\approx \text{sim}(i,j).\]
\end{enumerate}
Note to get 1. and 2., we multiplied the compositionality equation by $v_i$ and $\hat v_i$ respectively. Combining the equations, one gets:

\begin{enumerate}
\item 
Linear system for the oracle embeddings:
\[\|v_i\|-\text{sim}(i,j)\|v_j\|-\text{sim}(i,l)\|v_l\|+\text{sim}(i,k)\|v_k\|\approx 0\]
\item
Linear system for the estimated embeddings:
\[\|\hat v_i\|-\text{sim}(i,j)\|\hat v_j\|-\text{sim}(i,l)\|\hat v_l\|+\text{sim}(i,k)\|\hat v_k\|\approx 0\]
\end{enumerate}

We immediately notice the above two linear systems (with unknowns being $\|v_i\|$ and $\|\hat v_i\|$, the lengths of the vectors) have the same coefficients. Since such a system adopts a unique least square solution, the vector lengths are close in both embeddings. Together with the fact that the angles (cosine similarities) are close, we conclude that the two sets of vectors have similar lengths and angles, hence all the inner products should be close, meaning the PIP loss between them is small.

\item
(Necessity) The necessity part requires that if the PIP loss is close to 0, \textit{i.e.} $\|E_1E_1^T-E_2E_2^T\|\approx 0$, then $E_2\approx E_1T$ for some unitary matrix $T$. Let $E_1=UDV^T$ and $E_2=X\Lambda Y^T$ be the SVDs, we claim that we only need to show $UD\approx X\Lambda$. The reason is, if we can prove the claim, then $E_1VY^T\approx E_2$, or $T=VY^T$ is the desired unitary transformation. We prove the claim by induction, with the assumption that the singular values are distinct. Note the PIP loss equals
\[\|E_1E_1^T-E_2E_2\|=\|UD^2U^T-X\Lambda^2X^T\|\]
where $\Lambda=diag(\lambda_i)$ and $D=diag(d_i)$. Without loss of generality, suppose $\lambda_1\ge d_1$. Now let $x_1$ be the first column of $X$, namely, the singular vector corresponding to the largest singular value $\lambda_{1}$. Regard $EE^T-FF^T$ as an operator, we have
\begin{align*}
\|E_2E_2^T x_1\|-\|E_1E_1^T x_1\|&\le 
\|(E_1E_1^T-E_2E_2^T) x_1\|\\
&\le \|E_1E_1^T-E_2E_2^T\|_{op}\\
&\le \|E_1E_1^T-E_2E_2^T\|_{F}
\end{align*}
Now, notice
\[\|E_2E_2^Tx_1\|= \|X\Lambda^2X^Tx_1\|=\lambda_1^2,\]
\begin{equation}
\|E_1E_1^Tx_1\| =\|UD^2U^Tx_1\|=\sum_{i=1}^n d_i^2\langle u_i,x_1\rangle\le d_1^2
\label{eq:pf1}
\end{equation}
So $0\le \lambda_1^2-d_1^2\le \|E_1E_1^T-E_2E_2^T\|\approx 0$. As a result, we have
\begin{enumerate}
\item $d_1\approx\lambda_1$
\item $u_1\approx x_1$, in order to achieve equality in equation (\ref{eq:pf1})
\end{enumerate}
This argument can then be repeated using the Courant-Fischer minimax characterization for the rest of the singular values. As a result, we showed that $UD\approx X\Lambda$, and hence the embedding $E_2$ can indeed be obtained by applying a unitary transformation on $E_1$, or $E_2\approx E_1T$ for some unitary $T$. This establishes the (approximate) equivalence relation of $E_1\overset{U}{\sim} E_2$, which ultimately concludes that all UIPs are maintained between $E_1$ and $E_2$, including compositionality and similarity.
\end{itemize}

The above arguments establish the the equivalence between the PIP loss and the differences in UIPs. In Section \ref{sec:perturbation}, we apply matrix perturbation theory to analyze the PIP loss between the oracle embedding and the trained embedding, and discover a bias-variance trade-off in dimensionality selection.

\section{Matrix Perturbation Theory: Uncovering a Bias-Variance Trade-off in Dimensionality Selection} \label{sec:perturbation}

In Section \ref{sec:factorization} and Section \ref{sec:PIP} we introduced two key ingredients: the matrix factorization nature of many vector embedding algorithms and the PIP loss. The connection between the two is matrix perturbation theory, which reveals a bias-variance trade-off in dimensionality selection. For the best of our knowledge, this is the first result that is able to rigorously model and explain the phenomenon that there usually exists an optimal dimensionality for vector embeddings.

Recall that an embedding procedure $f_{\alpha,k}(M)=U_{\cdot,1:k}D_{1:k,1:k}^\alpha$ takes a matrix $M$ as the input and outputs the embedding matrix using SVD. It has two hyper-parameters: $\alpha$ and $k$. We regard $\alpha$ as inherent to the embedding procedure itself, and only discuss about the effect of dimensionality $k$. In practice, vector embeddings are trained from estimated matrices. For example, the co-occurrence matrix, upon which word embeddings can be constructed, is estimated from a corpus. As a result, the input $\tilde M$ to the embedding procedure $f_{\alpha, k}(\cdot)$ is noisy, where the source of noise includes, but not limited to, data impurity and estimation error. 

Imagine there is an unobserved, clean matrix $M$ with rank $d$, which in an ideal scenario, can be accurately estimated from an infinitely-long corpus. With $M$, we get the oracle embedding $E=f_{\alpha,d}(M)$. During training, however, we observe the matrix $\tilde M=M+Z$, and obtain $\hat E=f_{\alpha,k}(\tilde M)$. It has long been an open question on how to choose the hyper-parameter $k$. Using our notations, we have already showed that one should minimize the PIP loss $\|EE^T-\hat E\hat E^T\|$ to ensure $\hat E$ is functionally similar to the ground truth $E$. The optimal dimensionality $k$ should be chosen so that the PIP loss between $\hat E$ and $E$ is minimized. However, there seems to be a natural obstacle; in order to compute the PIP loss, one needs to know the oracle $E$ to begin with! Luckily, this issue can be resolved by the unitary-invariance observation. The intuition is that the embedding matrix $E$, obtained from SVD, consists of signal magnitudes which are unitary-invariant, and signal directions which are not. As a result, the factors that determine the PIP loss are the unitary-invariant ones, namely the signal magnitudes (\textit{i.e.}, singular values). To rigorously model this intuition, we use the concept of the principal angles between subspaces. This notion was first introduced by Jordan \citep{jordan1875essai}, and it will be a central concept in our derivations.

\begin{definition}[Principal Angles]
Let $X$, $Y\in\mathbb{R}^{n\times k}$ be two orthogonal matrices, with $k\le n$. Let $UDV^T=X^TY$ be the singular value decomposition of $X^TY$. Then
\[D=\cos(\Theta)=diag(\cos(\theta_1), \cdots, \cos(\theta_k)).\]
In particular, $\Theta=(\theta_1,\cdots,\theta_k)$ are the principal angles between $\mathcal{X}$ and $\mathcal{Y}$, subspaces spanned by the columns of $X$ and $Y$.
\end{definition}
We will first introduce the following lemmas, which are classical results from matrix perturbation theory. 
\begin{lemma}\label{lemma:1}
For orthogonal matrices $X_0\in\mathbb{R}^{n\times k},Y_1\in\mathbb{R}^{n\times (n-k)}$, the SVD of their inner product equals
\[\text{SVD}(X_0^TY_1)=U_0\sin(\Theta)\tilde V_1^T\]
where $\Theta$ are the principal angles between $X_0$ and $Y_0$, the orthonormal complement of $Y_1$. The sine function is applied element-wise to the vector $\Theta$ of the principal angles.
\end{lemma}
\begin{lemma}\label{lemma:2}
Suppose $X$, $Y$ are two orthogonal matrices of $\mathbb{R}^{n\times n}$, and $k\le n$. Let $X=[X_0,X_1]$ and $Y=[Y_0,Y_1]$, where $X_0, Y_0\in\mathbb{R}^{n\times k}$ are the first $k$ columns of $X$ and $Y$ respectively. Then 
\[\|X_0X_0^T-Y_0Y_0^T\|=c\|X_0^TY_1\|\]
where $c$ is a constant depending on the norm only. $c=1$ for 2-norm and $\sqrt{2}$ for Frobenius norm.
\end{lemma}
The proof of  the lemmas are deferred to the appendix, which can also be found in \citet{stewart1990matrix}.

\citet{arora2016blog} discussed in a post about the the existence of an optimal dimensionality: ``\textit{... A striking finding in empirical work on word embeddings is that there is a sweet spot for the dimensionality of word vectors: neither too small, nor too large}" \footnote{\url{http://www.offconvex.org/2016/02/14/word-embeddings-2/}}. He proceeded by discussing two possible explanations: low dimensional projection (like the Johnson-Lindenstrauss Lemma) and the standard generalization theory (like the VC dimension), and pointed out why neither is sufficient for modeling this phenomenon. Our research shows that the existence of an optimal dimensionality can be attributed to a bias-variance trade-off, which we will explicitly formulate in Section \ref{sec:trade-off}, \ref{generic} and \ref{sec:upper_bound}.  Assume the optimal dimensionality is $k^*$. When $k<k^*$, the embedding model is not expressive enough, leading to a high bias. When $k>k^*$, the embedding model will fit the noise, leading to a high variance. However, this simple bias-variance trade-off on embedding dimensionality has never been formally discussed before, due to the lack of analytic-friendly embedding distance metric. Equipped with the PIP loss, we give a mathematical characterization of the bias-variance trade-off using matrix perturbation theory.

Lemma \ref{lemma:1} and \ref{lemma:2} provide handles for analyzing the PIP loss using matrix perturbation theory, which involves a few steps. First, we transform the PIP loss into sums of inner products through Lemma \ref{lemma:2} and Theorem \ref{theorem:2} (introduced later). Then we apply Lemma \ref{lemma:1}, which turns them into principal angles between spaces. We develop matrix perturbation tools that approximate the principal angles, using quantities that can be readily estimated from data (specifically, noise standard deviation and the spectrum of signal matrix). Indeed, by inspecting the left hand side of Lemma \ref{lemma:2}, we already see its superficial resemblance of our central concept: the PIP loss $\|EE^T-\hat E\hat E^T\|$. However, Lemma \ref{lemma:2} is applicable if both $E$ and $\hat{E}$ have orthonormal columns, which is the case only when $\alpha=0$ in $E=U_{\cdot,1:d}D_{1:d,1:d}^\alpha$. Nevertheless, we will first take a quick look at this scenario before going to the generic case, as it provides the right intuition. 

\subsection{Bias Variance Trade-off: $\alpha=0$}\label{sec:trade-off}
The theorem below captures the trade-off when $\alpha=0$:
\begin{theorem}\label{theorem:1}
Let $E\in\mathbb{R}^{n\times d}$ and $\hat E\in\mathbb{R}^{n\times k}$ be the oracle and estimated embeddings, where $k\le d$. Assume both have orthogonal columns. Then
\[\|\text{PIP}(E)-\text{PIP}(\hat E)\|=\sqrt{d-k+2\|\hat{E}^TE^\perp\|^2}\]
\end{theorem}
\begin{proof}
To simplify notation, denote $X_0=E$, $Y_0=\hat E$, and let $X=[X_0,X_1]$, $Y=[Y_0,Y_1]$ be the complete $n$ by $n$ orthogonal matrices. Since $k\le d$, we can further split $X_0$ into $X_{0,1}$ and $X_{0,2}$, where the former has $k$ columns and the latter $d-k$.
\begin{align*}
\|EE^T-\hat E\hat E^T\|^2=
&\|X_{0,1}X_{0,1}^T-Y_0Y_0^T+X_{0,2}X_{0,2}^T\|^2\\
=&\|X_{0,1}X_{0,1}^T-Y_0Y_0^T\|^2+\|X_{0,2}X_{0,2}^T\|^2+2\langle X_{0,1}X_{0,1}^T-Y_0Y_0^T, X_{0,2}X_{0,2}^T\rangle\\
\overset{(a)}{=}&2\|Y_0^T[X_{0,2},X_1]\|^2+d-k-2\langle  Y_0Y_0^T, X_{0,2}X_{0,2}^T\rangle\\
=&2\|Y_0^TX_{0,2}\|^2+2\|Y_0^TX_1\|^2+d-k-2\langle  Y_0Y_0^T, X_{0,2}X_{0,2}^T\rangle\\
=&d-k+2\|Y_0^TX_1\|^2=d-k+2\|\hat E^TE^\perp\|^2
\end{align*}
where in equality (a) we used Lemma \ref{lemma:2}.
\end{proof}
The observation is that the right hand side of Theorem \ref{theorem:1} now consists of two parts, which we identify as bias and variance. The first part, $d-k$ is how much we sacrifice by taking $k$ less than the actual $d$, as the signals corresponding to the rest $d-k$ singular values are lost. However, $\|\hat E^TE^\perp\|$ increases as $k$ increases, as the noise perturbs the subspaces, and the dimensions corresponding to singular values closer to 0 are more prone to perturbation. The PIP loss-minimizing $k$ lies between 0 and the rank $d$ of the ground-truth signal matrix $M$.

\subsection{Bias Variance Trade-off: $0<\alpha\le 1$}\label{generic}
As pointed out by several papers \citep{caron2001experiments,bullinaria2012extracting,turney2012domain,levy2014neural}, the generic embeddings can be characterized by $E=U_{1:k,\cdot}D^\alpha_{1:k,1:k}$ for some $\alpha\in[0,1]$. The difficulty for $\alpha>0$ is that the columns are no longer orthonormal for both $E$ and $\hat E$, breaking assumptions needed in matrix perturbation theory. We develop an approximation method where the lemmas are applied in a telescoping fashion, which is a new technique to the best of our knowledge. The proof of the theorem is deferred to the appendix.
\begin{theorem}[telescoping theorem]\label{theorem:2}
Suppose the ground truth and the estimated embeddings are $E=U_{\cdot,1:d}D_{1:d,1:d}^\alpha\in\mathbb{R}^{n\times d}$, and $\hat E=\tilde U_{\cdot,1:k}{\tilde D}_{1:k,1:k}^\alpha\in\mathbb{R}^{n\times k}$ respectively, for some $k\le d$. Let $D=diag(\lambda_i)$ and $\tilde D=diag(\tilde \lambda_i)$, then
\begin{align*}
\|\text{PIP}(E)-\text{PIP}(\hat E)\|\le&\sqrt{\sum_{i=k+1}^d \lambda_i^{4\alpha}}+\sqrt{\sum_{i=1}^k (\lambda_i^{2\alpha}-\tilde\lambda_{i}^{2\alpha})^2}+\sqrt{2}\sum_{i=1}^k (\lambda_i^{2\alpha}-\lambda_{i+1}^{2\alpha})\|\tilde U_{\cdot,1:i}^T U_{\cdot,i:n}\|
\end{align*}
\end{theorem}
The noise not only affects the signal directions as when $\alpha=0$, but also the signal magnitudes. As before, the three terms in Theorem \ref{theorem:2} can be characterized into bias and variances. The first term is the bias as we lose part of the signal by choosing $k\le d$. The second term is the variance on the \textit{signal magnitudes}, and the third term is the variance on the \textit{signal directions}.

It is interesting to compare Theorem \ref{theorem:1} and \ref{theorem:2} when $\alpha=0$, to which both apply. Taking $\alpha\rightarrow 0$ in Theorem \ref{theorem:2} gives the upper bound $\|EE^T-\hat E\hat E^T\|\le \sqrt{d-k}+\sqrt{2}\|\hat{E}^TE^\perp\|$ while Theorem \ref{theorem:1} states that $\|EE^T-\hat E\hat E^T\|=\sqrt{d-k+2\|\hat{E}^TE^\perp\|^2}$. We observe that Theorem \ref{theorem:1} provides a tighter upper bound (which is also exact). Indeed, we traded some tightness for generality, to deal with $\alpha>0$. This looseness is mainly from the Minkovski's inequality in the proof, which is technically needed in the absence of orthonormality when $\alpha>0$.

Note that we have not achieved our goal yet, which is to obtain an approximation of the bias-variance trade-off that is independent of the signal directions $U$. The unitary-invariance comes into play after taking expectation. It turns out that IID noise does not bias against any particular direction, and this fact will give a direction-independent approximation.

\subsection{The Trade-off Between the Signal Spectrum and the Noise Power}
\label{sec:upper_bound}
In section \ref{sec:trade-off} and \ref{generic}, we discussed how to transform the PIP loss into bias and variance terms. We now present the main theorem that contains only the signal spectrum and noise standard deviation. This theorem shows that the bias-variance trade-off is fundamentally between the signal spectrum and noise power, which is similar to the signal-to-noise ratio notion in information theory and signal processing. We also briefly discuss how to estimate the quantities from data.
\begin{theorem}[Main theorem]\label{theorem:main}
Suppose $\tilde M=M+Z$, where $M$ is the symmetric signal matrix with spectrum $\{\lambda_i\}_{i=1}^d$, and $Z$ is a symmetric random noise matrix with iid, zero mean, variance $\sigma^2$ entries. For any $0\le \alpha \le 1$ and $k\le d$, let the oracle and estimated embeddings be
\[ E=U_{\cdot,1:d}D_{1:d,1:d}^\alpha,\ \hat E=\tilde U_{\cdot,1:k}\tilde D_{1:k,1:k}^\alpha\]
where
$M=UDV^T$, $\tilde M=\tilde U\tilde D\tilde V^T$
are the SVDs. Then
\begin{enumerate}
\item When $\alpha=0$,
\begin{equation*}
\mathbb E\|EE^T-\hat E\hat E^T\|\le\sqrt{d-k+2\sigma^2\sum_{r\le k,\\ s>d}(\lambda_{r}-\lambda_{s})^{-2}}
\end{equation*}
\item When $0<\alpha\le 1$,
\begin{align*}
\mathbb E\|EE^T-\hat E\hat E^T\|\le &\sqrt{\sum_{i=k+1}^d \lambda_i^{4\alpha}}+2\sqrt{2n}\alpha\sigma\sqrt{\sum_{i=1}^k \lambda_i^{4\alpha-2}}+\sqrt{2}\sum_{i=1}^k (\lambda_i^{2\alpha}-\lambda_{i+1}^{2\alpha})\sigma\sqrt{\sum_{r\le i<s}(\lambda_{r}-\lambda_{s})^{-2}}
\end{align*}
\end{enumerate}
\end{theorem}

\begin{proof}
We sketch the proof for part 2, as the proof of part 1 can be done with the same arguments. We start by taking expectation on both sides of the telescoping Theorem \ref{theorem:2}:
\begin{align*}
\mathbb E\|EE^T-\hat E\hat E^T\|\le&\sqrt{\sum_{i=k+1}^d \lambda_i^{4\alpha}}+\mathbb E\sqrt{\sum_{i=1}^k (\lambda_i^{2\alpha}-\tilde\lambda_{i}^{2\alpha})^2}+\sqrt{2}\sum_{i=1}^k (\lambda_i^{2\alpha}-\lambda_{i+1}^{2\alpha})\mathbb E\|\tilde U_{\cdot,1:i}^T U_{\cdot,i:n}\|,
\end{align*}
The first term involves the spectrum, which is the same after taking expectation. The second term can be upper bounded using Lemma \ref{lemma:bias1} below, derived from Weyl's theorem. We state the lemma, and leave the proof to the appendix.
\begin{lemma}\label{lemma:bias1}
With conditions in Theorem \ref{theorem:main},
\begin{align*}
\mathbb E\sqrt{\sum_{i=1}^k (\lambda_i^{2\alpha}-\tilde\lambda_{i}^{2\alpha})^2}\le 2\sqrt{2n}\alpha\sigma\sqrt{\sum_{i=1}^k \lambda_i^{4\alpha-2}}
\end{align*}
\end{lemma}
For the last term, we use the Sylvester operator technique by \citet{stewart1990matrix}. Our result is presented in Theorem \ref{theorem:T}, the proof of which is discussed in the appendix.
\begin{theorem}\label{theorem:T}
For two matrices $M$ and $\tilde M=M+Z$, denote their SVDs as $M=UDV^T$ and $\tilde M=\tilde U\tilde D \tilde V^T$. Write the singular matrices in block form as $U=[U_0,U_1]$, $\tilde U=[\tilde U_0,\tilde U_1]$, and similarly partition $D$ into diagonal blocks $D_0$ and $D_1$. If the spectrum of $D_0$ and $D_1$ has separation
\[\delta_k\overset{\Delta}{=}\min_{1\le i\le k,k< j\le n}\{\lambda_{i}-\lambda_{j}\}=\lambda_k-\lambda_{k+1}>0,\]
and $Z$ has iid, zero mean entries with variance $\sigma^2$, then
\[
\mathbb E\|\tilde U_1^TU_0\|\le\sigma\sqrt{\sum_{\substack{1\le i\le k<j\le n}}(\lambda_{i}-\lambda_{j})^{-2}}\]
\end{theorem}
Now, collect results in Lemma \ref{lemma:bias1} and Theorem \ref{theorem:T}, we obtain an upper bound for the PIP loss:
\begin{align*}
\mathbb E\|EE^T-\hat E\hat E^T\|\le 
&\sqrt{\sum_{i=k+1}^d \lambda_i^{4\alpha}}+2\sqrt{2n}\alpha\sigma\sqrt{\sum_{i=1}^k \lambda_i^{4\alpha-2}}+\sqrt{2}\sum_{i=0}^k (\lambda_i^{2\alpha}-\lambda_{i+1}^{2\alpha})\sigma\sqrt{\sum_{r\le i<s}(\lambda_{r}-\lambda_{s})^{-2}}
\end{align*}
which completes the proof. The symmetricity assumptions on $M$ and $Z$ simplify the proof, and we note that most matrices in NLP are indeed symmetric. Non-symmetric matrices can be symmetrized with $M^TM$ or the Jordan-Wielandt symmetrization technique. 
\end{proof}
As a final remark, our result captures the trade-off between $M$'s spectrum $\{\lambda_i\}_{i=1}^d$ and the noise variance $\sigma^2$. We briefly discuss how to estimate them from data.

\subsection{Noise Estimation}\label{sec:noise_est}
We note that for most NLP tasks, the matrices are obtained by co-occurrence counting or transformations of counting, including taking log or normalization. This observation holds for word embeddings (PPMI, Word2Vec, GloVe) and document embeddings (TF, TF-IDF), which are all based on co-occurrence statistics. No matter the transformation is logarithm or normalization, the dominating noise term is the additive noise introduced in the counting procedure. We use a count-twice trick: we randomly split the data into two equally large subsets, and get matrices $\tilde{M}_1=M+Z_1$, $\tilde{M}_2=M+Z_2$ in $\mathbb R^{m\times n}$, where $Z_1,Z_2$ are now two independent copies of noise with variance $2\sigma^2$. Now, $\tilde{M}_1-\tilde{M}_2=Z_1-Z_2$ is a random matrix with with zero mean and variance $4\sigma^2$. Our estimator is the sample standard deviation, a consistent estimator:
\[\hat\sigma=\frac{1}{2\sqrt{mn}}\|\tilde{M}_1-\tilde{M}_2\|_F\]
\subsection{Spectral Estimation}\label{sec:spectrum_est}
Spectral estimation is a well studied subject in statistical literature \citep{cai2010singular,candes2009exact}. There are different spectrum estimators with different assumptions, which give different bounds. The estimator we use here is the well-established universal singular value thresholding (USVT) estimator proposed by \citet{chatterjee2015matrix}.
\[\hat\lambda_i=(\tilde{\lambda}_i-2\sigma\sqrt{n})_+,\]
where $\tilde{\lambda}_i$ is the $i$-th empirical singular value and $\sigma$ is the noise standard deviation. This estimator is shown to be minimax optimal \citep{chatterjee2015matrix}. Beside the estimators' theoretical performance guarantee, the simplicity of their forms bring extra computational benefits, as the estimated spectrum can be obtained from the empirical spectrum by taking a simple soft-thresholding.

\section{New Discoveries and Methodologies}
\label{sec:discoveries}
In Sections \ref{sec:PIP} and \ref{sec:perturbation}, we discussed the PIP loss, its bias-variance trade-off and the matrix perturbation theory analysis. The PIP loss and its associated analytical framework not only give insights to the empirically observed phenomena like the existence of a ``sweet dimensionality''. In this section, we present three new discoveries. The first discovery is that vector embedding's robustness to over-parameterization increases with respect to the exponent $\alpha$, and symmetric algorithms (including skip-gram and GloVe) are robust to over-parametrization. The second one is the forward stability of vector embeddings. The last discovery is that the optimal dimensionality can be explicitly found by optimizing over the bias-variance trade-off curve.\footnote{All codes are available at https://github.com/aaaasssddf/PIP-experiments}

\subsection{Robustness to Over-parametrization with Respect to $\alpha$}
\label{sec:robust_to_overfitting}
The bias-variance trade-off reveals the sub-optimality for dimensionalities that are either too small or too large. Over-fitting occurs when the dimensionality is too large, or equivalently when there is over-parametrization. We show in this section that over-parametrization affects embeddings with different $\alpha$ in different ways. Specifically, embedding procedures with large $\alpha$ are less affected by the downside from over-parametrization, compared to procedures with small $\alpha$.

We start with a quantitative analysis of the effect of dimensionality on Theorem \ref{theorem:main}. As $k$ increases, bias 
\begin{small}
$\sqrt{\sum_{i=k}^d \lambda_i^{4\alpha}}$
\end{small} 
decreases. It can be viewed as a \textit{zero-order} term, for it consists of arithmetic means of singular values themselves, which are mostly dominated by the large ones. As a result, as long as $k$ is large enough (say, include singular values that has more than half of the total energy), continuing to increase $k$ only has marginal effect on this term.

However, the second and third term, the variances, demonstrate \textit{first order} effect, which contains the difference of the singular values, \textit{i.e.} singular gaps. Specifically, the term
\begin{small}
$2\sqrt{2n}\alpha\sigma\sqrt{\sum_{i=1}^k \lambda_i^{4\alpha-2}}$
\end{small}
increases with respect to $k$, at rate dominated by $\lambda_k^{2\alpha-1}$. Not as obvious as the second term, the last term also increases at the same rate. Note in
\begin{small}
$(\lambda_k^{2\alpha}-\lambda_{k+1}^{2\alpha})\sqrt{\sum_{r\le k<s}(\lambda_{r}-\lambda_{s})^{-2}}$,
\end{small}
 the square root term is dominated by $(\lambda_{k}-\lambda_{k+1})^{-1}$ which gets closer to infinity as $k$ gets larger. However, $\lambda_k^{2\alpha}-\lambda_{k+1}^{2\alpha}$ can potentially offset this first order effect. Specifically, consider the smallest non-zero singular value $\lambda_d$, whose gap to 0 is $\lambda_d$. Note when the two terms are multiplied,
\[(\lambda_d^{2\alpha}-0)\delta_d^{-1}= \lambda_d^{2\alpha-1},\]
which shows the two variance terms have the same rate of $\lambda_k^{2\alpha-1}$. When $k$ increases, $\lambda_k$ goes to zero. As a result, the term $\lambda_k^{2\alpha-1}$ increases as $\alpha$ decreases, and can be unbounded for $\alpha<0.5$. Since the upper bound is minimax tight for infinitesimal perturbation \citep{vu2011singular}, it is a good indicator for the actual sensitivity. In other words, for smaller $\alpha$, the upper bound grows larger, indicating that the model will be more sensitive to over-parametrization.
\begin{figure}[htb]
\centering
\hspace*{\fill}%
\begin{subfigure}[b]{0.32\textwidth}
\includegraphics[width=\textwidth]{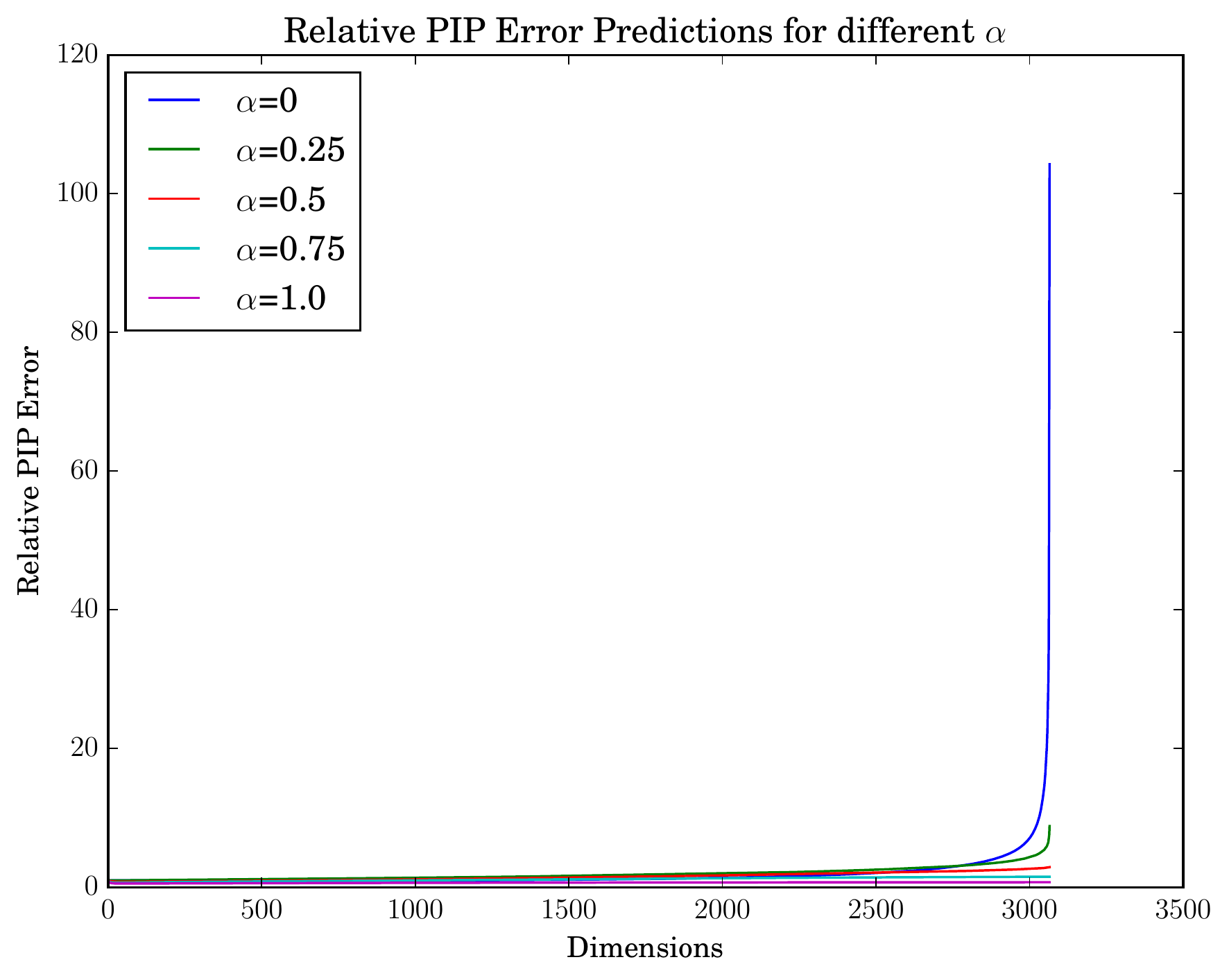}
            \caption[]%
            {{\small Theorem \ref{theorem:main}}} 
            \label{subfig:pip_bound_theory}
        \end{subfigure}
        \begin{subfigure}[b]{0.32\textwidth}           \includegraphics[width=\textwidth]{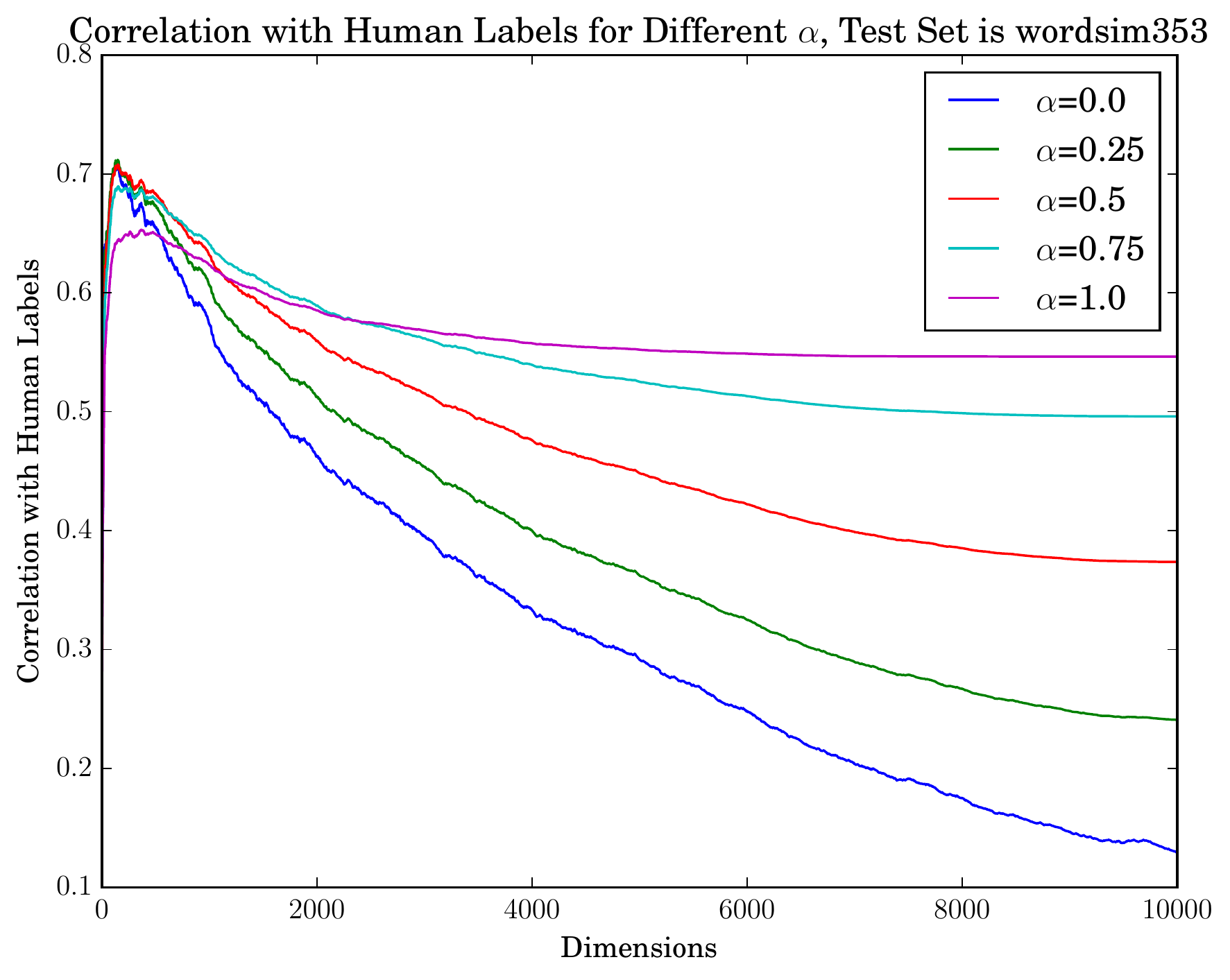}
            \caption[]%
            {{\small WordSim353}}    
            \label{subfig:ws353}
        \end{subfigure}
\begin{subfigure}[b]{0.32\textwidth}           \includegraphics[width=\textwidth]{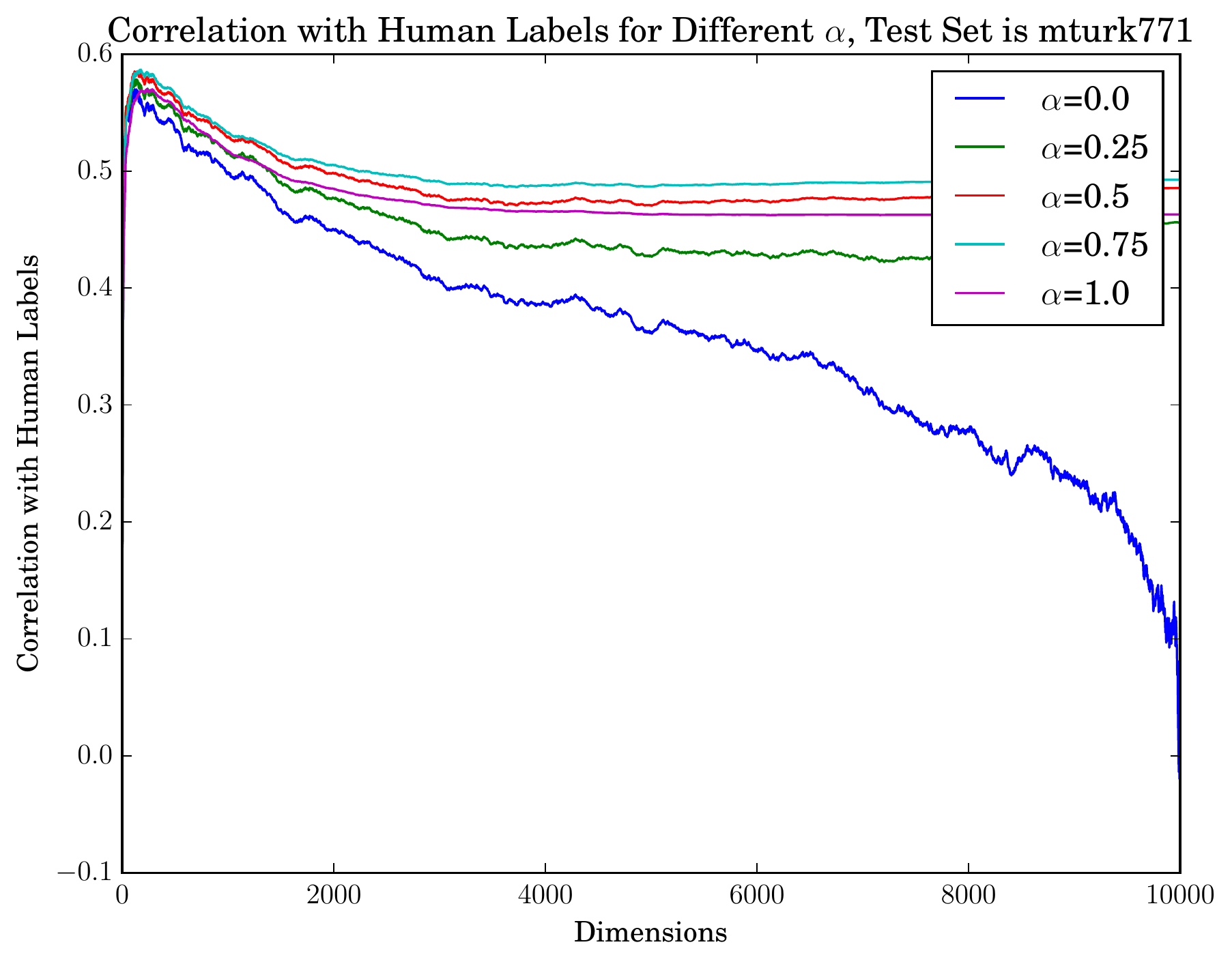}
            \caption[]%
            {{\small Mturk771}}    
            \label{subfig:mturk771}
        \end{subfigure}
        \hspace*{\fill}%
\caption{Sensitivity to Over-parametrization}
\label{fig:corr}
\end{figure}

In Figure \ref{subfig:pip_bound_theory}, we plotted the bounds of Theorem \ref{theorem:main} for different $\alpha$, using the spectrum and noise of the PPMI matrix obtained from the Text8 corpus. It shows that the PIP loss upper bound increases as $\alpha$ decreases. Figure \ref{subfig:ws353} and \ref{subfig:mturk771} compares word embeddings obtained by factorizing PPMI matrix on the Text8 corpus, evaluated on human labeled similarity dataset \citep{wordsim353,mturk771}. The outcomes validate the theory: indeed the performance drop due to over-parametrization is more significant for smaller $\alpha$.

A corollary is about Word2Vec \citep{mikolov2013exploiting} and GloVe \citep{pennington2014glove}, both of which are implicitly doing a symmetric ($\alpha=0.5$) matrix factorization. Our theory predicts that they should be robust to over-parametrization. We verify this by training skip-gram and GloVe models on the Text8 corpus. Figure \ref{fig:word2vec} shows Word2Vec with extreme over-parametrization (up to $k=10000$) still performs within 80\% to 90\% of the optimal performance, for both top-4 hit rate of Google analogy test \citep{mikolov2013efficient} and similarity tests \citep{wordsim353,mturk771}. Similar phenomenon is observed for GloVe as well. Performances drop with respect to over-parametrization, but the drops are insignificant. It's worth noting that while GloVe is comparable to Word2Vec on similarity tests, it is considerably worse on the analogy test.
\begin{figure}[htb]
\centering
\hspace*{\fill}%
\begin{subfigure}[b]{0.32\textwidth}
\includegraphics[width=\textwidth]{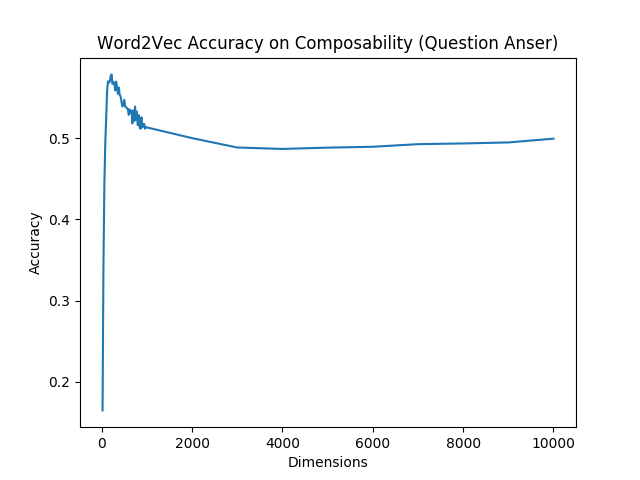}
            \caption[]%
            {{\small Google Analogy}} 
            \label{subfig:w2v_compositionality}
        \end{subfigure}
        \hfill
\begin{subfigure}[b]{0.32\textwidth}           \includegraphics[width=\textwidth]{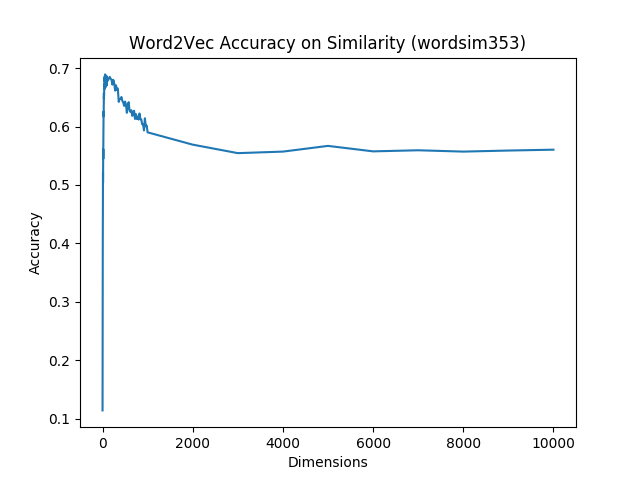}
            \caption[]%
            {{\small WordSim353}} 
            \label{subfig:w2v_wordsim353}
\end{subfigure}
        \hspace*{\fill}%
\begin{subfigure}[b]{0.32\textwidth}           \includegraphics[width=\textwidth]{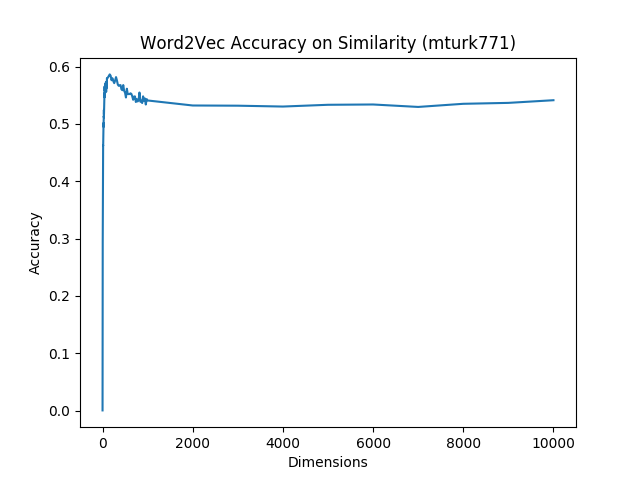}
            \caption[]%
            {{\small Mturk771}} 
            \label{subfig:w2v_mturk771}
\end{subfigure}
\caption{Word2Vec Performance w.r.t. Dimensions}
\label{fig:word2vec}
\end{figure}

\begin{figure}[htb]
\centering
\hspace*{\fill}%
\begin{subfigure}[b]{0.32\textwidth}
\includegraphics[width=\textwidth]{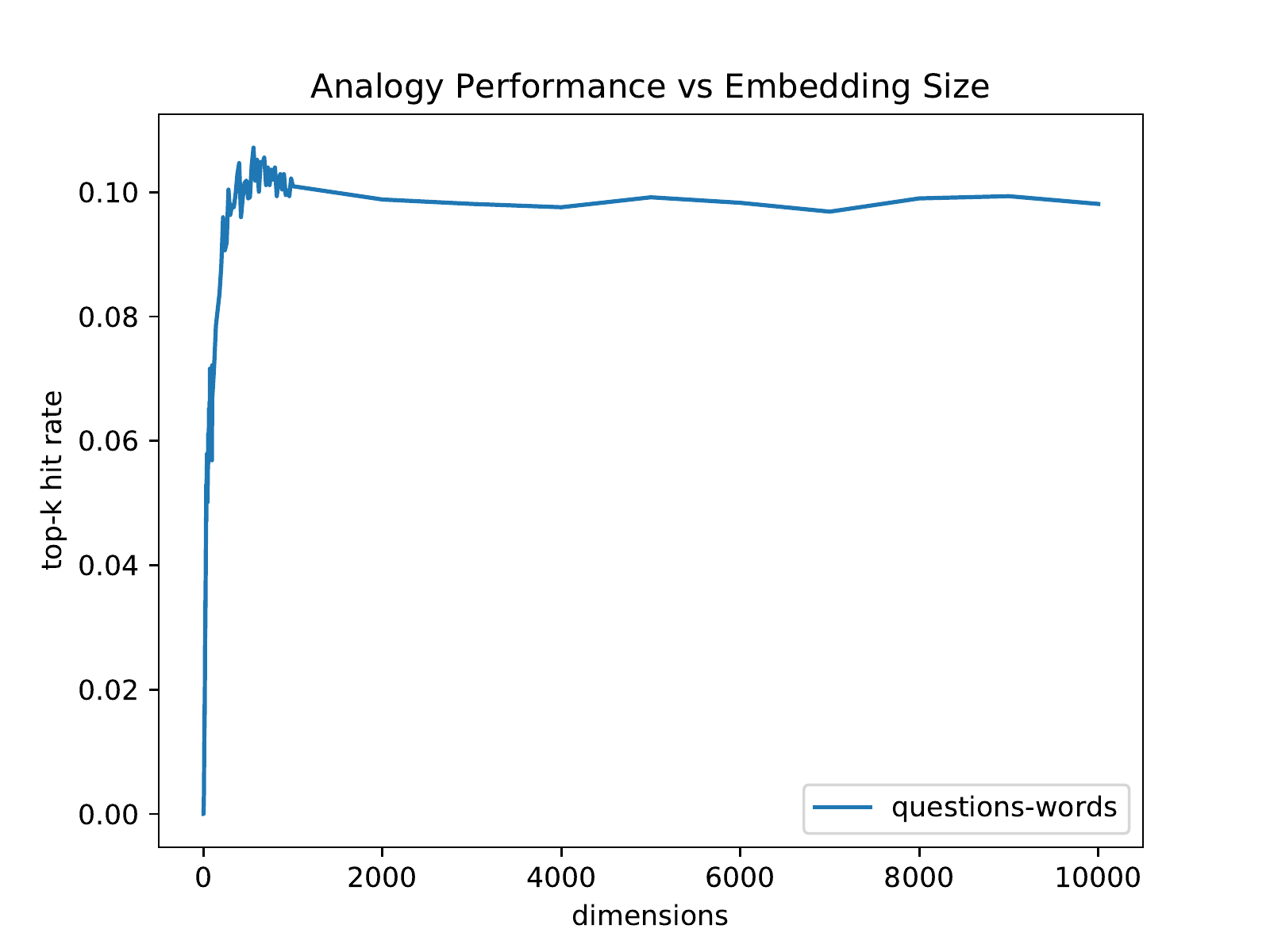}
            \caption[]%
            {{\small Google Analogy}} 
            \label{subfig:glove_compositionality}
        \end{subfigure}
        \hfill
\begin{subfigure}[b]{0.32\textwidth}           \includegraphics[width=\textwidth]{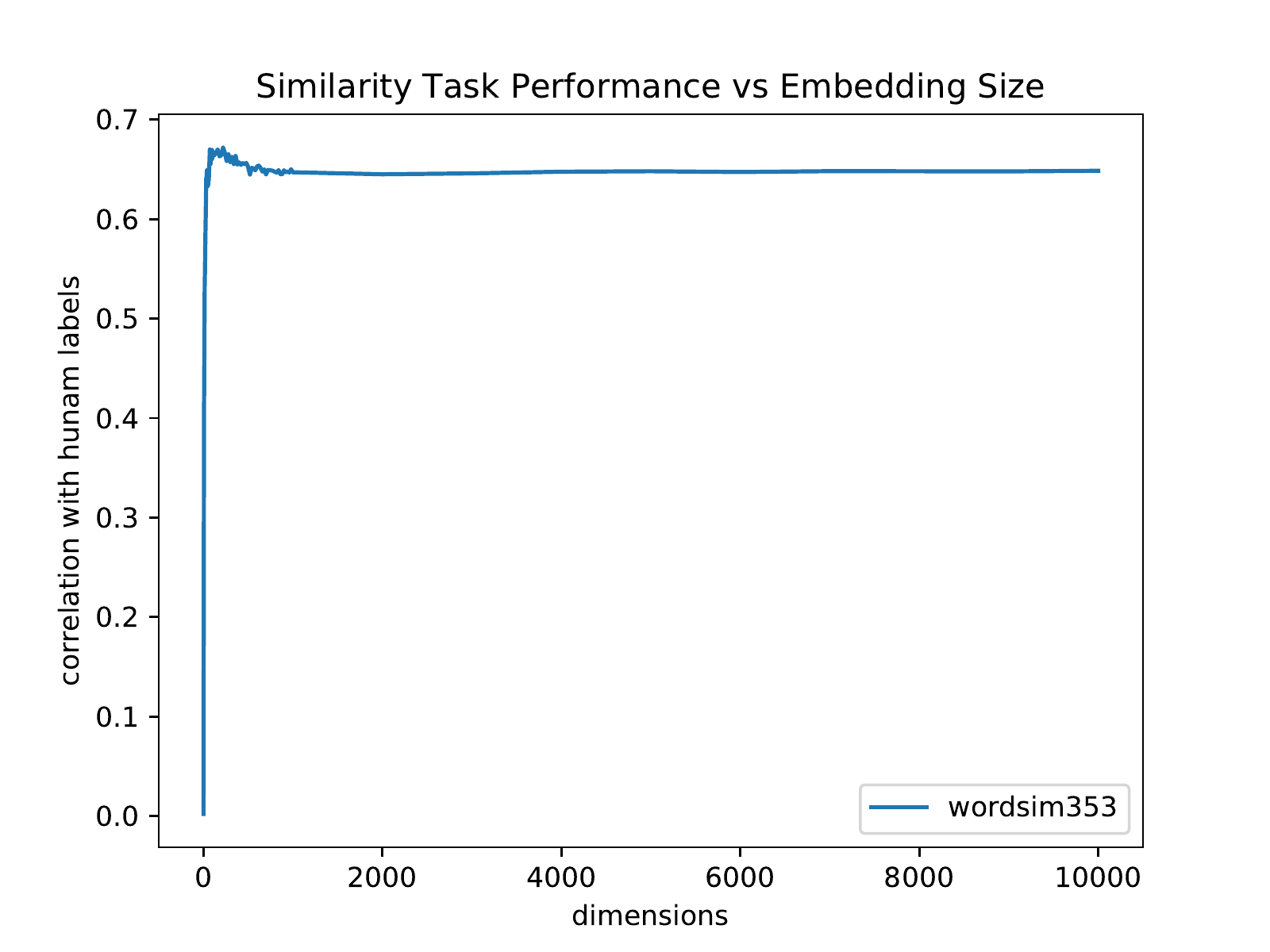}
            \caption[]%
            {{\small WordSim353}} 
            \label{subfig:glove_wordsim353}
\end{subfigure}
\hspace*{\fill}%
\begin{subfigure}[b]{0.32\textwidth}           \includegraphics[width=\textwidth]{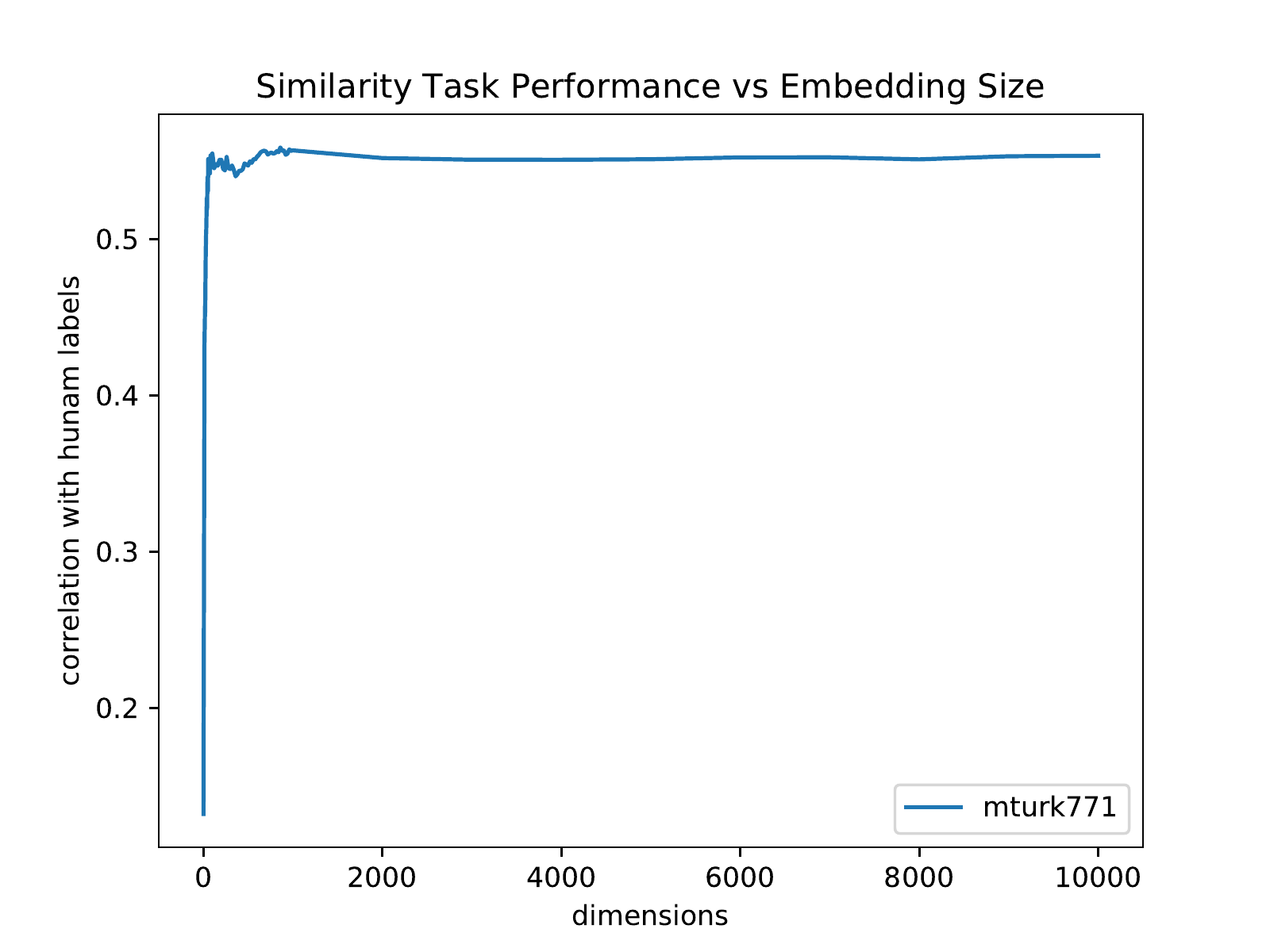}
            \caption[]%
            {{\small Mturk771}} 
            \label{subfig:glove_mturk771}
\end{subfigure}
\caption{GloVe Performance w.r.t. Dimensions}
\label{fig:glove}
\end{figure}

This robustness allows dimensionality selection to slightly err on the larger side. Performance degradation from over-parametrization is not significant for embedding procedures with $\alpha\ge 0.5$. However, for practical purposes, careful dimensionality selection is still worthwhile. The training time grows linearly with respect to the dimensionality, so a 2000-dimensional model will take 10x more time to train than a 200-dimensional model. Moreover, this over-parametrization will weigh on the downstream applications, as their dimensionalities will have to increase in order to accommodate the increased embedding dimensionalities. The extra overhead may have potential side effect on the efficiency, responsiveness and scalability of the machine learning systems.

\subsection{Forward Stability of Embedding Algorithms under the PIP Loss}
Numerical stability is a central topic for algorithm analysis. An algorithm is practically usable only when it is numerically stable, and one important notion of numerical stability is the forward stability \citep{higham2002accuracy}. An algorithm $f$ is forward stable when small perturbations in the input cause small perturbations in the output; namely, $\|f(x+\delta)-f(x)\|$ is small when $\|\delta\|$ is small. Without forward stability, the result of an algorithm cannot be trusted. Errors and variations introduced in data processing, parameter mis-specification, stochastic optimization procedures, or even floating point rounding, will significantly change the output. To the best of of our knowledge, the previous literature did not touch the forward stability of embedding algorithms, although it is an important question that needs a thorough examination. We briefly discuss about numerical stability of skip-gram Word2Vec and GloVe, including how the obstacles to forward stability analysis can be resolved by using the PIP loss.

\subsubsection{Theoretical Analysis}
Recall the forward stability is a reflection of how much $f(x)$ will be changed when $x$ is perturbed by $\delta$.
There are two obstacles in this formulation for vector embeddings:
\begin{enumerate}
\item The loss metric should be able to handle embeddings of different dimensionalities. For example, if both $E_1$ and $E_2$ are of $\mathbb{R}^{n\times k}$, one can try to directly compare $\|E_1-E_2\|$ using some suitable norm. However when $E_1\in \mathbb{R}^{n\times k_1}$ and $E_2\in\mathbb{R}^{n\times k_2}$ with $k_1\ne k_2$, such direct comparisons no longer make mathematical sense.

\item The loss metric should be unitary-invariant. Recall the raw outcome of embedding algorithms can be completely different due to the unitary-invariance, as shown pictorially in Figure \ref{fig:embeddings}. For most common norms like 2-norm and Frobenius norm, $\|E_1-E_2\|$ will be large for unitarily equivalent $E_1$ and $E_2$. This means when we use non-unitary-invariant metrics to perform forward stability analysis, it is likely that the conclusion we arrive at will be negative.
\end{enumerate}

In other words, forward stability analysis should take the unitary-invariance into consideration. The norm we use should reflect the stability of the \textit{functionality} of the embeddings, and one metric readily comes to play under this criterion is the PIP distance. In other words, we should check if Word2Vec and GloVe are forward stable under the PIP distance. To show this is empirically true, we ran two passes of Word2Vec and GloVe on the Text8 corpus, getting two sets of trained embeddings, and compare the PIP distance between them. Notice that both Word2Vec and GloVe apply stochastic optimization procedures, which inevitably introduces noise during training. Figure \ref{fig:word2vec_relative_PIP} and \ref{fig:glove_relative_PIP} show the relative forward error, also known as the noise-to-signal ratio
\[\text{NSR}=\frac{\|E_1E_1^T-E_2E_2^T\|^2}{\|E_1E_1^T\|\cdot \|E_2E_2^T\|}\]

\subsubsection{Experimental Results}
In Figure \ref{fig:word2vec_relative_PIP} and \ref{fig:glove_relative_PIP}, the PIP loss matrices between two runs are shown. Figure \ref{subfig:word2vec_0} and \ref{subfig:glove_0} are embeddings with dimensionality between 1 and 100, Figure \ref{subfig:word2vec_1} and \ref{subfig:glove_1} are embeddings with dimensionality between 100 and 1000 with an increment of 20, and Figure \ref{subfig:word2vec_2} and \ref{subfig:glove_2} are embeddings of dimensionality between 1000 and 10000 with an increment of 1000.

\begin{figure}[htb]
\centering
\hspace*{\fill}%
\begin{subfigure}[b]{0.32\textwidth}
\includegraphics[width=\textwidth]{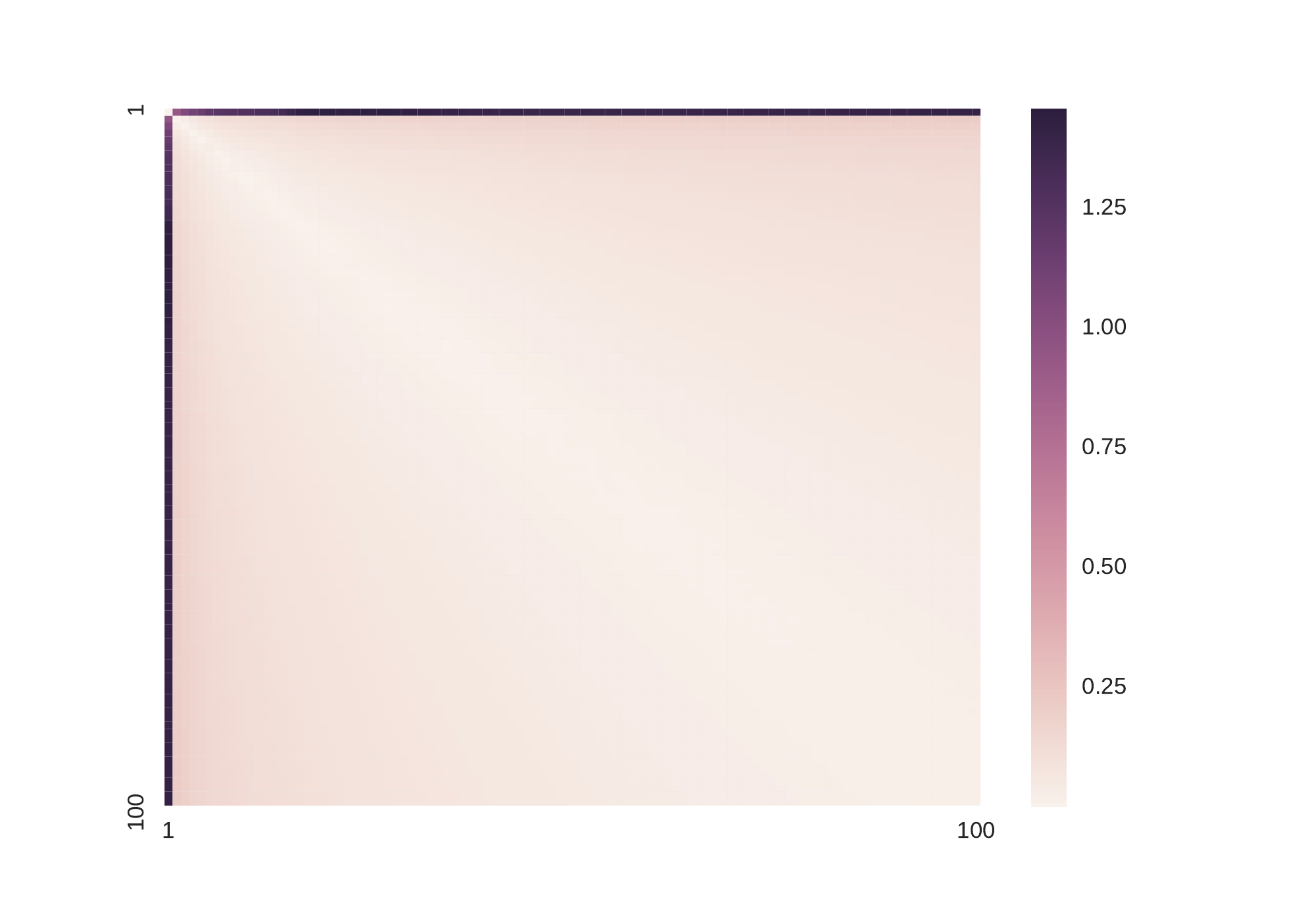}
            \caption[]%
            {{\small Dimension 1-100}} 
            \label{subfig:word2vec_0}
        \end{subfigure}
        \hfill
\begin{subfigure}[b]{0.32\textwidth}           \includegraphics[width=\textwidth]{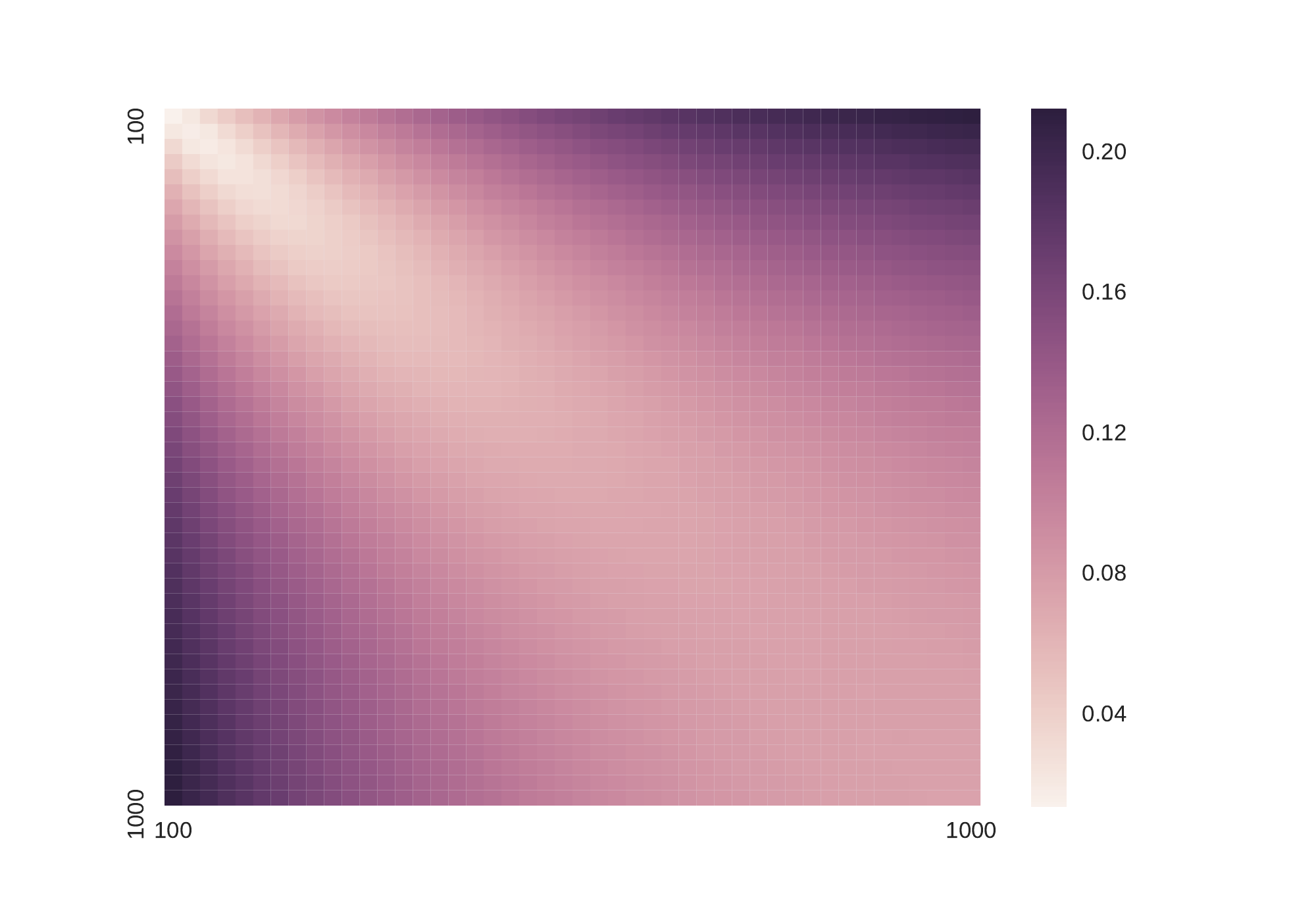}
            \caption[]%
            {{\small Dimension 100-1000}} 
            \label{subfig:word2vec_1}
\end{subfigure}
        \hspace*{\fill}%
\begin{subfigure}[b]{0.32\textwidth}           \includegraphics[width=\textwidth]{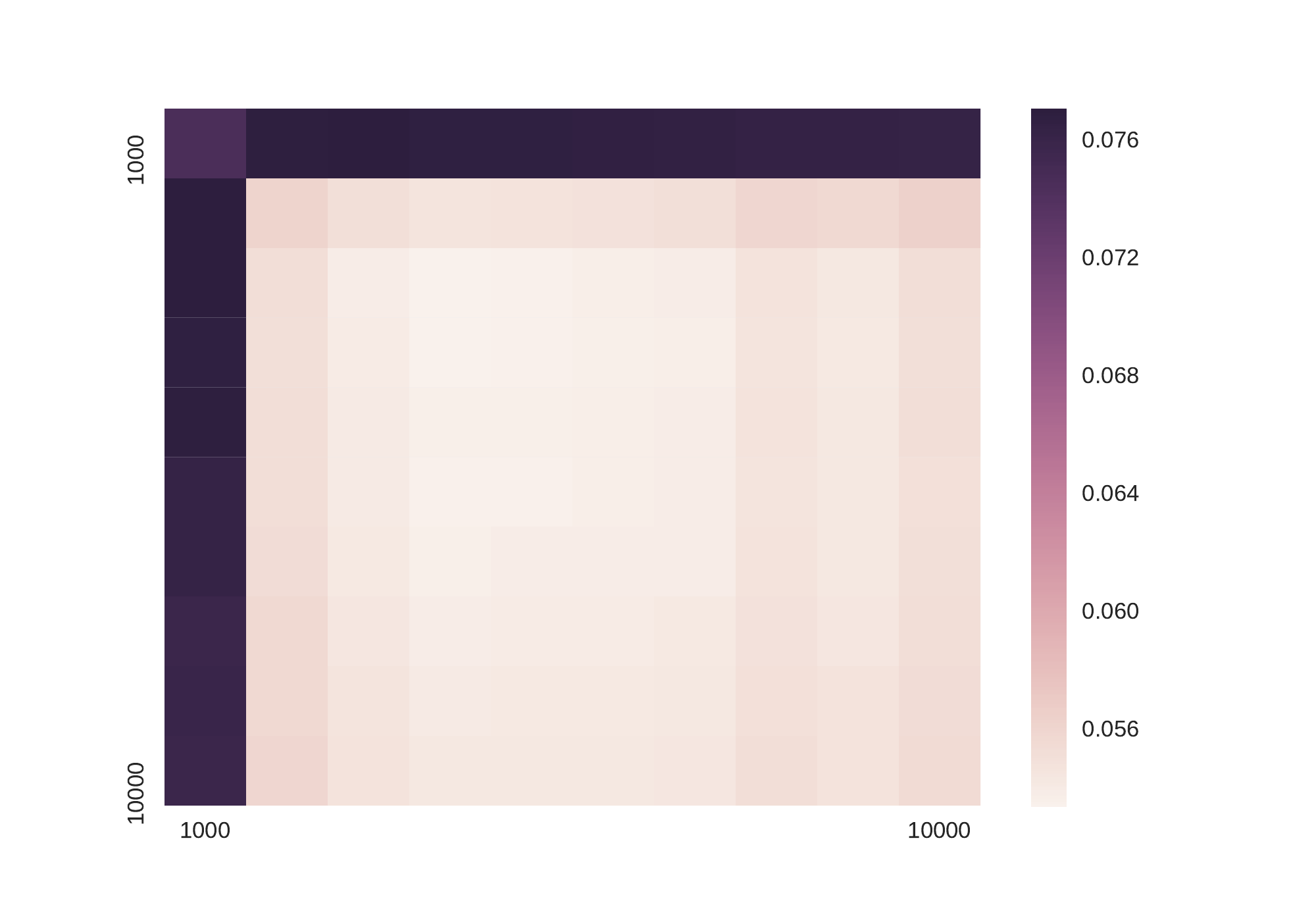}
            \caption[]%
            {{\small Dimension 1000-10000}} 
            \label{subfig:word2vec_2}
\end{subfigure}
\caption{Word2Vec Relative PIP Distance for Different Dimensionalities}
\label{fig:word2vec_relative_PIP}
\end{figure}
\begin{figure}[htb]
\centering
\hspace*{\fill}%
\begin{subfigure}[b]{0.32\textwidth}
\includegraphics[width=\textwidth]{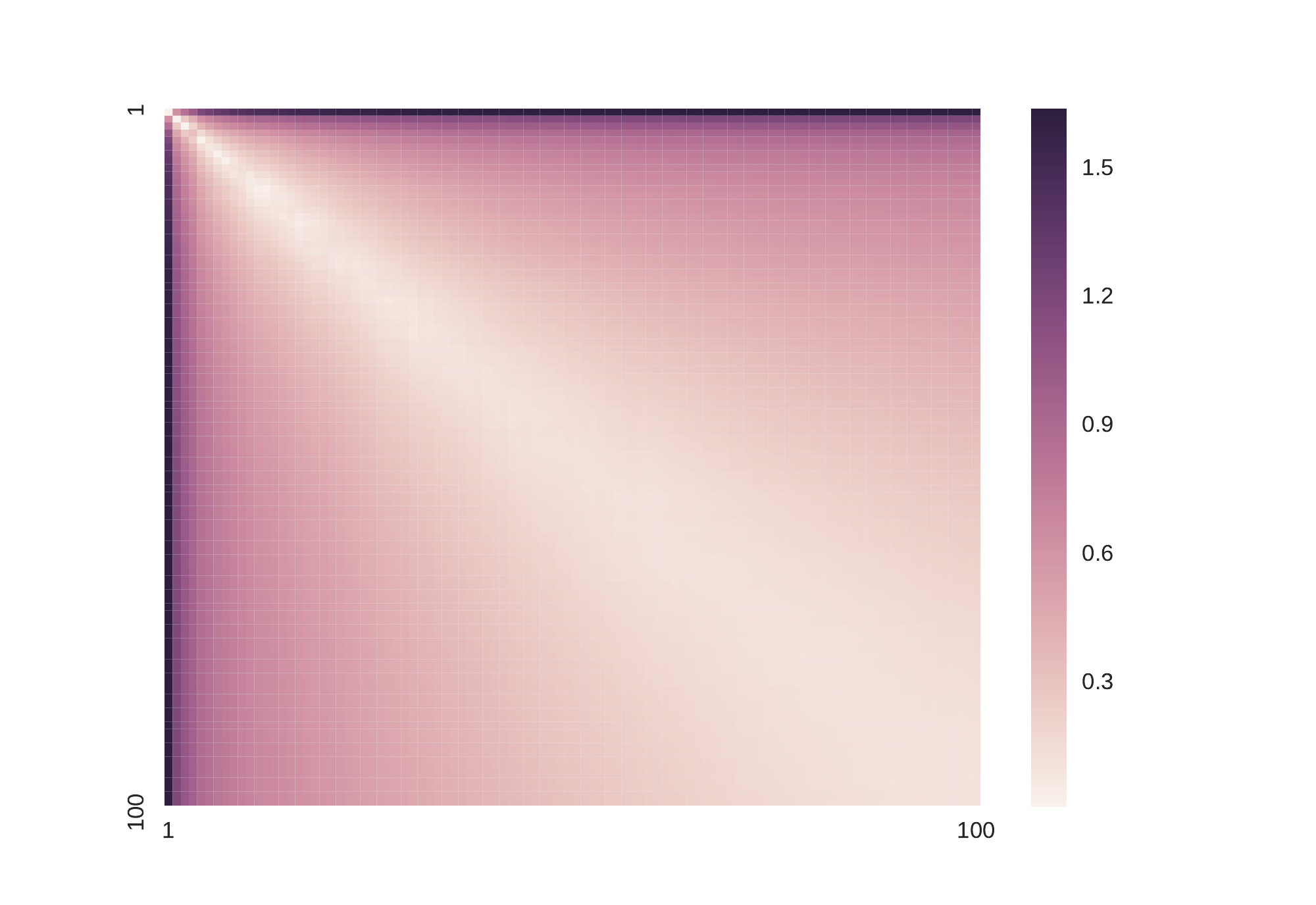}
            \caption[]%
            {{\small Dimension 1-100}} 
            \label{subfig:glove_0}
        \end{subfigure}
        \hfill
\begin{subfigure}[b]{0.32\textwidth}           \includegraphics[width=\textwidth]{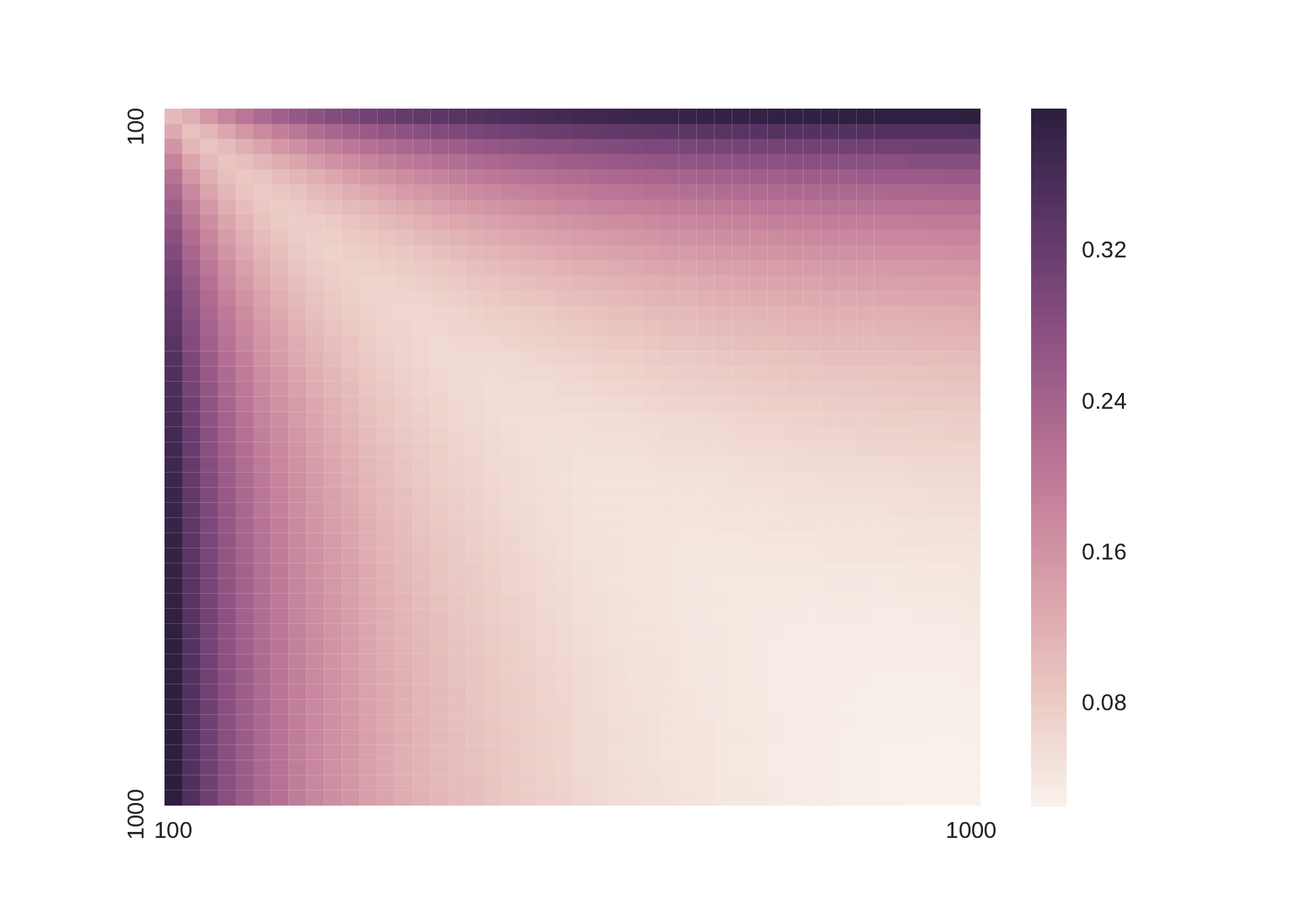}
            \caption[]%
            {{\small Dimension 100-1000}} 
            \label{subfig:glove_1}
\end{subfigure}
        \hspace*{\fill}%
\begin{subfigure}[b]{0.32\textwidth}           \includegraphics[width=\textwidth]{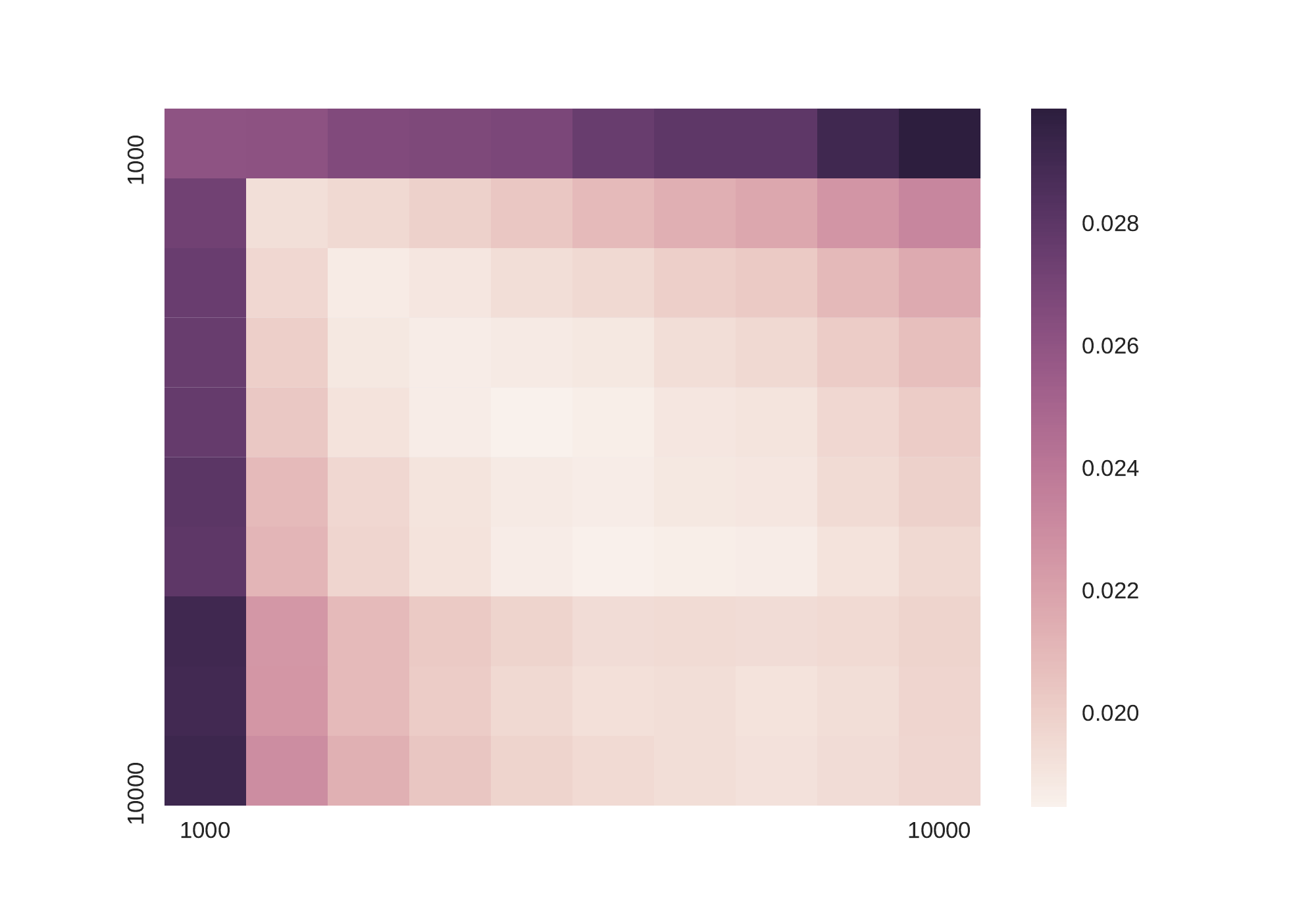}
            \caption[]%
            {{\small Dimension 1000-10000}} 
            \label{subfig:glove_2}
\end{subfigure}
\caption{GloVe Relative PIP Distance for Different Dimensionalities}
\label{fig:glove_relative_PIP}
\end{figure}
The plots show that both Word2Vec and GloVe are forward stable under the PIP distance. For the same dimensionality between two runs, the relative forward errors are no greater than 0.05, usually around 0.05 for Word2Vec and 0.02 for GloVe. We also clearly observe the stability \textit{between} different dimensionalities from the banded structures of the matrices. Models with similar dimensionalities enjoy small relative forward errors. In other words, we showed empirically that Word2Vec and GloVe are forward stable with respect to both noise introduced in training, and with respect to dimensionality. The forward stability with respect to dimensionality indicates that a minor mis-specification of dimensionality will not significantly change the embedding, which is another interpretation of the robustness of parameter-specification discussed in Section \ref{sec:robust_to_overfitting}.

\subsection{Dimensionality Selection: Optimize over the Bias-Variance Trade-off}
LSA/LSI are embedding algorithms using explicit matrix factorization and skip-gram/GloVe are implicit matrix factorization algorithms. With the estimated signal matrix spectrum, we can use Monte-Carlo methods to estimate the PIP loss by generating random signal directions, as directions do not matter due to the unitary-invariance. To select the optimal dimensionality, one simply finds the dimensionality that minimizes the estimated PIP loss. In this experiment, we show that the optimal dimensionalities selected by minimizing the PIP loss achieve empirically-optimal results on human-labeled intrinsic tests.

\subsubsection{Dimensionality Selection: Word Embedding with LSA}
\label{sec:word_embedding_dim}
We calculate the PIP loss minimizing dimensionalities for word embeddings obtained from factorizing the PPMI matrices, which is constructed on the Text8 corpus. At the same time, we find the \textit{$p\%$ range of near-optimality}, the dimensionalities whose optimality-gap is within $p\%$ of the gap of a non-informative, 1-d embedding, plotted in Figure \ref{fig:level_sets}. We observe a plateau around the optimal dimensionalities, and the width of the plateau increases with respect to $\alpha$, the robustness phenomena described in the previous section. Table \ref{table:dimensions} shows the theoretically optimal PIP loss minimizing dimensionalities, along with the $5\%$, $10\%$, $20\%$ and $50\%$ range of near-optimality (the level sets in Figure \ref{fig:level_sets}), plus the empirically optimal dimensionalities from word similarity tests. We do observe that the empirically optimal dimensionalities lie close to the theoretically optimal ones, all of which are within the 5\% ranges except for $\alpha=0$, which is within the 20\% range.

\begin{figure}[h]
\centering
\hspace*{\fill}%
\begin{subfigure}[b]{0.32\textwidth}
\includegraphics[width=\textwidth]{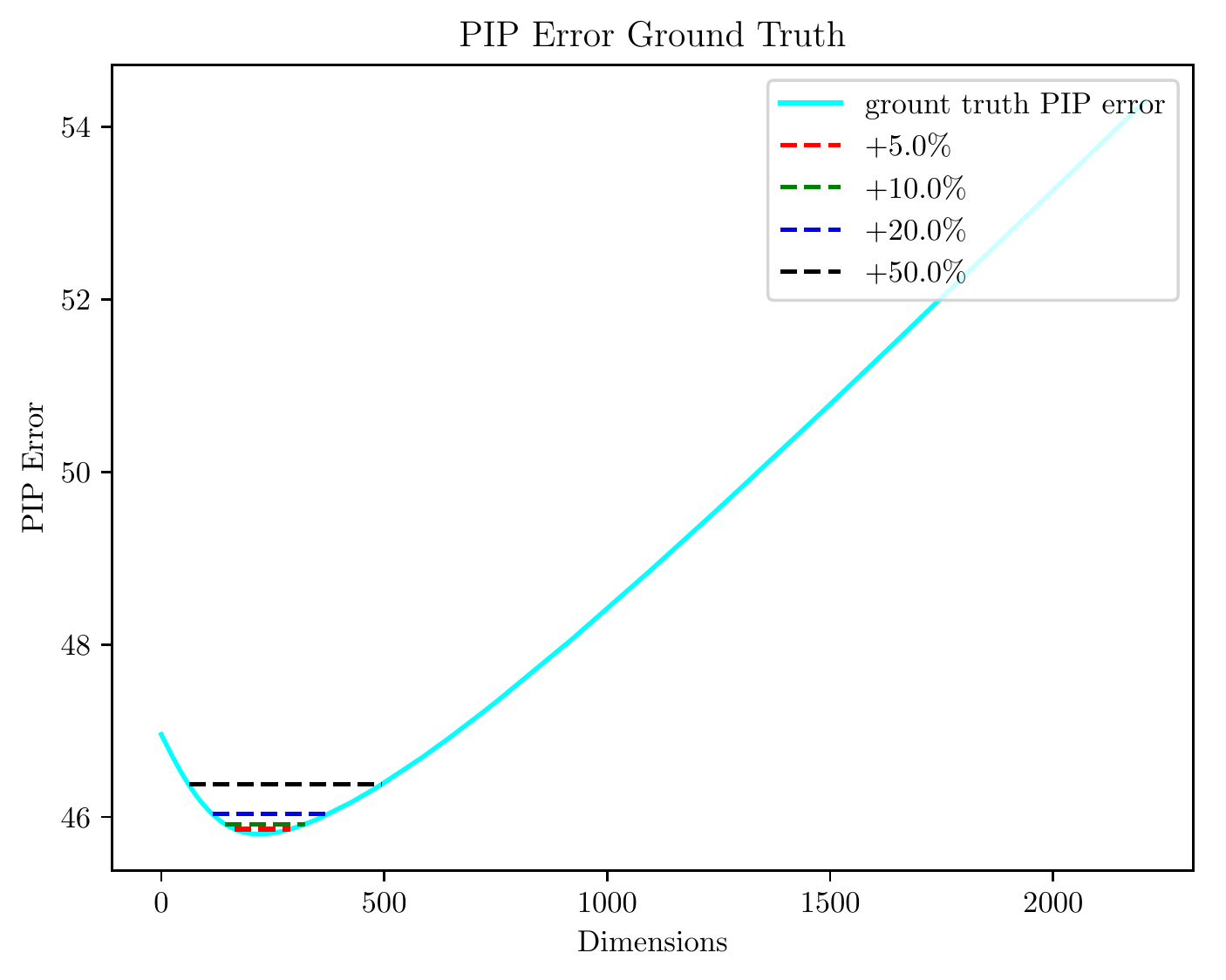}
            \caption[]%
            {{\small $\alpha=0$}} 
            \label{subfig:pip_gt_0}
\end{subfigure}
\hfill
\begin{subfigure}[b]{0.32\textwidth}           \includegraphics[width=\textwidth]{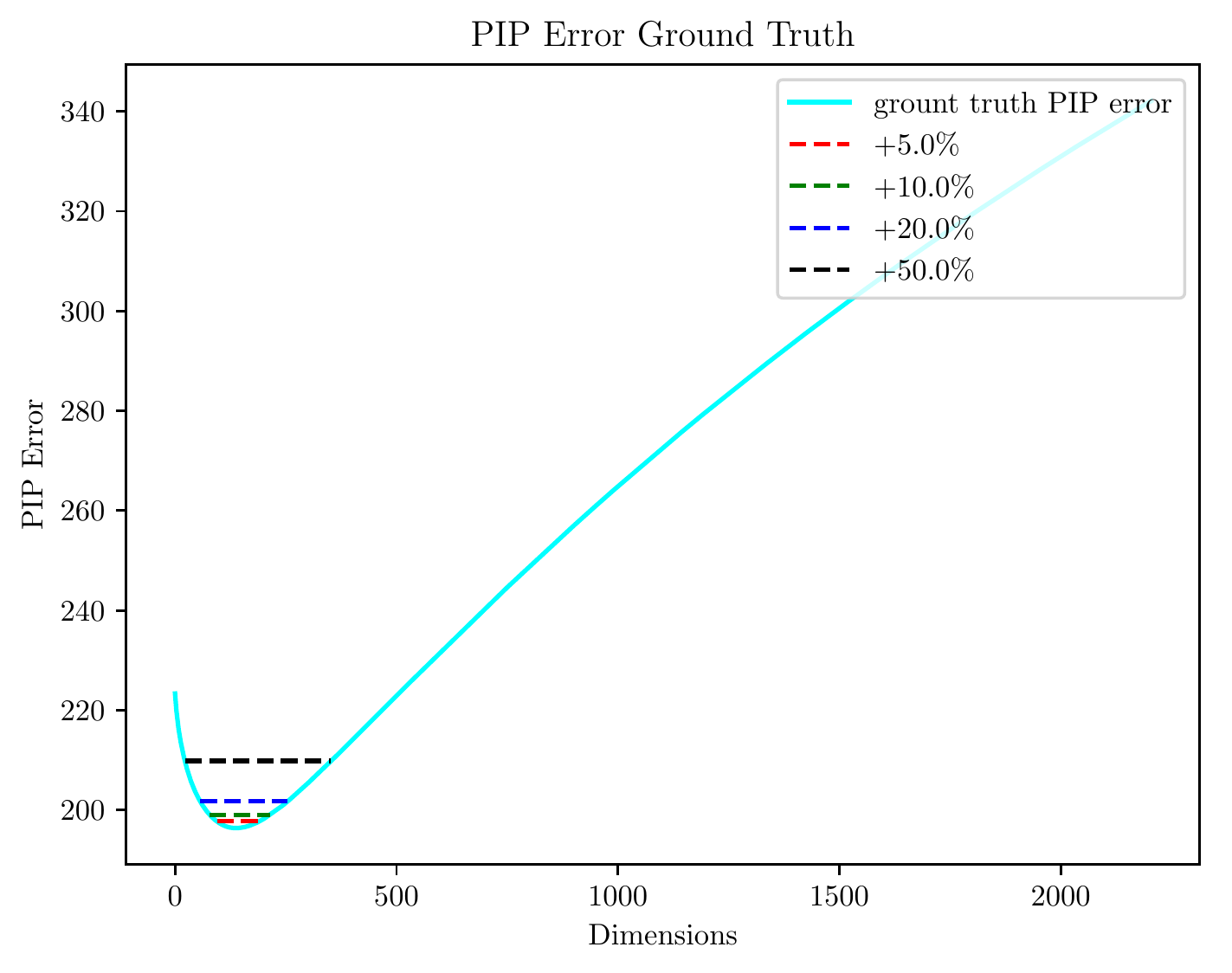}
            \caption[]%
            {{\small $\alpha=0.25$}}    
            \label{subfig:pip_gt_025}
\end{subfigure}
\hfill
\begin{subfigure}[b]{0.32\textwidth}
\includegraphics[width=\textwidth]{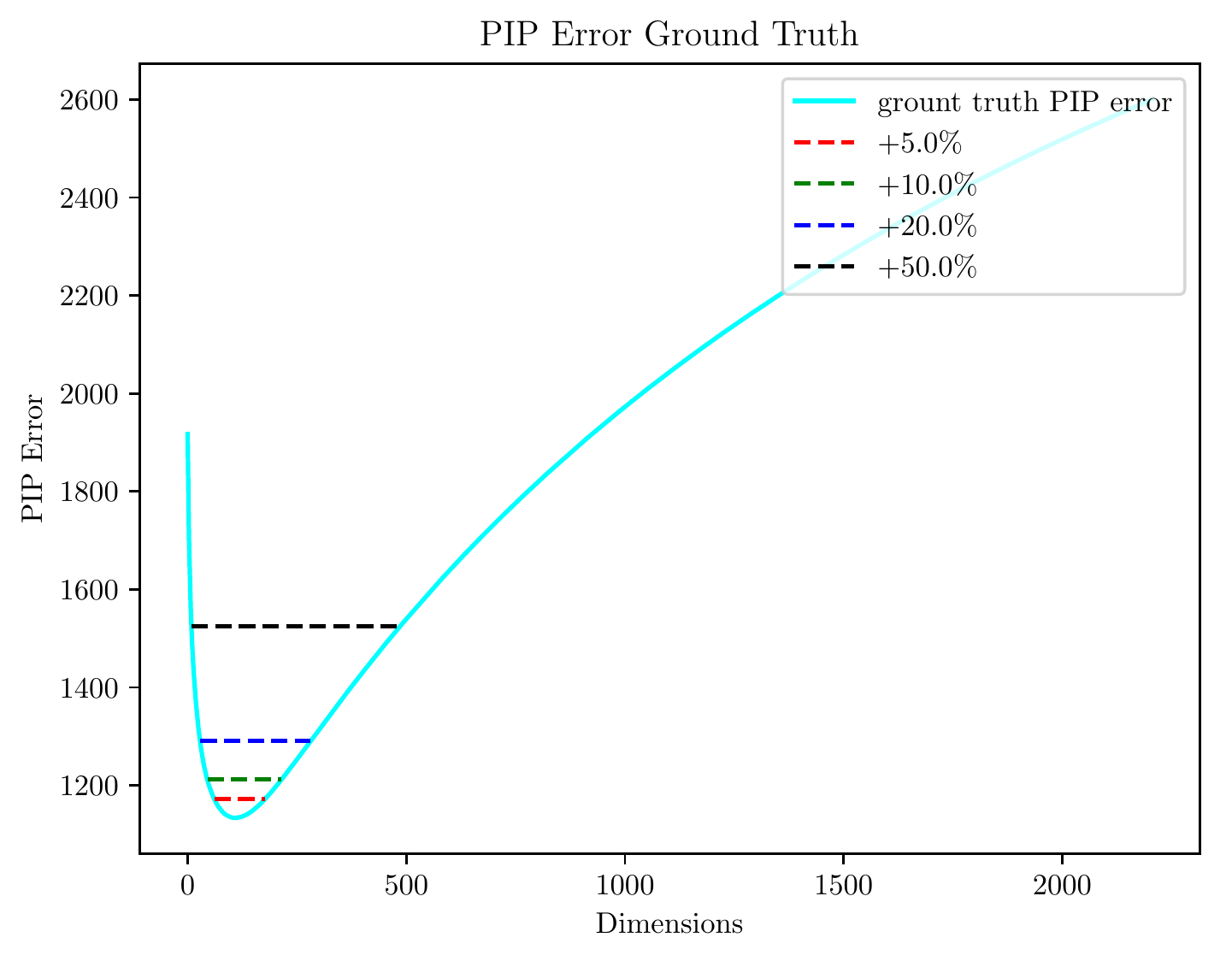}
            \caption[]%
            {{\small $\alpha=0.5$}} 
            \label{subfig:pip_gt_05}
\end{subfigure}
\medskip

\begin{subfigure}[b]{0.32\textwidth}           \includegraphics[width=\textwidth]{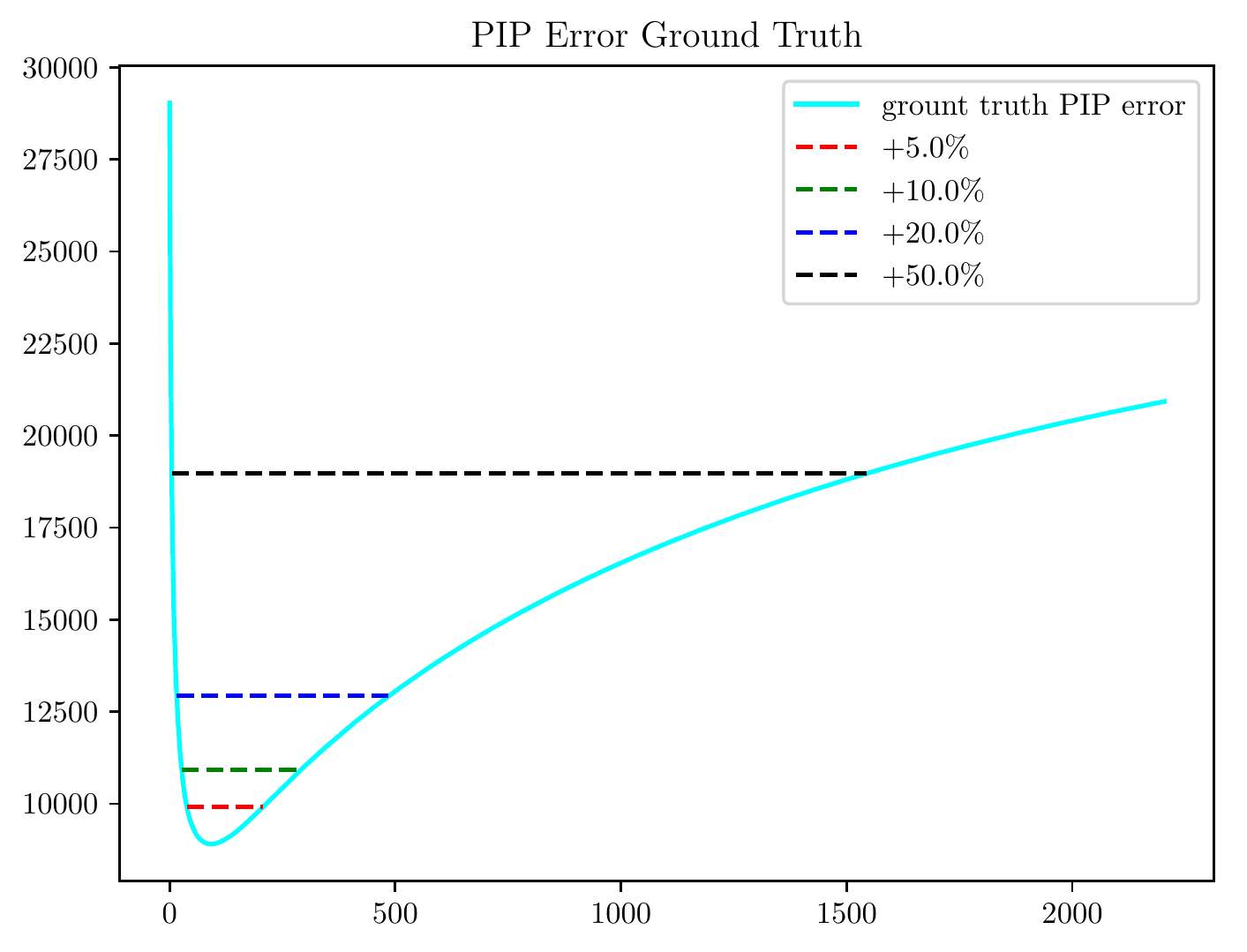}
            \caption[]%
            {{\small $\alpha=0.75$}}    
            \label{subfig:pip_gt_075}
\end{subfigure}
\quad
\begin{subfigure}[b]{0.32\textwidth}         \includegraphics[width=\textwidth]{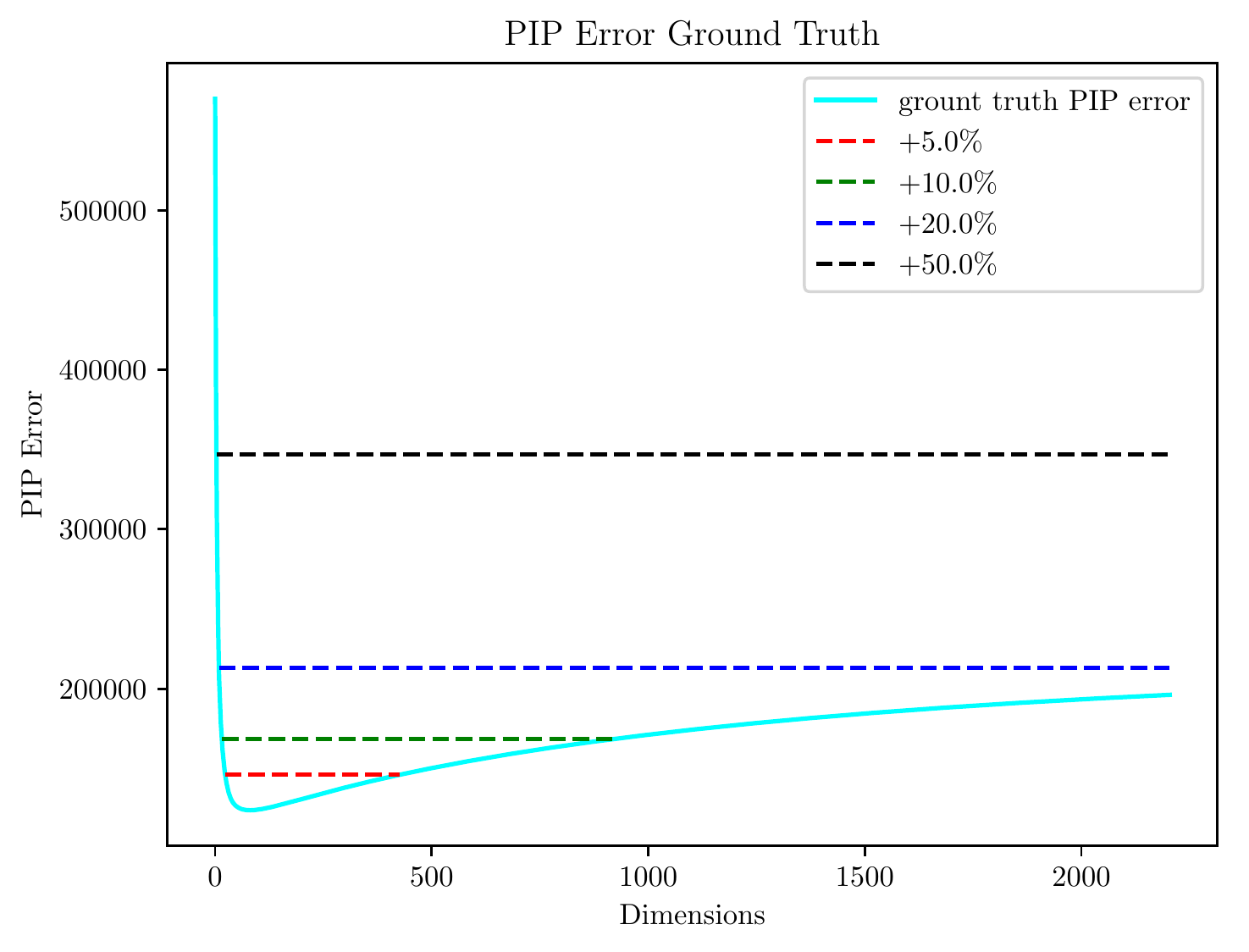}
            \caption[]%
            {{\small $\alpha=1$}}    
            \label{subfig:pip_gt_1}
\end{subfigure}
\caption{Bias-Variance Trade-off of the PIP Loss and Range of Near-optimality}
\label{fig:level_sets}
\end{figure}

\captionof{table}{Optimal Dimensions and Regions for PPMI Matrix on the Text8 Corpus} \label{table:dimensions} 
\resizebox{\columnwidth}{!}{%
\begin{tabular}{ |c| c| c| c|c|c|c|c|}
 \hline
 $\alpha$ & $\text{PIP} \arg\min$ & $\min$ +5\% & $\min$ +10\% & $\min$ +20\% & $\min$ +50\% & WS353 opt. & MT771 opt. \\
  \hline
0 & 214 & [164,289] & [143,322] & [115,347] & [62,494] & 127 & 116 \\
\hline
0.25 & 138 & [95,190] & [78,214] & [57,254] & [23,352] & 146 & 116 \\
\hline
0.5 & 108 & [61,177] & [45,214] & [29,280] & [9,486] & 146 & 116 \\
\hline
0.75 & 90 & [39,206] & [27,290] & [16,485]& [5,1544] & 155 & 176 \\
\hline
1 & 82 & [23,426] & [16,918] & [9,2204]& [3,2204] & 365 & 282 \\
\hline
\end{tabular}
}
\subsubsection{Dimensionality Selection: Document Embedding with LSI}
We tested TF-IDF document embedding on the STS dataset \citep{marelli2014semeval}, and compared the result with human labeled document similarities. For symmetric decomposition ($\alpha=0.5$), the theoretical PIP loss minimizing dimensionality is 144, and between 100 to 206 the PIP loss is within 5\%. The empirically optimal dimensionality on the STS dataset is 164, with a correlation of 0.653 to the human labels. The empirically optimal dimensionality is within the 5\% interval of near-optimality predicted by the theory. The bias-variance curve for the PIP loss and the actual performance on document similarity test are plotted in Figure \ref{fig:sts}.
\begin{figure}[h]
\centering
\begin{subfigure}[b]{0.38\textwidth}
\captionsetup{justification=centering}
\includegraphics[width=\textwidth]{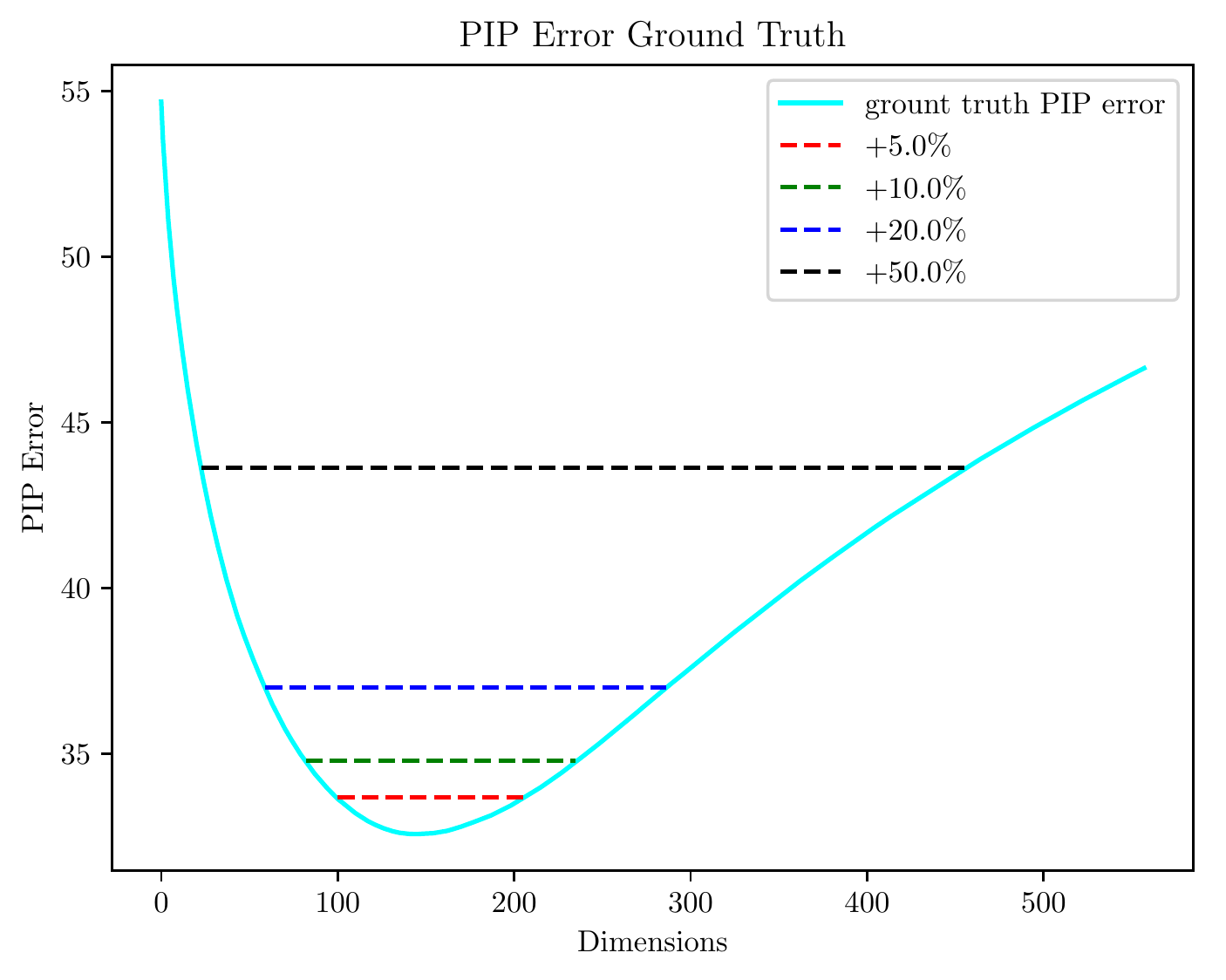}
            \caption[]%
            {{\small Bias-Variance Trade-off of the PIP Loss}} 
            \label{subfig:pip_gt_0_5_sts}
\end{subfigure}
\quad
\begin{subfigure}[b]{0.44\textwidth}
\captionsetup{justification=centering}
\includegraphics[width=\textwidth]{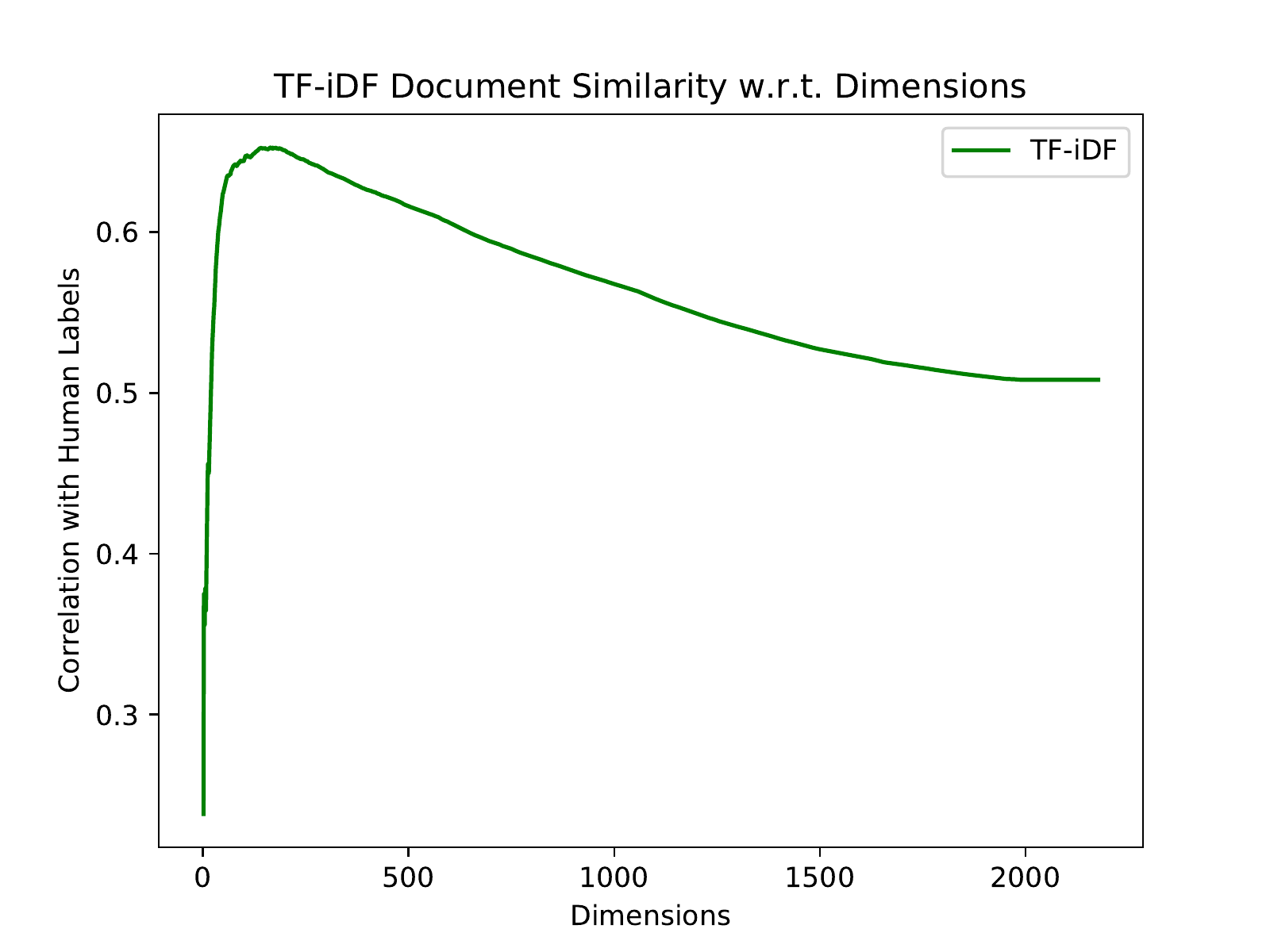}
            \caption[]%
            {{\small TF-IDF Correlation with Human Labels for Different Dimensionalities}}    
      \label{subfig:gt_sts}
\end{subfigure}
\caption{TF-IDF Document Similarity on the STS Dataset}
\label{fig:sts}
\end{figure}

\subsubsection{Dimensionality Selection: Skip-gram Word2Vec and Glove}
\label{sec:word2vec_and_glove}
We now present an analysis of the popular word embedding models, Word2Vec and GloVe. As discussed earlier in section \ref{sec:implicit_factorization}, although sometimes classified as neural network-based embedding procedures, both algorithms are in fact  performing implicit matrix factorizations. We use the PMI matrix as the surrogate for Word2Vec without negative sampling, SPPMI for Word2Vec with negative sampling, and the log-count matrix for GloVe. For the purpose of validating our theory, we train the models with different dimensionalities on the Text8 corpus, and measure the intrinsic functionality test scores for all the trained models. To cover as many dimensionalities as possible, and at the same time make the computation feasible, we train the models on the following subsets of dimensionalities. From dimensionality 1 to 100, we train both models at an increment of 1. From 100 to 1000, we train with an increment of 20 ($k=100,120,140,\cdots$). From 1000 to 10000, the increment is 1000 ($k=1000,2000,3000,\cdots$). We use the native implementations of Word2Vec and GloVe, retrieved from their GitHub repositories\footnote{https://github.com/tensorflow/models/tree/master/tutorials/embedding}\footnote{https://github.com/stanfordnlp/GloVe}. On the trained embeddings, the intrinsic functionality tests \citep{mikolov2013efficient,wordsim353,mturk771} will produce empirically optimal dimensionalities. At the same time, we derive the theoretically optimal dimensionalities for the embedding procedures by directly optimizing over the PIP loss bias-variance trade-off on their surrogate matrices. Like previously in the LSA and LSI experiments, we will show that the optimal dimensionalities from theory predictions and empirical experiments usually agree with each other. This allows researchers to perform dimensionality selection without having to rely on computation-heavy and time-consuming empirical parameter tuning.

\paragraph{Theoretical Analysis}
On the Text8 corpus, we notice that the log-count matrix is less noisy than the PMI matrix; they have similar signal strength, but the estimated noise standard deviation is $0.351$ for the PMI matrix and $0.14$ for the log-count matrix. Figure \ref{fig:word2vec_vs_glove} shows not only the plateau is broader for the log-count matrix, but the PIP loss minimizing dimensionality is more to the right. Table \ref{tab:word2vec_vs_glove} quantitatively compares the optimal dimensionalities and $p\%$-range of near-optimality for $p=5,10,20,50$.

\begin{figure}[htb]
\centering
\begin{subfigure}[b]{0.44\textwidth}
\includegraphics[width=\textwidth]{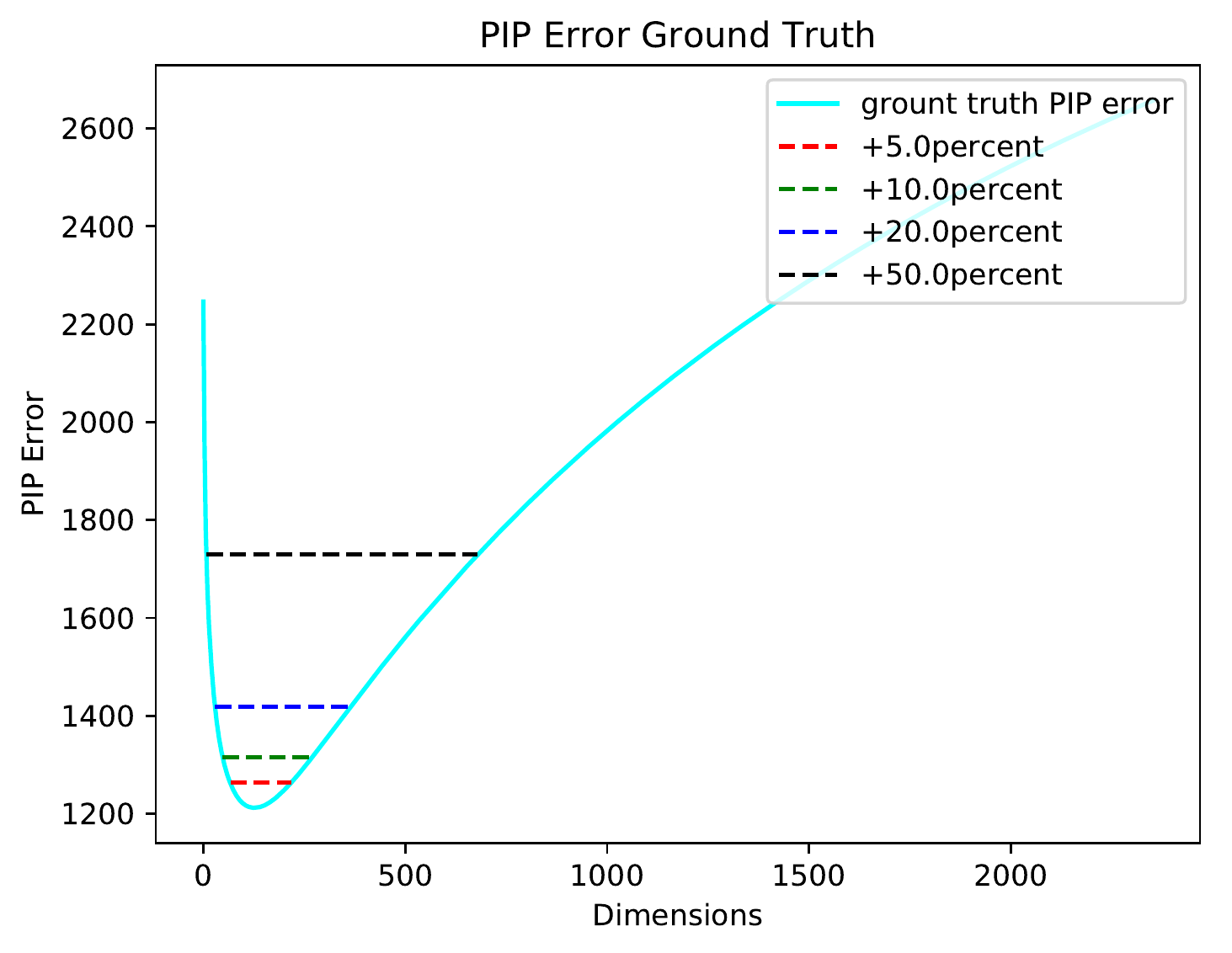}
            \caption[]%
            {{\small Bias-Variance Trade-off: PMI Matrix}} 
            \label{subfig:pip_gt_0_5_word2vec}
\end{subfigure}
\quad
\begin{subfigure}[b]{0.44\textwidth}           \includegraphics[width=\textwidth]{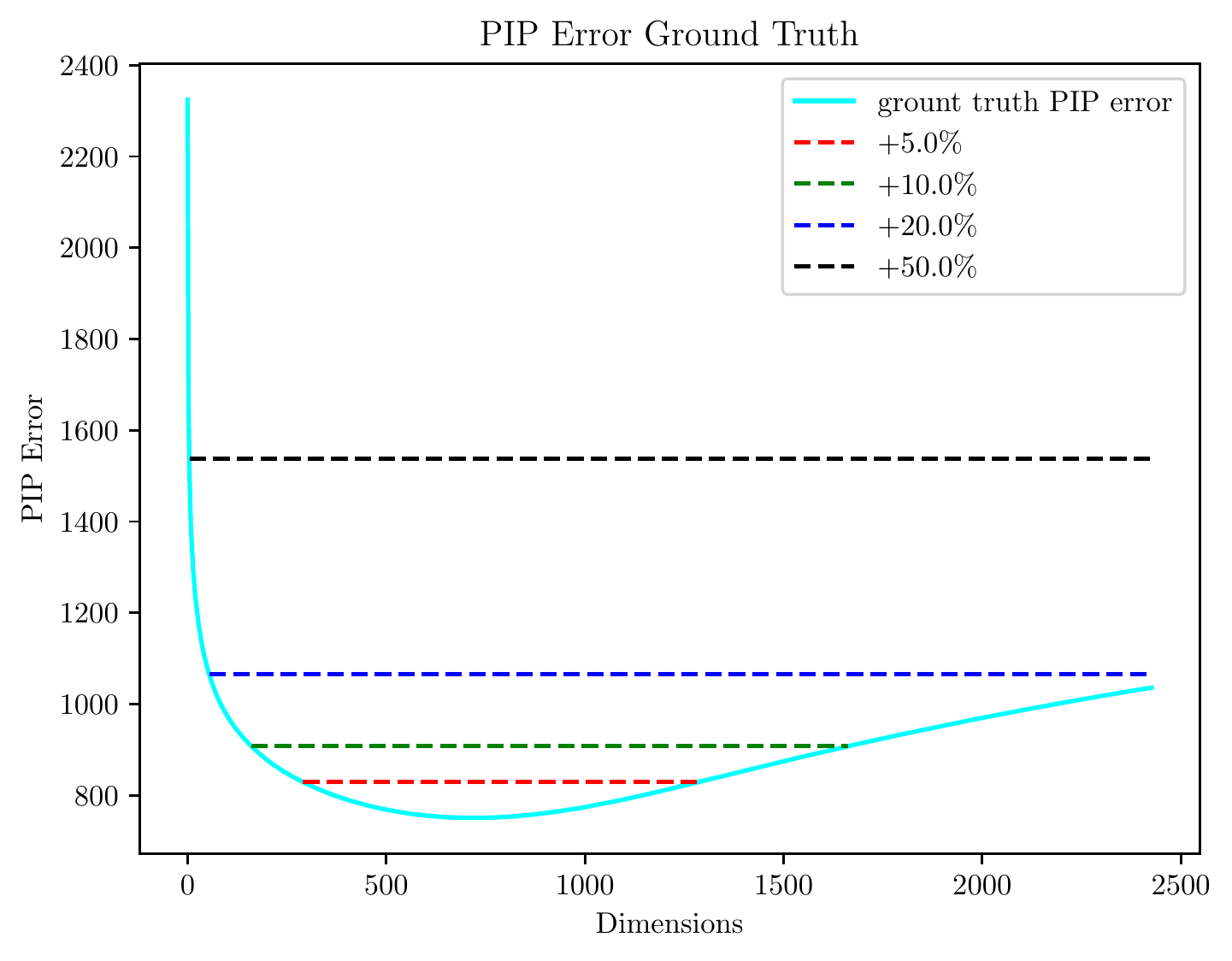}
            \caption[]%
            {{\small Bias-Variance Trade-off: Log-count Matrix}}    
      \label{subfig:pip_gt_0_5_glove}
\end{subfigure}
\caption{Word2Vec and GloVe Surrogate Matrices on the Text8 corpus}
\label{fig:word2vec_vs_glove}
\end{figure}

The PIP loss minimizing dimensionality is 108 for the PMI matrix and 719 for the log-count matrix. At the same time, the intervals are wider for the log-count matrix. These observations indicate that on the Text8 corpus, compared to Word2Vec, GloVe should achieve optimality at a larger dimensionality, and should be less sensitive to over-parametrization.

\captionof{table}{Optimal Dimensionalities and Intervals for the Text8 Corpus} \label{tab:word2vec_vs_glove}
\resizebox{\columnwidth}{!}{%
\begin{tabular}{ | c|c| c| c| c|c|}
 \hline
Surrogate Matrix & $\text{PIP} \arg\min$ & $\min$ $+5\%$ & $\min$ $+10\%$ & $\min$ $+20\%$ & $\min$ $+50\%$\\
  \hline
Skip-gram (PMI) & 129 & [67,218] & [48,269] & [29,365] & [9,679] \\
\hline
GloVe (log-count)& 719 & [290,1286] & [160,1663] & [55,2426] & [5,2426] \\
\hline
\end{tabular}
}

\paragraph{Experimental Results}
We verified the theoretical predictions with intrinsic similarity and compositionality tests on Word2Vec and GloVe trained on the Text8 corpus. The results are summarized as follows.
\begin{itemize}
\item Optimal Dimensionalities. For GloVe, the empirically optimal dimensionalities for WordSim353, Mturk771 and Google Analogy are 220
, 860
, and 
560. For Word2Vec, the three numbers are 56, 102 and 220 respectively. The empirically optimal dimensionalities are indeed much larger for GloVe, which agrees with the prediction of the theoretical results (Table \ref{tab:word2vec_vs_glove}). In terms of optimal dimensionalities for Word2Vec, one of the empirically optimal dimensionalities given by the three tests is within the 5\% interval, and the other two are within the 10\% interval with the PMI matrix as the surrogate. For GloVe, two are within the 5\% interval, and the other is within the 10\% interval with the log-count matrix surrogate. 

\item Performance Degradation with Over-parametrization. Figure \ref{fig:word2vec_vs_glove} shows the theoretical prediction that GloVe should be more robust to over-parametrization than Word2Vec as it has a wider and flatter plateau near optimality. The empirical evidence can be readily seen by comparing Figure \ref{fig:word2vec} and Figure \ref{fig:glove}. The drop of performance due to over-parametrization is less obvious for GloVe than Word2Vec. This phenomenon can be consistently observed for the WordSim353, Mturk771 and Google Analogy tests.
\end{itemize}

\paragraph{Runtime Comparison}
The PIP loss minimizing method for dimensionality selection is theoretically justified and accurate. Moreover, it is faster than the empirical grid search method, as shown in the runtime comparison. Experiments are done on a server with dual Xeon E5-2630Lv2 CPUs and 64GB of RAM\footnote{evaluation details are listed in the appendix due to space limitation}.

\captionof{table}{Time Needed to Run Dimensionality Selection on the Text8 Corpus} \label{tab:runtime}
\resizebox{\columnwidth}{!}{%
\begin{tabular}{ | c|c| c| c| c|}
 \hline
Dimensionality Selection Time & PPMI LSA  & skip-gram Word2Vec (1-400) dims & GloVe (1-400) dims\\
  \hline
Empirical validation & 1.7 hours & 2.5 days & 1.8 days \\
\hline
Minimizing PIP loss & 1.2 hours & 42 minutes & 42 minutes \\
\hline
\end{tabular}
}
We can see that the PIP loss minimizing method is orders of magnitudes faster for skip-gram Word2Vec and GloVe, as the empirical methods have to actually train models for different dimensionalities. For LSA, the runtimes are similar because both need to perform a singular value decomposition. Still, the PIP loss minimizing method is faster because it does not need to perform intrinsic functionality tests on every dimensionality.

\subsubsection{Dimensionality Selection for Large Corpus}
In this experiment, we discuss about training embeddings for large corpus. The example we study is the Wikipedia Corpus which has more than two billion tokens, retrieved on Jun 23, 2017. We build both the PMI matrix and the log-count matrix as surrogates to Word2Vec and GloVe, using the most frequent 10000 words. The estimations of signal spectrum and noise follow the discussions in Section \ref{sec:noise_est} and \ref{sec:spectrum_est}. 
Our theory predicts the PIP loss-minimizing dimensionality is 391 for Word2Vec and 255 for GloVe.
\captionof{table}{Optimal Dimensions and Regions for Wikipedia Corpus} \label{table:dimensions_wiki} 
\resizebox{\columnwidth}{!}{%
\begin{tabular}{ |c|c| c| c| c|c|}
 \hline
Surrogate Matrix&$\text{PIP} \arg\min$ &$\min$ $+5\%$ &$\min$ $+10\%$ &$\min$ $+20\%$ &$\min$ $+50\%$\\
  \hline
Word2Vec (PMI) &391 & [211,658] & [155,839] & [94,1257] & [24,2110]\\
\hline
GloVe (Log-count)& 255 & [115,491] & [77,699] & [41,1269] & [8,1613]\\
\hline
\end{tabular}
}

\begin{figure}[htb]
\centering
\begin{subfigure}[b]{0.44\textwidth}
\captionsetup{justification=centering}
\includegraphics[width=\textwidth]{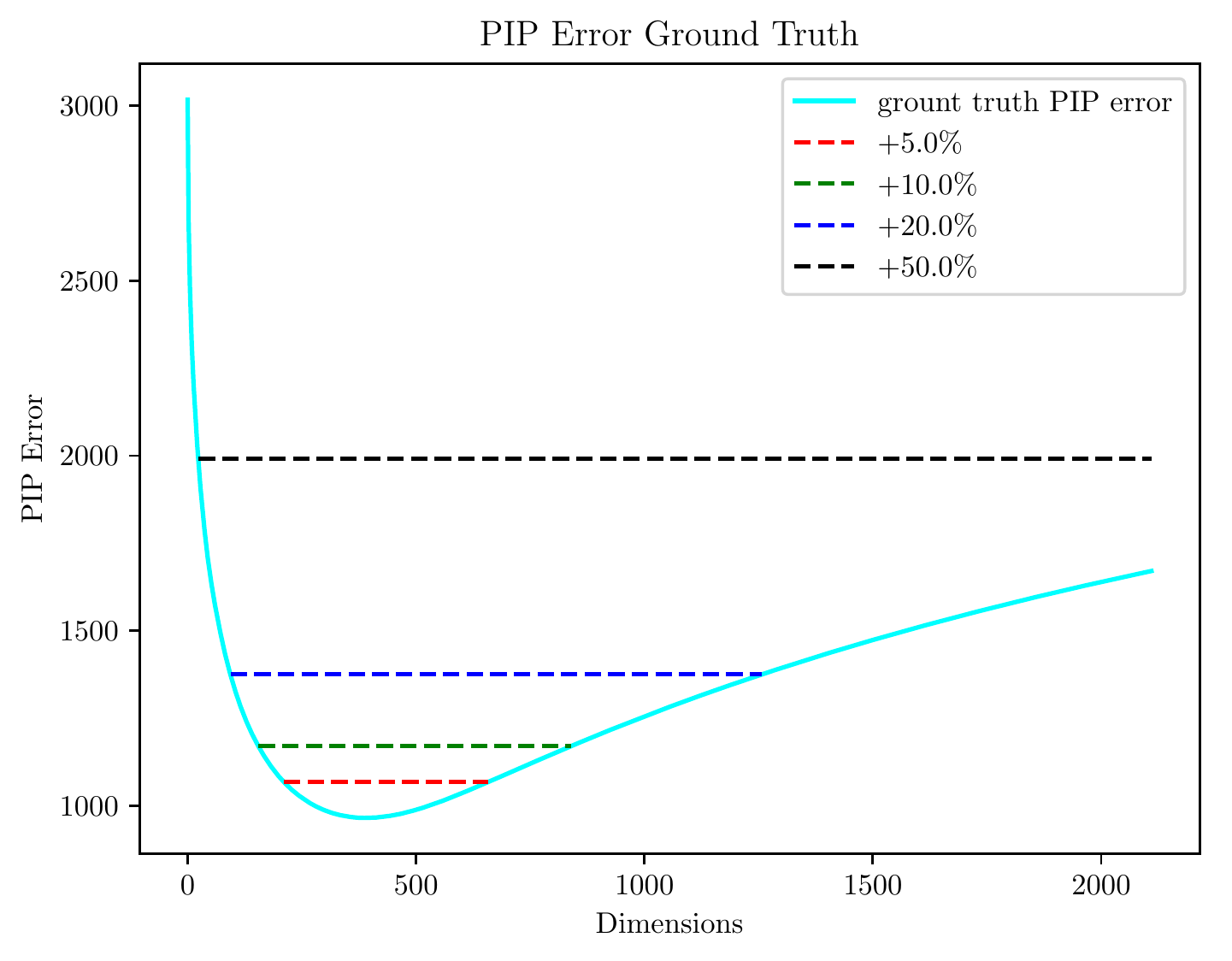}
            \caption[]%
            {{\small Dim vs PIP Loss: Word2Vec with PPMI Matrix Surrogate}} 
            \label{subfig:pip_gt_0_5_word2vec_wiki}
\end{subfigure}
\quad
\begin{subfigure}[b]{0.44\textwidth}
\captionsetup{justification=centering}
\includegraphics[width=\textwidth]{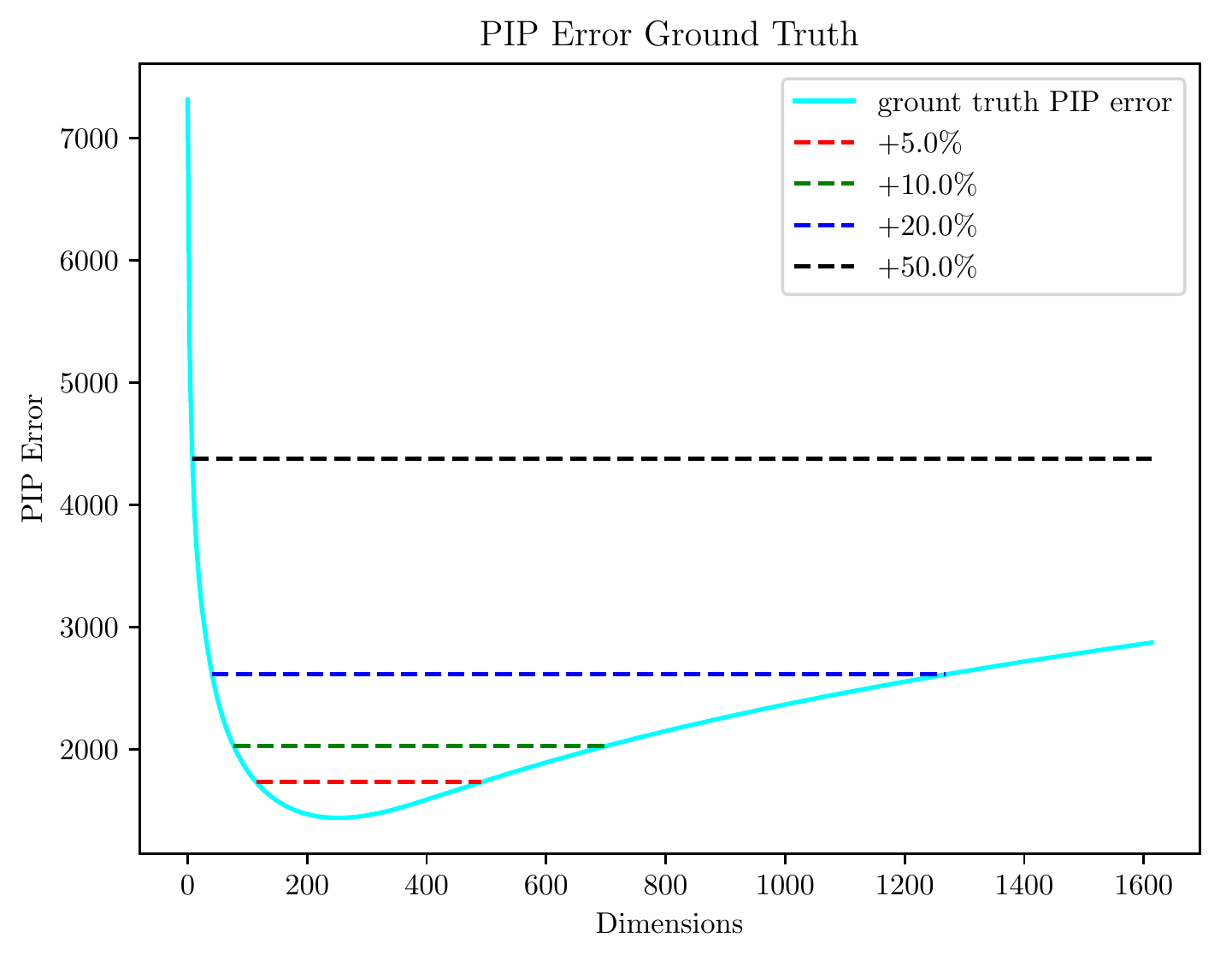}
            \caption[]%
            {{\small Dim vs PIP Loss: GloVe with Log-count Matrix Surrogate}}    
      \label{subfig:pip_gt_0_5_glove_wiki}
\end{subfigure}
\caption{Comparison of Word2Vec and GloVe Surrogate Matrices and Dimensionality}
\label{fig:word2vec_vs_glove_wiki}
\end{figure}
Table \ref{table:dimensions_wiki} and Figure \ref{fig:word2vec_vs_glove_wiki} explain why word embeddings are so useful even without careful dimensionality selection: the PMI surrogate has an interval of length 447 (from 211 to 658) and the log-count surrogate has an interval of length 376 (from 115 to 491), such that any dimensionality within the intervals has a PIP loss within 5\% of the optimal ones. Such a flat plateau near optimality is the key to their empirical successes. Combining results from the Text8 corpus, we have the following summary for embedding dimensionality selection. As a general rule of thumb, optimal dimensionality grows with respect to increased signal-to-noise ratio. Large corpus usually has reduced estimation error and smaller variance, so the bias-variance trade-off is shifted to the right side. At the same time, heterogeneous corpus will have increased variance, shifting the optimal dimensionality to the left. For small corpora (at the scale of the Text8 corpus, with 10-100M tokens) with relatively high estimation noise, dimensionalities between 100 and 200 will be sufficient, while for larger corpora (at the scale of the Wikipedia corpus, with 1-10B tokens), dimensionalities between 200 and 400 should be preferred. For even larger corpora, using larger dimensionalities should yield increased performance, although we did not verify this due to lack of such corpora and limitation of computational power. More importantly, if one is willing to get the spectrum and noise estimators, one will be able to find the exact PIP loss minimizing dimensionality.

We would like to end this section with some comments on why the PIP loss minimizing method should be preferred for dimensionality selection. The empirical selection procedure has a few drawbacks. First, training several models of different dimensionalities can be time consuming. Second, there are no intrinsic functionality tests that can serve as golden standards. For example, the analogy test could give an optimal dimensionality of 300 while the similarity test could give an optimal dimensionality of 400. In this case, it is not obvious how a consensus should be reached. Third, the empirical procedure is likely to give a dimensionality that is larger than the actual optimal one, resulting in increased training and inference complexity. Meanwhile, the PIP loss-minimization criterion provides a single objective, is fast to evaluate, and is theoretically justified. We thus encourage researcher to select the dimensionality hyper-parameter using this criterion, which involves only estimating the spectrum of the signal matrix and noise standard deviation.

\section{Conclusion}
In this paper, we present a theoretical framework for understanding vector embedding dimensionalities. We propose the PIP loss, a unitary-invariant metric on vector embedding similarity. We focus on embeddings that can be obtained from explicit or implicit matrix factorizations, and develop tools that reveal a fundamental bias-variance trade-off in dimensionality selection, which explains the existence of a ``sweet dimensionality''. We use our framework to discover the robustness of vector embeddings, and also use the same framework to show the forward stability of embedding procedures including skip-gram Word2Vec and GloVe. Lastly, we provide a new dimensionality selection method by minimizing the PIP loss. All of our theoretical discoveries are well validated on real datasets. Although our applications focus primarily on NLP, the theory itself has much wider applicability, with the only requirement that the embeddings are from matrix factorizations. We believe that many models in data mining, spectral clustering and networks can be analyzed under the same unified framework.

\newpage
\bibliography{example_paper}

\begin{thebibliography}{63}
\providecommand{\natexlab}[1]{#1}
\providecommand{\url}[1]{\texttt{#1}}
\expandafter\ifx\csname urlstyle\endcsname\relax
  \providecommand{\doi}[1]{doi: #1}\else
  \providecommand{\doi}{doi: \begingroup \urlstyle{rm}\Url}\fi

\bibitem[Arora(2016)]{arora2016blog}
Sanjeev Arora.
\newblock Word embeddings: Explaining their properties, 2016.
\newblock URL \url{http://www.offconvex.org/2016/02/14/word-embeddings-2/}.
\newblock [Online; accessed 19-May-2018].

\bibitem[Arora et~al.(2015)Arora, Li, Liang, Ma, and Risteski]{arora2015rand}
Sanjeev Arora, Yuanzhi Li, Yingyu Liang, Tengyu Ma, and Andrej Risteski.
\newblock Rand-walk: A latent variable model approach to word embeddings.
\newblock \emph{arXiv preprint arXiv:1502.03520}, 2015.

\bibitem[Artetxe et~al.(2016)Artetxe, Labaka, and Agirre]{artetxe2016learning}
Mikel Artetxe, Gorka Labaka, and Eneko Agirre.
\newblock Learning principled bilingual mappings of word embeddings while
  preserving monolingual invariance.
\newblock In \emph{Proceedings of the 2016 Conference on Empirical Methods in
  Natural Language Processing}, pages 2289--2294, 2016.

\bibitem[Bahdanau et~al.(2014)Bahdanau, Cho, and Bengio]{bahdanau2014neural}
Dzmitry Bahdanau, Kyunghyun Cho, and Yoshua Bengio.
\newblock Neural machine translation by jointly learning to align and
  translate.
\newblock \emph{arXiv preprint arXiv:1409.0473}, 2014.

\bibitem[Baroni et~al.(2014)Baroni, Dinu, and Kruszewski]{baroni2014don}
Marco Baroni, Georgiana Dinu, and Germ{\'a}n Kruszewski.
\newblock Don't count, predict! a systematic comparison of context-counting vs.
  context-predicting semantic vectors.
\newblock In \emph{Proceedings of the 52nd Annual Meeting of the Association
  for Computational Linguistics (Volume 1: Long Papers)}, volume~1, pages
  238--247, 2014.

\bibitem[Bengio et~al.(2003)Bengio, Ducharme, Vincent, and
  Jauvin]{bengio2003neural}
Yoshua Bengio, R{\'e}jean Ducharme, Pascal Vincent, and Christian Jauvin.
\newblock A neural probabilistic language model.
\newblock \emph{Journal of machine learning research}, 3\penalty0
  (Feb):\penalty0 1137--1155, 2003.

\bibitem[Bradford(2008)]{bradford2008empirical}
Roger~B Bradford.
\newblock An empirical study of required dimensionality for large-scale latent
  semantic indexing applications.
\newblock In \emph{Proceedings of the 17th ACM conference on Information and
  knowledge management}, pages 153--162. ACM, 2008.

\bibitem[Breese et~al.(1998)Breese, Heckerman, and Kadie]{breese1998empirical}
John~S Breese, David Heckerman, and Carl Kadie.
\newblock Empirical analysis of predictive algorithms for collaborative
  filtering.
\newblock In \emph{Proceedings of the Fourteenth conference on Uncertainty in
  artificial intelligence}, pages 43--52. Morgan Kaufmann Publishers Inc.,
  1998.

\bibitem[Bullinaria and Levy(2012)]{bullinaria2012extracting}
John~A Bullinaria and Joseph~P Levy.
\newblock Extracting semantic representations from word co-occurrence
  statistics: stop-lists, stemming, and svd.
\newblock \emph{Behavior research methods}, 44\penalty0 (3):\penalty0 890--907,
  2012.

\bibitem[Cai et~al.(2010)Cai, Cand{\`e}s, and Shen]{cai2010singular}
Jian-Feng Cai, Emmanuel~J Cand{\`e}s, and Zuowei Shen.
\newblock A singular value thresholding algorithm for matrix completion.
\newblock \emph{SIAM Journal on Optimization}, 20\penalty0 (4):\penalty0
  1956--1982, 2010.

\bibitem[Cand{\`e}s and Recht(2009)]{candes2009exact}
Emmanuel~J Cand{\`e}s and Benjamin Recht.
\newblock Exact matrix completion via convex optimization.
\newblock \emph{Foundations of Computational mathematics}, 9\penalty0
  (6):\penalty0 717, 2009.

\bibitem[Caron(2001)]{caron2001experiments}
John Caron.
\newblock Experiments with {LSA} scoring: optimal rank and basis.
\newblock In \emph{Computational information retrieval}, pages 157--169.
  Society for Industrial and Applied Mathematics, 2001.

\bibitem[Chatterjee(2015)]{chatterjee2015matrix}
Sourav Chatterjee.
\newblock Matrix estimation by universal singular value thresholding.
\newblock \emph{The Annals of Statistics}, 43\penalty0 (1):\penalty0 177--214,
  2015.

\bibitem[Church and Hanks(1990)]{church1990word}
Kenneth~Ward Church and Patrick Hanks.
\newblock Word association norms, mutual information, and lexicography.
\newblock \emph{Computational linguistics}, 16\penalty0 (1):\penalty0 22--29,
  1990.

\bibitem[Davis and Kahan(1970)]{davis1970rotation}
Chandler Davis and William~Morton Kahan.
\newblock The rotation of eigenvectors by a perturbation. iii.
\newblock \emph{SIAM Journal on Numerical Analysis}, 7\penalty0 (1):\penalty0
  1--46, 1970.

\bibitem[Deerwester et~al.(1990)Deerwester, Dumais, Furnas, Landauer, and
  Harshman]{deerwester1990indexing}
Scott Deerwester, Susan~T Dumais, George~W Furnas, Thomas~K Landauer, and
  Richard Harshman.
\newblock Indexing by latent semantic analysis.
\newblock \emph{Journal of the American society for information science},
  41\penalty0 (6):\penalty0 391, 1990.

\bibitem[Finkelstein et~al.(2001)Finkelstein, Gabrilovich, Matias, Rivlin,
  Solan, Wolfman, and Ruppin]{wordsim353}
Lev Finkelstein, Evgeniy Gabrilovich, Yossi Matias, Ehud Rivlin, Zach Solan,
  Gadi Wolfman, and Eytan Ruppin.
\newblock Placing search in context: The concept revisited.
\newblock In \emph{Proceedings of the 10th international conference on World
  Wide Web}, pages 406--414. ACM, 2001.

\bibitem[Firth(1957)]{firth1957synopsis}
John~R Firth.
\newblock A synopsis of linguistic theory, 1930-1955.
\newblock \emph{Studies in linguistic analysis}, 1957.

\bibitem[Frome et~al.(2013)Frome, Corrado, Shlens, Bengio, Dean, Mikolov,
  et~al.]{frome2013devise}
Andrea Frome, Greg~S Corrado, Jon Shlens, Samy Bengio, Jeff Dean, Tomas
  Mikolov, et~al.
\newblock Devise: A deep visual-semantic embedding model.
\newblock In \emph{Advances in neural information processing systems}, pages
  2121--2129, 2013.

\bibitem[Gittens et~al.(2017)Gittens, Achlioptas, and Mahoney]{gittens2017skip}
Alex Gittens, Dimitris Achlioptas, and Michael~W Mahoney.
\newblock Skip-gram-zipf+ uniform= vector additivity.
\newblock In \emph{Proceedings of the 55th Annual Meeting of the Association
  for Computational Linguistics (Volume 1: Long Papers)}, volume~1, pages
  69--76, 2017.

\bibitem[Halawi et~al.(2012)Halawi, Dror, Gabrilovich, and Koren]{mturk771}
Guy Halawi, Gideon Dror, Evgeniy Gabrilovich, and Yehuda Koren.
\newblock Large-scale learning of word relatedness with constraints.
\newblock In \emph{Proceedings of the 18th ACM SIGKDD international conference
  on Knowledge discovery and data mining}, pages 1406--1414. ACM, 2012.

\bibitem[Hamilton et~al.(2016)Hamilton, Leskovec, and
  Jurafsky]{hamilton2016diachronic}
William~L Hamilton, Jure Leskovec, and Dan Jurafsky.
\newblock Diachronic word embeddings reveal statistical laws of semantic
  change.
\newblock \emph{arXiv preprint arXiv:1605.09096}, 2016.

\bibitem[Harris(1954)]{harris1954distributional}
Zellig~S Harris.
\newblock Distributional structure.
\newblock \emph{Word}, 10\penalty0 (2-3):\penalty0 146--162, 1954.

\bibitem[Higham(2002)]{higham2002accuracy}
Nicholas~J Higham.
\newblock \emph{Accuracy and stability of numerical algorithms}, volume~80.
\newblock Siam, 2002.

\bibitem[Hochreiter and Schmidhuber(1997)]{hochreiter1997long}
Sepp Hochreiter and J{\"u}rgen Schmidhuber.
\newblock Long short-term memory.
\newblock \emph{Neural computation}, 9\penalty0 (8):\penalty0 1735--1780, 1997.

\bibitem[Jordan(1875)]{jordan1875essai}
Camille Jordan.
\newblock Essai sur la g{\'e}om{\'e}triean dimensions.
\newblock \emph{Bull. Soc. Math. France}, 3:\penalty0 103--174, 1875.

\bibitem[Kato(2013)]{kato2013perturbation}
Tosio Kato.
\newblock \emph{Perturbation theory for linear operators}, volume 132.
\newblock Springer Science \& Business Media, 2013.

\bibitem[Lample et~al.(2016)Lample, Ballesteros, Subramanian, Kawakami, and
  Dyer]{lample2016neural}
Guillaume Lample, Miguel Ballesteros, Sandeep Subramanian, Kazuya Kawakami, and
  Chris Dyer.
\newblock Neural architectures for named entity recognition.
\newblock In \emph{Proceedings of NAACL-HLT}, pages 260--270, 2016.

\bibitem[Landauer(2006)]{landauer2006latent}
Thomas~K Landauer.
\newblock \emph{Latent semantic analysis}.
\newblock Wiley Online Library, 2006.

\bibitem[Landauer et~al.(1998)Landauer, Foltz, and
  Laham]{landauer1998introduction}
Thomas~K Landauer, Peter~W Foltz, and Darrell Laham.
\newblock An introduction to latent semantic analysis.
\newblock \emph{Discourse processes}, 25\penalty0 (2-3):\penalty0 259--284,
  1998.

\bibitem[Levy and Goldberg(2014)]{levy2014neural}
Omer Levy and Yoav Goldberg.
\newblock Neural word embedding as implicit matrix factorization.
\newblock In \emph{Advances in neural information processing systems}, pages
  2177--2185, 2014.

\bibitem[Levy et~al.(2015)Levy, Goldberg, and Dagan]{levy2015improving}
Omer Levy, Yoav Goldberg, and Ido Dagan.
\newblock Improving distributional similarity with lessons learned from word
  embeddings.
\newblock \emph{Transactions of the Association for Computational Linguistics},
  3:\penalty0 211--225, 2015.

\bibitem[Lin and Pantel(2001)]{lin2001dirt}
Dekang Lin and Patrick Pantel.
\newblock Dirt@ sbt@ discovery of inference rules from text.
\newblock In \emph{Proceedings of the seventh ACM SIGKDD international
  conference on Knowledge discovery and data mining}, pages 323--328. ACM,
  2001.

\bibitem[Maaten and Hinton(2008)]{maaten2008visualizing}
Laurens van~der Maaten and Geoffrey Hinton.
\newblock Visualizing data using t-{SNE}.
\newblock \emph{Journal of Machine Learning Research}, 9\penalty0
  (Nov):\penalty0 2579--2605, 2008.

\bibitem[Marelli et~al.(2014)Marelli, Bentivogli, Baroni, Bernardi, Menini, and
  Zamparelli]{marelli2014semeval}
Marco Marelli, Luisa Bentivogli, Marco Baroni, Raffaella Bernardi, Stefano
  Menini, and Roberto Zamparelli.
\newblock Semeval-2014 task 1: Evaluation of compositional distributional
  semantic models on full sentences through semantic relatedness and textual
  entailment.
\newblock 2014.

\bibitem[Mikolov et~al.(2013{\natexlab{a}})Mikolov, Chen, Corrado, and
  Dean]{mikolov2013efficient}
Tomas Mikolov, Kai Chen, Greg Corrado, and Jeffrey Dean.
\newblock Efficient estimation of word representations in vector space.
\newblock \emph{arXiv preprint arXiv:1301.3781}, 2013{\natexlab{a}}.

\bibitem[Mikolov et~al.(2013{\natexlab{b}})Mikolov, Le, and
  Sutskever]{mikolov2013exploiting}
Tomas Mikolov, Quoc~V Le, and Ilya Sutskever.
\newblock Exploiting similarities among languages for machine translation.
\newblock \emph{arXiv preprint arXiv:1309.4168}, 2013{\natexlab{b}}.

\bibitem[Mikolov et~al.(2013{\natexlab{c}})Mikolov, Sutskever, Chen, Corrado,
  and Dean]{NIPS2013_5021}
Tomas Mikolov, Ilya Sutskever, Kai Chen, Greg~S Corrado, and Jeff Dean.
\newblock Distributed representations of words and phrases and their
  compositionality.
\newblock In C.~J.~C. Burges, L.~Bottou, M.~Welling, Z.~Ghahramani, and K.~Q.
  Weinberger, editors, \emph{Advances in Neural Information Processing Systems
  26}, pages 3111--3119. Curran Associates, Inc., 2013{\natexlab{c}}.

\bibitem[Mirsky(1960)]{mirsky1960symmetric}
Leon Mirsky.
\newblock Symmetric gauge functions and unitarily invariant norms.
\newblock \emph{The quarterly journal of mathematics}, 11\penalty0
  (1):\penalty0 50--59, 1960.

\bibitem[Nachmani et~al.(2017)Nachmani, Marciano, Lugosch, Gross, Burshtein,
  and Beery]{nachmani2017deep}
Eliya Nachmani, Elad Marciano, Loren Lugosch, Warren~J Gross, David Burshtein,
  and Yair Beery.
\newblock Deep learning methods for improved decoding of linear codes.
\newblock \emph{arXiv preprint arXiv:1706.07043}, 2017.

\bibitem[Nallapati et~al.(2016)Nallapati, Zhou, dos Santos, glar
  Gul{\c{c}}ehre, and Xiang]{nallapati2016abstractive}
Ramesh Nallapati, Bowen Zhou, Cicero dos Santos, {\c{C}}a~glar Gul{\c{c}}ehre,
  and Bing Xiang.
\newblock Abstractive text summarization using sequence-to-sequence rnns and
  beyond.
\newblock \emph{CoNLL 2016}, page 280, 2016.

\bibitem[Niwa and Nitta(1994)]{niwa1994co}
Yoshiki Niwa and Yoshihiko Nitta.
\newblock Co-occurrence vectors from corpora vs. distance vectors from
  dictionaries.
\newblock In \emph{Proceedings of the 15th conference on Computational
  linguistics-Volume 1}, pages 304--309. Association for Computational
  Linguistics, 1994.

\bibitem[Paige and Wei(1994)]{paige1994history}
Christopher~C Paige and M~Wei.
\newblock History and generality of the {CS} decomposition.
\newblock \emph{Linear Algebra and its Applications}, 208:\penalty0 303--326,
  1994.

\bibitem[Pennington et~al.(2014)Pennington, Socher, and
  Manning]{pennington2014glove}
Jeffrey Pennington, Richard Socher, and Christopher Manning.
\newblock {GloVe}: Global vectors for word representation.
\newblock In \emph{Proceedings of the 2014 conference on empirical methods in
  natural language processing (EMNLP)}, pages 1532--1543, 2014.

\bibitem[Rapp(2003)]{rapp2003word}
Reinhard Rapp.
\newblock Word sense discovery based on sense descriptor dissimilarity.
\newblock pages 315--322, 2003.

\bibitem[Salton(1971)]{salton1971smart}
Gerard Salton.
\newblock The smart retrieval system—experiments in automatic document
  processing.
\newblock 1971.

\bibitem[Salton and Buckley(1988)]{salton1988term}
Gerard Salton and Christopher Buckley.
\newblock Term-weighting approaches in automatic text retrieval.
\newblock \emph{Information processing \& management}, 24\penalty0
  (5):\penalty0 513--523, 1988.

\bibitem[Schnabel et~al.(2015)Schnabel, Labutov, Mimno, and
  Joachims]{schnabel2015evaluation}
Tobias Schnabel, Igor Labutov, David Mimno, and Thorsten Joachims.
\newblock Evaluation methods for unsupervised word embeddings.
\newblock In \emph{Proceedings of the 2015 Conference on Empirical Methods in
  Natural Language Processing}, pages 298--307, 2015.

\bibitem[Smith et~al.(2017)Smith, Turban, Hamblin, and
  Hammerla]{smith2017offline}
Samuel~L Smith, David~HP Turban, Steven Hamblin, and Nils~Y Hammerla.
\newblock Offline bilingual word vectors, orthogonal transformations and the
  inverted softmax.
\newblock \emph{arXiv preprint arXiv:1702.03859}, 2017.

\bibitem[Socher et~al.(2013)Socher, Perelygin, Wu, Chuang, Manning, Ng, and
  Potts]{socher2013recursive}
Richard Socher, Alex Perelygin, Jean Wu, Jason Chuang, Christopher~D Manning,
  Andrew Ng, and Christopher Potts.
\newblock Recursive deep models for semantic compositionality over a sentiment
  treebank.
\newblock In \emph{Proceedings of the 2013 conference on empirical methods in
  natural language processing}, pages 1631--1642, 2013.

\bibitem[Sparck~Jones(1972)]{sparck1972statistical}
Karen Sparck~Jones.
\newblock A statistical interpretation of term specificity and its application
  in retrieval.
\newblock \emph{Journal of documentation}, 28\penalty0 (1):\penalty0 11--21,
  1972.

\bibitem[Stewart(1990{\natexlab{a}})]{stewart1990matrix}
Gilbert~W Stewart.
\newblock Matrix perturbation theory.
\newblock 1990{\natexlab{a}}.

\bibitem[Stewart(1990{\natexlab{b}})]{stewart1990stochastic}
Gilbert~W Stewart.
\newblock Stochastic perturbation theory.
\newblock \emph{SIAM review}, 32\penalty0 (4):\penalty0 579--610,
  1990{\natexlab{b}}.

\bibitem[Sutskever et~al.(2014)Sutskever, Vinyals, and
  Le]{sutskever2014sequence}
Ilya Sutskever, Oriol Vinyals, and Quoc~V Le.
\newblock Sequence to sequence learning with neural networks.
\newblock In \emph{Advances in neural information processing systems}, pages
  3104--3112, 2014.

\bibitem[Turney(2012)]{turney2012domain}
Peter~D Turney.
\newblock Domain and function: A dual-space model of semantic relations and
  compositions.
\newblock \emph{Journal of Artificial Intelligence Research}, 44:\penalty0
  533--585, 2012.

\bibitem[Turney and Littman(2003)]{turney2003measuring}
Peter~D Turney and Michael~L Littman.
\newblock Measuring praise and criticism: Inference of semantic orientation
  from association.
\newblock \emph{ACM Transactions on Information Systems (TOIS)}, 21\penalty0
  (4):\penalty0 315--346, 2003.

\bibitem[Turney and Pantel(2010)]{turney2010frequency}
Peter~D Turney and Patrick Pantel.
\newblock From frequency to meaning: Vector space models of semantics.
\newblock \emph{Journal of artificial intelligence research}, 37:\penalty0
  141--188, 2010.

\bibitem[Vinyals et~al.(2015)Vinyals, Toshev, Bengio, and
  Erhan]{vinyals2015show}
Oriol Vinyals, Alexander Toshev, Samy Bengio, and Dumitru Erhan.
\newblock Show and tell: A neural image caption generator.
\newblock In \emph{Proceedings of the IEEE conference on computer vision and
  pattern recognition}, pages 3156--3164, 2015.

\bibitem[Vu(2011)]{vu2011singular}
Van Vu.
\newblock Singular vectors under random perturbation.
\newblock \emph{Random Structures \& Algorithms}, 39\penalty0 (4):\penalty0
  526--538, 2011.

\bibitem[Weyl(1912)]{weyl1912asymptotische}
Hermann Weyl.
\newblock Das asymptotische verteilungsgesetz der eigenwerte linearer
  partieller differentialgleichungen (mit einer anwendung auf die theorie der
  hohlraumstrahlung).
\newblock \emph{Mathematische Annalen}, 71\penalty0 (4):\penalty0 441--479,
  1912.

\bibitem[Wu et~al.(2016)Wu, Schuster, Chen, Le, Norouzi, Macherey, Krikun, Cao,
  Gao, Macherey, et~al.]{wu2016google}
Yonghui Wu, Mike Schuster, Zhifeng Chen, Quoc~V Le, Mohammad Norouzi, Wolfgang
  Macherey, Maxim Krikun, Yuan Cao, Qin Gao, Klaus Macherey, et~al.
\newblock Google's neural machine translation system: Bridging the gap between
  human and machine translation.
\newblock \emph{arXiv preprint arXiv:1609.08144}, 2016.

\bibitem[Xu et~al.(2015)Xu, Ba, Kiros, Cho, Courville, Salakhudinov, Zemel, and
  Bengio]{xu2015show}
Kelvin Xu, Jimmy Ba, Ryan Kiros, Kyunghyun Cho, Aaron Courville, Ruslan
  Salakhudinov, Rich Zemel, and Yoshua Bengio.
\newblock Show, attend and tell: Neural image caption generation with visual
  attention.
\newblock In \emph{International Conference on Machine Learning}, pages
  2048--2057, 2015.

\bibitem[Yin et~al.(2017)Yin, Chang, and Zhang]{yin2017deepprobe}
Zi~Yin, Keng-hao Chang, and Ruofei Zhang.
\newblock Deepprobe: Information directed sequence understanding and chatbot
  design via recurrent neural networks.
\newblock In \emph{Proceedings of the 23rd ACM SIGKDD International Conference
  on Knowledge Discovery and Data Mining}, pages 2131--2139. ACM, 2017.

\end{thebibliography}
\bibliographystyle{plainnat}

\newpage
\section{Appendix}

\subsection{Proof of Lemma \ref{lemma:1}}
\begin{proof}
We prove this lemma by obtaining the eigendecomposition of 
$X_0^TY_1(X_0^TY_1)^T$:
\begin{align*}
X_0^TY_1Y_1^TX_0&=X_0^T(I-Y_0Y_0^T)X_0\\
&=I-U_0\cos^2(\Theta)U_0^T\\
&=U_0\sin^2(\Theta)U_0^T
\end{align*}
Hence the $X_0^TY_1$ has singular value decomposition of $U_0\sin(\Theta)\tilde V_1^T$ for some orthogonal $\tilde{V_1}$.
\end{proof}
\subsection{Proof of Lemma \ref{lemma:2}}
\begin{proof}
Note $Y_0=UU^TY_0=U(
\begin{bmatrix}
    X_0^T  \\
    X_1^T
\end{bmatrix}Y_0)$, so
\begin{align*}
Y_0Y_0^T&=U(\begin{bmatrix}
    X_0^T  \\
    X_1^T
\end{bmatrix}Y_0Y_0^T
\begin{bmatrix}
    X_0 & X_1
\end{bmatrix})U^T\\
&=U
\begin{bmatrix}
    X_0^TY_0Y_0^TX_0 & X_0^TY_0Y_0^TX_1 \\
    X_1^TY_0Y_0^TX_0 & X_1^TY_0Y_0^TX_1
\end{bmatrix}
U^T
\end{align*}
Let $X_0^TY_0=U_0\cos(\Theta)V_0^T$, $Y_0^TX_1=V_0\sin(\Theta)\tilde U_1^T$ by Lemma \ref{lemma:1}. For any unit invariant norm,
\begin{align*}
\norm{Y_0Y_0^T-X_0X_0^T}=&\norm{U
(\begin{bmatrix}
    X_0^TY_0Y_0^TX_0 & X_0^TY_0Y_0^TX_1 \\
    X_1^TY_0Y_0^TX_0 & X_1^TY_0Y_0^TX_1
\end{bmatrix}-
\begin{bmatrix}
    I & 0 \\
    0 & 0
\end{bmatrix})
U^T}\\
=&\norm{
\begin{bmatrix}
    X_0^TY_0Y_0^TX_0 & X_0^TY_0Y_0^TX_1 \\
    X_1^TY_0Y_0^TX_0 & X_1^TY_0Y_0^TX_1
\end{bmatrix}-
\begin{bmatrix}
    I & 0 \\
    0 & 0
\end{bmatrix}
}\\
=&\norm{
\begin{bmatrix}
    U_0\cos^2(\Theta)U_0^T & U_0\cos(\Theta)\sin(\Theta)\tilde U_1^T \\
    \tilde U_1\cos(\Theta)\sin(\Theta)U_0^T & \tilde U_1\sin^2(\Theta)\tilde U_1^T
\end{bmatrix}-
\begin{bmatrix}
    I & 0 \\
    0 & 0
\end{bmatrix}
}\\
=&\norm{
\begin{bmatrix}
    U_0 & 0 \\
    0 & \tilde U_1
\end{bmatrix}
\begin{bmatrix}
    -\sin^2(\Theta) & \cos(\Theta)\sin(\Theta) \\
    \cos(\Theta)\sin(\Theta) & \sin^2(\Theta)
\end{bmatrix}
\begin{bmatrix}
    U_0 & 0 \\
    0 & \tilde U_1
\end{bmatrix}^T
}\\
=&\norm{
\begin{bmatrix}
    -\sin^2(\Theta) & \cos(\Theta)\sin(\Theta) \\
    \cos(\Theta)\sin(\Theta) & \sin^2(\Theta)
\end{bmatrix}
}\\
=&\norm{
\begin{bmatrix}
\sin(\Theta) & 0\\
0 & \sin(\Theta)
\end{bmatrix}
\begin{bmatrix}
    -\sin(\Theta) & \cos(\Theta) \\
    \cos(\Theta) & \sin(\Theta)
\end{bmatrix}
}\\
=&\norm{
\begin{bmatrix}
\sin(\Theta) & 0\\
0 & \sin(\Theta)
\end{bmatrix}}
\end{align*}
On the other hand by the definition of principal angles,
\[\|X_0^TY_1\|=\|\sin(\Theta)\|\]
So we established the lemma. Specifically, we have
\begin{enumerate}
\item $\|X_0X_0^T-Y_0Y_0^T\|_2=\|X_0^TY_1\|_2$
\item $\|X_0X_0^T-Y_0Y_0^T\|_F=\sqrt{2}\|X_0^TY_1\|_F$
\end{enumerate}
\end{proof}
Without loss of soundness, we omitted in the proof sub-blocks of identities or zeros for simplicity. Interested readers can refer to classical matrix CS-decomposition texts, for example \citet{stewart1990matrix,paige1994history, davis1970rotation, kato2013perturbation}, for a comprehensive treatment of this topic.

\subsection{Proof of Theorem \ref{theorem:2}}
\begin{proof}
Let $E=X_0D_{0}^\alpha$ and $\hat E=Y_0 \tilde D_{0}^\alpha$, where for notation simplicity we denote $D_0=D_{1:d,1:d}=diag(\lambda_1, \cdots, \lambda_d)$ and $\tilde D_0=\tilde D_{1:k,1:k}=diag(\tilde\lambda_1, \cdots, \tilde\lambda_k)$, with $k\le d$. Observe $D_0$ is diagonal and the entries are in descending order. As a result, we can write $D_0$ as a telescoping sum:
\[D_0^\alpha=\sum_{i=1}^k (\lambda_i^\alpha-\lambda_{i+1}^{\alpha})I_{i,i}\]
where $I_{i,i}$ is the $i$ by $i$ dimension identity matrix and $\lambda_{d+1}=0$ is adopted. As a result, we can telescope the difference between the PIP matrices. Note we again split $X_0\in\mathbb{R}^{n\times d}$ into $X_{0,0}\in\mathbb{R}^{n\times k}$ and $X_{0,1}\in\mathbb{R}^{n\times (d-k)}$, together with $D_{0,0}=diag(\lambda_1,\cdots,\lambda_k)$ and $D_{0,1}=diag(\lambda_{k+1},\cdots,\lambda_d)$, to match the dimension of the trained embedding matrix.
\begin{align*}
\|EE^T-\hat E\hat E^T\|&=\|X_{0,1}D_{0,1}^{2\alpha}X_{0,1}^T-Y_0\tilde D_0^{2\alpha}Y_0^T+X_{0,2}D_{0,2}^{2\alpha}X_{0,2}^T\|\\
&\le\|X_{0,2}D_{0,2}^{2\alpha}X_{0,2}^T\|+\|X_{0,1}D_{0,1}^{2\alpha}X_{0,1}^T-Y_0\tilde D_0^{2\alpha}Y_0^T\|\\
&=\|X_{0,2}D_{0,2}^{2\alpha}X_{0,2}^T\|+\|X_{0,1}D_{0,1}^{2\alpha}X_{0,1}^T-Y_0D_{0,1}^{2\alpha}Y_0^T+Y_0D_{0,1}^{2\alpha}Y_0^T-Y_0\tilde D_0^{2\alpha}Y_0^T\|\\
&\le\|X_{0,2}D_{0,2}^{2\alpha}X_{0,2}^T\|+\|X_{0,1}D_{0,1}^{2\alpha}X_{0,1}^T-Y_0D_{0,1}^{2\alpha}Y_0^T\|+\|Y_0D_{0,1}^{2\alpha}Y_0^T-Y_0\tilde D_0^{2\alpha}Y_0^T\|
\end{align*}
We now upper bound the above 3 terms separately.
\begin{enumerate}
\item Term 1 can be computed directly:
\begin{align*}
\|X_{0,2}D_{0,2}^{2\alpha}X_{0,2}^T\|&=\sqrt{\|\sum_{i=k+1}^d \lambda_i^{2\alpha}x_{\cdot, i}x_{\cdot, i}^T\|^2}=\sqrt{\sum_{i=k+1}^d \lambda_i^{4\alpha}}\\
\end{align*}

\item We bound term 2 using the telescoping observation and Lemma \ref{lemma:2}:
\begin{align*}
\|X_{0,1}D_{0,1}^{2\alpha}X_{0,1}^T-Y_0D_{0,1}^{2\alpha}Y_0^T\|&=\|\sum_{i=1}^k (\lambda_i^{2\alpha}-\lambda_{i+1}^{2\alpha})(X_{\cdot,1:i}X_{\cdot,1:i}^T-Y_{\cdot,1:i}Y_{\cdot,1:i}^T)\|\\
&\le\sum_{i=1}^k (\lambda_i^{2\alpha}-\lambda_{i+1}^{2\alpha})\|X_{\cdot,1:i}X_{\cdot,1:i}^T-Y_{\cdot,1:i}Y_{\cdot,1:i}^T\|\\
&=\sqrt{2}\sum_{i=1}^k (\lambda_i^{2\alpha}-\lambda_{i+1}^{2\alpha})\|Y_{\cdot,1:i}^T X_{\cdot,i:n}\|
\end{align*}
\item Third term:
\begin{align*}
\|Y_0D_{0,1}^{2\alpha}Y_0^T-Y_0\tilde D_0^{2\alpha}Y_0^T\|&=\sqrt{\|\sum_{i=1}^k (\lambda_i^{2\alpha}-\tilde\lambda_{i}^{2\alpha})Y_{\cdot,i}Y_{\cdot,i}^T\|^2}\\
&=\sqrt{\sum_{i=1}^k (\lambda_i^{2\alpha}-\tilde\lambda_{i}^{2\alpha})^2}
\end{align*}
\end{enumerate}
Collect all the terms above, we arrive at an upper bound for the PIP discrepancy:
\begin{align*}
\|EE^T-\hat E\hat E^T\|&\le\sqrt{\sum_{i=k+1}^d \lambda_i^{4\alpha}}+\sqrt{\sum_{i=1}^k (\lambda_i^{2\alpha}-\tilde\lambda_{i}^{2\alpha})^2}+\sqrt{2}\sum_{i=1}^k (\lambda_i^{2\alpha}-\lambda_{i+1}^{2\alpha})\|Y_{\cdot,1:i}^T X_{\cdot,i:n}\|
\end{align*}
\end{proof}

\subsection{Proof of Lemma \ref{lemma:bias1}}
\begin{proof}
To bound the term $\sqrt{\sum_{i=1}^k (\lambda_i^{2\alpha}-\tilde\lambda_{i}^{2\alpha})^2}$, we use a classical result \citep{weyl1912asymptotische,mirsky1960symmetric}.

\begin{theorem}[Weyl]\label{theorem:weyl}
Let $\{\lambda_i\}_{i=1}^n$ and $\{\tilde\lambda_i\}_{i=1}^n$ be the spectrum of $M$ and $\tilde M=M+Z$, where we include 0 as part of the spectrum. Then
\[\max_{i}|\lambda_i-\tilde\lambda_i|\le \|Z\|_2\]
\end{theorem}
\begin{theorem}[Mirsky-Wielandt-Hoffman]\label{theorem:mirsky}
Let $\{\lambda_i\}_{i=1}^n$ and $\{\tilde\lambda_i\}_{i=1}^n$ be the spectrum of $M$ and $\tilde M=M+Z$. Then
\[(\sum_{i=1}^n|\lambda_i-\tilde\lambda_i|^p)^{1/p}\le \|Z\|_{S_p}\]
\end{theorem}


We use a first-order Taylor expansion followed by applying Weyl's Theorem \ref{theorem:weyl}:
\begin{align*}
\sqrt{\sum_{i=1}^k (\lambda_i^{2\alpha}-\tilde\lambda_{i}^{2\alpha})^2}&\approx \sqrt{\sum_{i=1}^k (2\alpha\lambda_i^{2\alpha-1}(\lambda_i-\tilde\lambda_{i}))^2}\\
&=2\alpha\sqrt{\sum_{i=1}^k \lambda_i^{4\alpha-2}(\lambda_i-\tilde\lambda_{i})^{2}}\\
&\le 2\alpha\|N\|_2\sqrt{\sum_{i=1}^k \lambda_i^{4\alpha-2}}
\end{align*}
Now take expectation on both sides and use Tracy-Widom Law:
\[\mathbb E[\sqrt{\sum_{i=1}^k (\lambda_i^{2\alpha}-\tilde\lambda_{i}^{2\alpha})^2}]\le
2\sqrt{2n}\alpha\sigma\sqrt{\sum_{i=1}^k \lambda_i^{4\alpha-2}}\]

\end{proof}
A further comment is that this bound can tightened for $\alpha=0.5$, by using Mirsky-Wieland-Hoffman's theorem instead of Weyl's theorem \citep{stewart1990matrix}. In this case,
\[\mathbb E[\sqrt{\sum_{i=0}^k (\lambda_i^{2\alpha}-\tilde\lambda_{i}^{2\alpha})^2}]\le
k\sigma\]
where we can further save a $\sqrt{2n/k}$ factor.
\subsection{Proof of Theorem \ref{theorem:T}}
Classical matrix perturbation theory focuses on bounds; namely, the theory provides upper bounds on how much an invariant subspace of a matrix $\tilde A=A+E$ will differ from that of $A$. Note that we switched notation to accommodate matrix perturbation theory conventions (where usually $A$ denotes the unperturbed matrix, $\tilde A$ denotes the one after perturbation, and $E$ denotes the noise). The most famous and widely-used ones are the $\sin \Theta$ theorems:
\begin{theorem}[$\sin \Theta$]\label{theorem:davis-kahan}
For two matrices $A$ and $\tilde A=A+E$, denote their singular value decompositions as $A=XDU^T$ and $\tilde A=Y\Lambda V^T$. Formally construct the column blocks $X=[X_0,X_1]$ and $Y=[Y_0,Y_1]$ where both $X_0$ and $Y_0\in\mathbb{R}^{n\times k}$, if the spectrum of $D_0$ and $D_1$ has separation 
\[\delta_k\overset{\Delta}{=}\min_{1\le i\le k,1\le j\le n-k}\{(D_0)_{ii}-(D_1)_{jj}\},\]
then
\[\|Y_1^TX_0\|\le \frac{\|Y_1^TEX_0\|}{\delta_k}\le\frac{\|E\|}{\delta_k}\]
\end{theorem}
Theoretically, the $\sin \Theta$ theorem should provide an upper bound on the invariant subspace discrepancies caused by the perturbation. However, we found that the bounds become extremely loose, making it barely usable for real data. Specifically, when the separation $\delta_k$ becomes small, the bound can be quite large. So what was going on and how should we fix it?

In the \textit{minimax} sense, the gap $\delta_k$ indeed dictates the max possible discrepancy, and is tight. However, the noise $E$ in our application is random, not adversarial. So the universal guarantee by the $\sin \Theta$ theorem is too conservative. Our approach uses a technique first discovered by Stewart in a series of papers \citep{stewart1990matrix,stewart1990stochastic}. Instead of looking for a universal upper bound, we derive a \textit{first order approximation} of the perturbation.

\subsubsection{First Order Approximation of $\|Y_1^TX_0\|$}
We split the signal $A$ and noise $E$ matrices into block form, with $A_{11}, E_{11}\in\mathbb R^{k\times k}$, $A_{12}, E_{12}\in\mathbb R^{k\times (n-k)}$, $A_{21}, E_{21}\in\mathbb R^{(n-k)\times k}$ and $A_{22}, E_{22}\in\mathbb R^{(n-k)\times (n-k)}$.
\[
A=
\left[
\begin{array}{c|c}
A_{11} & A_{12} \\
\hline
A_{21} & A_{22}
\end{array}
\right],~
E=
\left[
\begin{array}{c|c}
E_{11} & E_{12} \\
\hline
E_{21} & E_{22}
\end{array}
\right]
\]
As noted by Stewart in \citep{stewart1990stochastic},
\begin{equation}\label{eq:t1}
X_0=Y_0(I+P^TP)^{\frac{1}{2}}-X_1P
\end{equation}
and
\begin{equation}\label{eq:t2}
Y_1=(X_1-X_0P^T)(I+P^TP)^{-\frac{1}{2}}
\end{equation}
where $P$ is the solution to the equation
\begin{equation}\label{eq:t3}
T(P) + (E_{22}P-P E_{11}) = E_{21}-
P\tilde A_{12}P
\end{equation}
The operator $T$ is a linear operator on $P\in \mathbb{R}^{(n-k)\times k}\rightarrow \mathbb{R}^{(n-k)\times k}$, defined as
\[T(P)= A_{22}P-PA_{11}\]
Now, we drop the second order terms in equation (\ref{eq:t1}) and (\ref{eq:t2}),
\[X_0\approx Y_0-X_1P,\ Y_1\approx X_1-X_0P^T\]
So 
\begin{align*}
Y_1^TX_0&\approx Y_1^T(Y_0-X_1P)=Y_1^TX_1P\\
&\approx(X_1^T-PX_0^T)X_1P=P
\end{align*}
As a result, $\|Y_1^TX_0\|\approx\|P\|$.

To approximate $P$, we drop the second order terms on $P$ in equation (\ref{eq:t3}), and get:
\begin{equation}\label{eq:sylvester}
T(P)\approx E_{21}
\end{equation}
or $P\approx T^{-1}(E_{21})$ as long as $T$ is invertible. Our final approximation is
\begin{equation}\label{eq:t5}
\|Y_1^TX_0\|\approx\|T^{-1}(E_{21})\|
\end{equation}
\subsubsection{The Sylvester Operator $T$}
To solve equation (\ref{eq:t5}), we perform a spectral analysis on $T$:
\begin{lemma}\label{lemma:inv}
There are $k(n-k)$ eigenvalues of $T$, which are
\[\lambda_{ij}=(D_0)_{ii}-(D_1)_{jj}\]
\end{lemma}
\begin{proof}
By definition, $T(P)=\lambda P$ implies 
\[A_{22}P-PA_{11}=\lambda P\]
Let $A_{11}=U_0D_0U_0^T$, $A_{22}=U_1D_1U_1^T$ and $\tilde P=U_1^TPU_0$, we have
\[D_0\tilde P-\tilde PD_{1}=\lambda \tilde P\]
Note that when $\tilde{P}=e_ie_j^T$,
\[D_0e_ie_j^T-e_ie_j^TD_1=((D_0)_{ii}-(D_1)_{jj})e_ie_j^T\]
So we know that the operator $T$ has eigenvalue $\lambda_{ij}=(D_0)_{ii}-(D_1)_{jj}$ with eigen-function $U_1e_ie_j^TU_0^T$.
\end{proof}
Lemma \ref{lemma:inv} not only gives an orthogonal decomposition of the operator $T$, but also points out when $T$ is invertible, namely the spectrum $D_0$ and $D_1$ do not overlap, or equivalently $\delta_k>0$. Since $E_{12}$ has iid entries with variance $\sigma^2$, using Lemma \ref{lemma:inv} together with equation (\ref{eq:t5}) from last section, we conclude
\begin{align*}
\|Y_1^TX_0\|&\approx\|T^{-1}(E_{21})\|\\
&=\|\sum_{1\le i\le k,1\le j\le n-k}\lambda_{ij}^{-1}\langle E_{21},e_ie_j^T\rangle\|\\
&=\sqrt{\sum_{1\le i\le k,1\le j\le n-k}\lambda_{ij}^{-2}E_{21,ij}^2}
\end{align*}
By Jensen's inequality,
\begin{align*}
\mathbb E\|Y_1^TX_0\|\le\sqrt{\sum_{i,j}\lambda_{ij}^{-2}\sigma^2}=\sigma\sqrt{\sum_{i,j}\lambda_{ij}^{-2}}
\end{align*}
Our new bound is \textit{much sharper} than the $\sin\Theta$ theorem, which gives $\sigma\frac{\sqrt{k(n-k)}}{\delta}$ in this case. Notice if we upper bound every $\lambda_{ij}^{-2}$ with $\delta_k^{-2}$ in our result, we will obtain the same bound as the $\sin\Theta$ theorem. In other words, our bound considers \textit{every} singular value gap, not only the smallest one. This technical advantage can clearly be seen, both in the simulation and in the real data.

\subsection{Experimentation Setting for Dimensionality Selection Time Comparison}
For PIP loss minimizing method, we first estimate the spectrum of $M$ and noise standard deviation $\sigma$ with methods described in Section \ref{sec:noise_est} and \ref{sec:spectrum_est}. $E=UD^\alpha$ was generated with a random orthogonal matrix $U$. Note any orthogonal $U$ is equivalent due to the unitary invariance. For every dimensionality $k$, the PIP loss for $\hat E=\tilde{U}_{\cdot,1:k}\tilde{D}^\alpha_{1:k,1:k}$ was calculated and $\|\hat E\hat E^T-EE^T\|$ is computed. Sweeping through all $k$ is very efficient because one pass of full sweeping is equivalent of doing a single SVD on $\tilde M=M+Z$. The method is the same for LSA, skip-gram and GloVe, with different signal matrices (PPMI, PMI and log-count respectively).

For empirical selection method, the following approaches are taken:
\begin{itemize}
\item LSA: The PPMI matrix is constructed from the corpus, a full SVD is done. We truncate the SVD at $k$ to get dimensionality $k$ embedding. This embedding is then evaluated on the testsets \citep{mturk771,wordsim353}, and each testset will report an optimal dimensionality. Note the different testsets may not agree on the same dimensionality.
\item Skip-gram and GloVe: We obtained the source code from the authors' Github repositories\footnote{https://github.com/tensorflow/models/tree/master/tutorials/embedding}\footnote{https://github.com/stanfordnlp/GloVe}. We then train word embeddings from dimensionality 1 to 400, at an increment of 2. To make sure all CPUs are effectively used, we train multiple models at the same time. Each dimensionality is trained for 15 epochs. After finish training all dimensionalities, the models are evaluated on the testsets \citep{mturk771,wordsim353,mikolov2013efficient}, where each testset will report an optimal dimensionality. Note we already used a step size larger than 1 (2 in this case) for dimensionality increment. Had we used 1 (meaning we train every dimensionality between 1 and 400), the time spent will be doubled, which will be close to a week.

\end{itemize}

\end{document}